\documentclass[twocolumn]{svjour3} 

\smartqed  

\usepackage{graphicx}
\usepackage{float}
\usepackage{amsmath}
\usepackage{graphicx}
\usepackage{graphics}
\usepackage{subfigure}
\usepackage[]{natbib}
\usepackage{wrapfig}
\usepackage{color}
\usepackage{url}
\usepackage{natbib}
\usepackage{amssymb}
\usepackage{algpseudocode}
\usepackage{algorithm}
\usepackage{colortbl}
\usepackage[utf8]{inputenc}
\usepackage{xr}
\usepackage{url}
\usepackage{wrapfig}
\usepackage{subfigure}
\usepackage{epstopdf}

\DeclareMathOperator*{\argmax}{argmax}
\newfloat{algorithm}{t}{lop}

\journalname{Autonomous Robots: Special Issue on Active Perception}

\begin{document}

\title{Exploiting Submodular Value Functions\\for Scaling Up Active Perception} 

\author{Yash Satsangi \and Shimon Whiteson \and \newline Frans A. Oliehoek \and	Matthijs T. J. Spaan}

\authorrunning{Satsangi et al.} 

\institute{Yash Satsangi \at
              Informatics Institute, University of Amsterdam \\
              \email{y.satsangi@uva.nl} \\ 
              Tel.: +31-(0)20-525-8516 \\ 
              Fax.: +31-(0)20-525-7490 
              \and
	Shimon Whiteson\at
	University of Oxford. \\
	\vspace{-2mm} 
	\and
	Frans A. Oliehoek\at
	University of Liverpool. \\
	University of Amsterdam. \\
	\vspace{-2mm}
	\and 
	Matthijs T. J. Spaan\at
	Delft University of Technology.
}

\date{Received: date / Accepted: date}

\maketitle

\begin{abstract}
In \emph{active perception} tasks, an agent aims to select sensory actions that reduce its uncertainty about one or more hidden variables. For example, a mobile robot takes sensory actions to efficiently navigate in a new environment. While \emph{partially observable Markov decision processes} (POMDPs) provide a natural model for such problems, \emph{reward functions} that directly penalize uncertainty in the agent's belief can remove the \emph{piecewise-linear and convex} (PWLC) property of the \emph{value function} required by most POMDP planners. Furthermore, as the number of sensors available to the agent grows, the computational cost of POMDP planning grows exponentially with it, making POMDP planning infeasible with traditional methods.

In this article, we address a twofold challenge of modeling and planning for active perception tasks. We analyze $\rho$POMDP and POMDP-IR, two frameworks for modeling active perception tasks, that restore the PWLC property of the value function. We show the mathematical equivalence of these two frameworks by showing that given a $\rho$POMDP along with a policy, they can be reduced to a POMDP-IR and an equivalent policy (and vice-versa). We prove that the value function for the given $\rho$POMDP (and the given policy) and the reduced POMDP-IR (and the reduced policy) is the same. 
To efficiently plan for active perception tasks, we identify and exploit the independence properties of POMDP-IR to reduce the computational cost of solving POMDP-IR (and $\rho$POMDP). We propose \emph{greedy point-based value iteration (PBVI)}, a new POMDP planning method that uses \emph{greedy maximization} to greatly improve scalability in the action space of an active perception POMDP. Furthermore, we show that, under certain conditions, including \emph{submodularity}, the value function computed using greedy PBVI is guaranteed to have bounded error with respect to the optimal value function. We establish the conditions under which the value function of an active perception POMDP is guaranteed to be \emph{submodular}. Finally, we present a detailed empirical analysis on a dataset collected from a multi-camera tracking system employed in a shopping mall. Our method achieves similar performance to existing methods but at a fraction of the computational cost leading to better scalability for solving active perception tasks. 
\end{abstract}

\section{Introduction}

\nocite{*}

\emph{Multi-sensor systems} are becoming increasingly prevalent in a wide-range of settings. For example, multi-camera systems are now routinely used for security, surveillance and tracking \citep{ActiveSensingSensor,natarajan2012,SpaanJaamas}. A key challenge in the design  of these systems is the efficient allocation of scarce resources such as the bandwidth required to communicate the collected data to a central server, the CPU cycles required to process that data, and the energy costs of the entire system \citep{ActiveSensingSensor,ApproxDynamProg,Spaan09icaps}. For example, state of the art human activity recognition algorithms require high resolution video streams coupled with significant computational resources. When a human operator must monitor many camera streams, displaying only a small number of them can reduce the operator's cognitive load. IP-cameras connected directly to a local area network need to share bandwidth.  Such constraints gives rise to the \emph{dynamic sensor selection} \citep{satsangi2015}\footnote{This article extends the research already presented by \citet{satsangi2015} at AAAI 2015. In this article, we present additional theoretical results on equivalence of POMDP-IR and $\rho$POMDP, a new technique that exploits the independence properties of POMDP-IR to solve it more efficiently, and we present a detailed empirical analysis of belief-based rewards for POMDPs in active perception tasks. This is a corrected version of this article and contains the same corrections as pointed in \citet{satsangi2015}. We thank Csaba Szepesvari for pointing this. The original article is publicly made available by Springer Journals at https://doi.org/10.1007/s10514-017-9666-5 .} problem where an agent at each time step, must select $K$ out of the $N$ available sensors to allocate these resources to, where $K$ is the maximum number of sensors allowed given the resource constraints. 

For example, consider the \emph{surveillance task}, in which a mobile robot aims to minimize its future uncertainty about the state of the environment but can use only $K$ of its $N$ sensors at each time step.
Surveillance is an example of an \emph{active perception} \citep{bajcsy1988active} task, where an agent takes actions to reduce uncertainty about one or more hidden variables, while reasoning about various resource constraints. When the state of the environment is static, a \emph{myopic} approach that always selects actions that maximize the immediate expected reduction in uncertainty is typically sufficient. However, when the state changes over time, a \emph{non-myopic} approach that reasons about the long-term effects of action selection performed at each time step can be better. For example, in the surveillance task, as the robot moves and the state of the environment changes, it becomes essential to reason about the long-term consequences of the robot's actions to minimize the future uncertainty.
 
A natural decision-theoretic model for such an approach is the \emph{partially observable Markov decision process} (POMDP) \citep{Sondik71,KaelblingPOMDP,mykelK}. POMDPs provide a comprehensive and powerful framework for planning under uncertainty. They can model the dynamic and partially observable state and express the goals of the systems in terms of rewards associated with state-action pairs. This model of the world can be used to compute closed-loop, long term policies that can help the agent to decide what actions to take given a belief about the state of the environment \citep{mobileRobot,kurniawati}.

In a typical POMDP, reducing uncertainty about the state is only \emph{a means to an end}. For example, a robot whose goal is to reach a particular location may take sensing actions that reduce its uncertainty about its current location because doing so helps it determine what future actions will bring it closer to its goal. By contrast, in active perception problems reducing uncertainty is \emph{an end in itself}. For example, in the surveillance task, the system's goal is typically to ascertain the state of its environment, not use that knowledge to achieve a goal. While perception is arguably always performed to aid decision-making, in an active perception problem that decision is made by another agent such as a human, that is not modeled as a part of the POMDP. For example, in the surveillance task, the robot might be able to detect a suspicious activity but only the human users of the system may decide how to react to such an activity. 

One way to formulate uncertainty reduction as an end in itself is to define a \emph{reward function} whose additive inverse is some measure of the agent's uncertainty about the hidden state, e.g., the \emph{entropy} of its \emph{belief}. However this formulation leads to a reward function that conditions on the belief, rather than the state and the resulting \emph{value function} is not PWLC, which makes many traditional POMDP solvers inapplicable. There exists \emph{online planning methods} \citep{MCTS, rtdpbel}, which generates policy on the fly, that do not require the PWLC property of the value function. However, many of these methods require multiple `hypothetical' belief updates to compute the optimal policy, which makes them unsuitable for sensor selection where the optimal policy must be computed in a fraction of a second. There exists other online planning methods that do not require hypothetical belief updates \citep{MCTS}, but since we are dealing with belief based rewards, they cannot be directly applied here.
Here, we address the case of \emph{offline planning} where the policy is computed before execution of the task.

Thus, to efficiently solve active perception problems, we must (a) model the problem with minimizing uncertainty as the objective while maintaining a PWLC value function and (b) use this model to solve the POMDP efficiently. Recently, two frameworks have been proposed, \emph{$\rho$POMDP} \citep{Mauricio}  and \emph{POMDP with Information Reward} (POMDP-IR) \citep{SpaanJaamas} to efficiently model active perception tasks, such that the PWLC property of the value function is maintained. The idea behind $\rho$POMDP is to find a PWLC approximation to the ``true'' continuous belief-based reward function, and then solve it with the traditional solvers. POMDP-IR, on the other hand, allows the agent to make predictions about the hidden state and the agent is rewarded for accurate predictions via a state-based reward function. There is no research that examines the relationship between these two frameworks, their pros and cons, or their efficacy in realistic tasks, thus it is not clear how to choose between these two frameworks to model the active perception problems. 

In this article, we address the problem of efficient modeling and planning for active perception tasks. First, we study the relationship between $\rho$POMDP and POMDP-IR. Specifically, we establish \emph{equivalence} between them by showing that any $\rho$POMDP can be reduced to a POMDP-IR (and vice-versa) that preserves the value function for equivalent policies. Having established the theoretical relationship between $\rho$POMDP and POMDP-IR, we model the surveillance task as a POMDP-IR and propose a new method to solve it efficiently by exploiting a simple insight that lets us decompose the maximization over prediction actions and normal actions while computing the value function.
 
Although POMDPs are computationally difficult to solve, recent methods \citep{littmanthesis,hauskrecht2000,anytimePBVI,perseus,poupart,jiparr,sarsop,pointSurvey} have proved successful in solving POMDPs with large state spaces. Solving active perception POMDPs pose a different challenge: as the number of sensors grows, the size of the action space $\binom{N}{K}$ grows exponentially with it. Current POMDP solvers fail to address scalability in the action space of a POMDP. We propose a new \emph{point-based} planning method that scales much better in the number of sensors for such POMDPs. The main idea is to replace the maximization operator in the Bellman optimality equation with \emph{greedy maximization} in which a subset of sensors is constructed iteratively by adding the sensor that gives the largest marginal increase in value.

We present theoretical results bounding the error in the value functions computed by this method. We prove that, under certain conditions including \emph{submodularity}, the value function computed using POMDP backups based on greedy maximization has bounded error. We achieve this by extending the existing results \citep{Nemhauser} for the greedy algorithm, which are valid only for a single time step, to a full sequential decision making setting where the greedy operator is employed multiple times over multiple time steps. In addition, we show that the conditions required for such a guarantee to hold are met, or approximately met, if the reward is defined using negative belief entropy.

Finally, we present a detailed empirical analysis on a real-life dataset from a multi-camera tracking system installed in a shopping mall. We identify and study the critical factors relevant to the performance and behavior of the agent in active perception tasks. We show that our proposed planner outperforms a myopic baseline and nearly matches the performance of existing point-based methods while incurring only a fraction of the computational cost, leading to much better scalability in the number of cameras.

\section{Related Work}
Sensor selection as an active perception task has been studied in many contexts. Most work focus on either open-loop or myopic solutions, e.g., \citep{ActiveSensingSensor}, \citep{Spaan09icaps}, \citep{ApproxDynamProg}, \citep{convexOpt}. 
\cite{ActiveSensingSensor} proposes a Monte-Carlo approach that mainly focuses on a myopic solution. \cite{ApproxDynamProg} and \cite{convexOpt} developed planning methods that can provide long-term but open-loop policies. 
By contrast, a POMDP-based approach enables a closed-loop, non-myopic approach can lead to a better performance when the underlying state of the world changes over time.
\cite{Spaan08pomdp}, \cite{Spaan09icaps}, \cite{SpaanNRS} and \cite{natarajan2012} also consider a POMDP-based approach to active and cooperative active perception. However, they consider an objective function that conditions on state and not on belief, as the belief-dependent rewards in POMDP break the PWLC property of the value function. They use point-based methods \citep{perseus} for solving the POMDPs. While recent point-based methods \citep{pointSurvey} for solving POMDPs scale reasonably in the state space of POMDPs, they do not address the scalability in the action and observation space of a POMDP.

Greedy PBVI focuses specially on the scalability in the action space of an active perception POMDP and provides better scalability by leveraging greedy maximization.
Traditionally, POMDPs require the reward function to be defined as a function of the state. However, for active perception POMDPs, the objective is to reduce the uncertainty in the belief of the agent. 

In recent years, applying greedy maximization to submodular functions has become a popular and effective approach to sensor placement/selection \citep{krause05optimal,krause07nectar,KZijcai09}. However, such work focuses on myopic or fully observable settings and thus does not enable the long-term planning required to cope with dynamic state in a POMDP. 

\emph{Adaptive submodularity} \citep{adaptiveSubmod} is a recently developed extension that addresses these limitations by allowing action selection to condition on previous observations. However, it assumes a static state and thus cannot model the dynamics of a POMDP across timesteps. Therefore, in a POMDP, adaptive submodularity is only applicable \emph{within} a timestep, during which state does not change but the agent can sequentially add sensors to a set.  In principle, adaptive submodularity could enable this intra-timestep sequential process to be adaptive, i.e., the choice of later sensors could condition on the observations generated by earlier sensors. However, this is not possible in our setting because (a) we assume that, due to computational costs, all sensors must be selected simultaneously; (b) information gain is not known to be adaptive submodular \citep{submodSurrogate}. Consequently, our analysis considers only classic, non-adaptive submodularity.  

To our knowledge, our work is the first to establish the sufficient conditions for the submodularity of POMDP value functions for active perception POMDPs and thus leverage greedy maximization to scalably compute bounded approximate policies for dynamic sensor selection modeled as a full POMDP.

\section{Background}

In this section, we provide background on POMDPs, active perception POMDPs and solution methods for POMDPs.

\subsection{Partially Observable Markov Decision Processes}

POMDPs provide a decision-theoretic framework for modeling partial
observability and dynamic environments. Formally, a POMDP is defined by a tuple $ \langle S, A, \Omega, T, O, R, b_{0}, h \rangle$. At each time step, the environment is in a state $s \in S$, the agent takes an action $a \in A$ and receives a reward whose expected value is $R(s,a)$, and the system transitions to a new state $s' \in S$ according to the transition function $T(s,a,s') = \Pr(s'|s,a)$. Then, the agent receives an observation $z \in \Omega$ according to the observation function $O(s',a,z) = \Pr(z|s',a)$. Starting from an initial belief $b_0$, the agent maintains a \emph{belief} $b(s)$ about the state which is a probability distribution over all the possible states. The number of time steps for which the decision process lasts, i.e., the horizon is denoted by $h$. If the agent took action $a$ in belief $b$ and got an observation $z$, then the updated belief $b^{a,z}(s)$ can be computed using Bayes rule.
A policy $\pi$ specifies how the agent acts in each belief.  
Given $b(s)$ and $R(s,a)$, one can compute a belief-based reward, $\rho(b,a)$ as:
\begin{equation} \label{immRew}
 	\rho(b,a) =  \sum_{s} b(s)R(s,a).
\end{equation}

The $t$-step \emph{value function} of a policy $V^{\pi}_{t}$ is defined as the expected future discounted reward the agent can gather by following $\pi$ for next $t$ steps. $V^{\pi}_{t}$ can be characterized recursively using the \emph{Bellman equation}: 
\begin{equation}
V_{t}^{\pi}(b)  \triangleq \Bigg[ \rho(b,a_{\pi}) +  \sum_{z\in\Omega} \Pr(z|a_{\pi},b)V_{t-1}^{\pi}(b^{a_{\pi},z})\Bigg],
\end{equation}
where $a_\pi = \pi(b)$ and $V^{\pi}_{0}(b) = 0$. 
The action-value function $Q^{\pi}_{t}(b,a)$ is the value of taking action $a$ and following $\pi$ thereafter: 
\begin{equation} \label{eq:q-def}
Q^{\pi}_{t}(b,a) \triangleq \rho(b,a) + \sum_{z \in \Omega} \Pr(z| a,b)V_{t-1}^{\pi}(b^{a,z}).
\end{equation}
The policy that maximizes $V^{\pi}_{t}$ is called the \emph{optimal policy} $\pi^{*}$ and the corresponding value function is called the \emph{optimal value function} $V^{*}_{t}$.
The \emph{optimal value function} $V^{*}_{t}(b)$ can be characterized recursively as:
\begin{equation} \label{eq:boe}
\begin{split}
V_{t}^{*}(b) &= \max_{a}\Bigg[ \rho(b,a) +  \sum_{z\in\Omega} \Pr(z|a,b)V_{t-1}^{*}(b^{a,z})\Bigg].
\end{split}
\end{equation}
We can also define \emph{Bellman optimality operator} $\mathfrak{B}^*$:
\begin{equation}
(\mathfrak{B}^{*}V_{t-1})(b) = \max_{a} [\rho(b,a) + \sum_{z \in \Omega} \Pr(z|a,b)V_{t-1}(b^{z,a})], \nonumber
\end{equation}
and write \eqref{eq:boe} as $V^{*}_{t}(b) = (\mathfrak{B}^{*}V^{*}_{t-1})(b).$ 

\begin{figure}
\begin{center}
\includegraphics[scale=0.12]{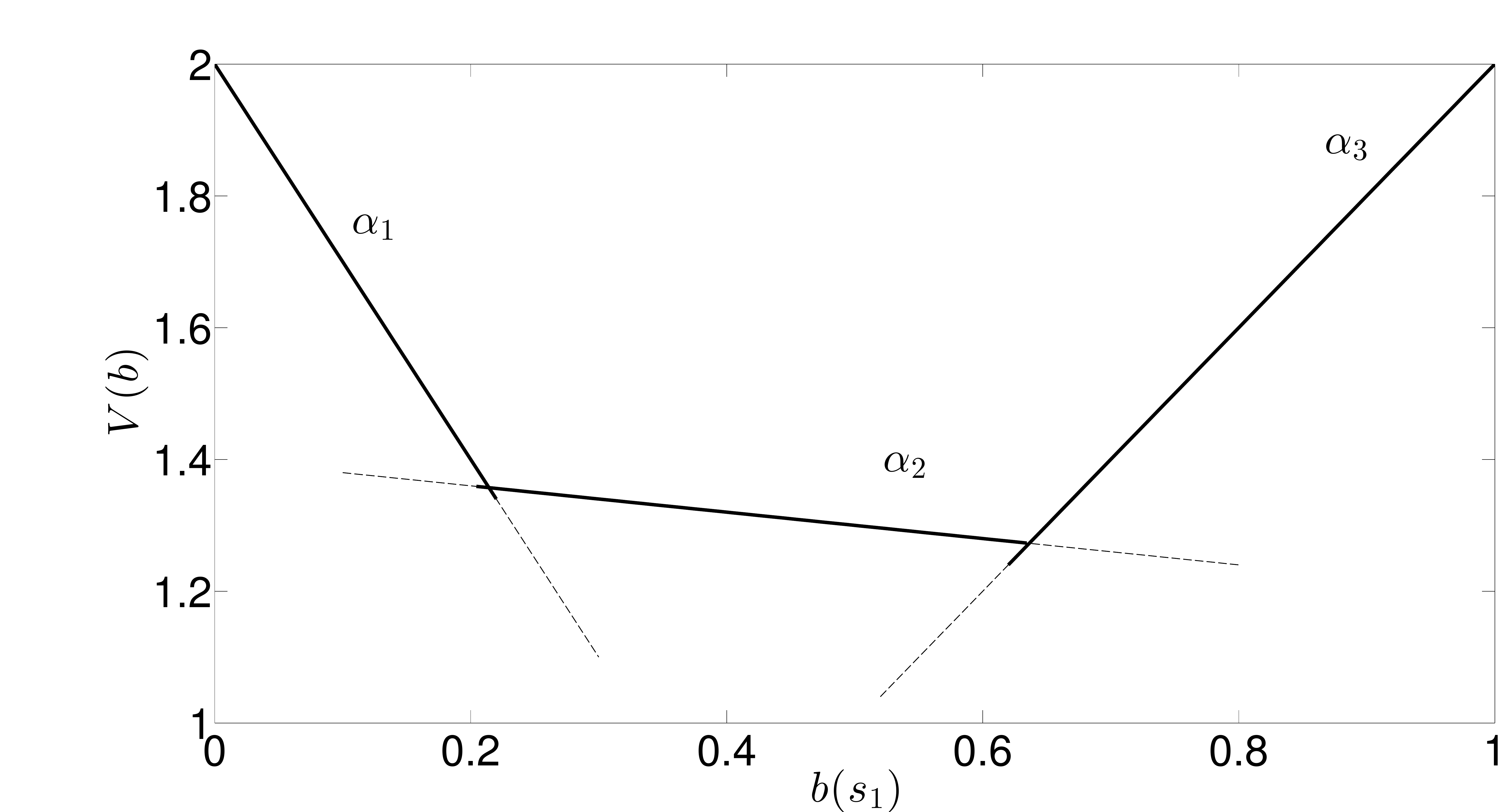}
    \caption{Illustration of the PWLC property of the value function. The value function is the upper surface indicated by the solid lines.}
\end{center}
\label{fig:pwlc1}
\end{figure}

An important consequence of these equations is that the value function is \emph{piecewise-linear and convex} (PWLC), as shown in Figure 1, a property exploited by most POMDP planners. \cite{Sondik71} showed that a PWLC value function at any finite time step $t$ can be expressed as a set of vectors: $\Gamma_{t} = \{ \alpha_{0}, \alpha_{1}, \dots, \alpha_{m} \}$. Each $\alpha_i$ represents an $|S|$-dimensional hyperplane defining the value function over a bounded region of the belief space. The value of a given belief point can be computed from the vectors as: $V_{t}^{*}(b) = \max_{\alpha_{i} \in \Gamma_t} \sum_{s} b(s)\alpha_{i}(s)$.

\subsection{POMDP Solvers}
Exact methods like Monahan's enumeration algorithm \citep{monahan1982} computes the value function for all possible belief points by computing the optimal $\mathbf{\Gamma}_{t}$. 
Point-based planners \citep{anytimePBVI,pointSurvey,perseus}, on the other hand, avoid the expense of solving for all belief points by computing $\Gamma_{t}$ only for a set of sampled beliefs $B$. Since exact POMDP solvers \citep{Sondik71,monahan1982} are intractable for all but the smallest POMDPs, we focus on point-based methods here. Point-based methods compute $\Gamma_{t}$ using the following recursive algorithm.

At each iteration (starting from t = 1), for each action $a$ and observation $z$, an intermediate $\Gamma^{a,z}_{t}$ is computed from $\Gamma_{t-1}$:
\begin{equation} \label{eq:gamma-az}
\Gamma_{t}^{a,z} = \{ \alpha_{i}^{a,z} : \alpha_{i} \in \Gamma_{t-1} \},
\end{equation}
 Next, $\Gamma^{a}_{t}$ is computed only for the sampled beliefs, i.e., $\Gamma^{a}_{t} = \{\alpha^a_{b}: b \in B\}$, where:
\begin{equation} \label{eq:pbvistep}
\alpha_{b}^{a} = \Gamma^{a} + \sum_{z \in \Omega} \argmax_{\alpha \in \Gamma_{t}^{a,z}} \sum_{s'} b(s')\alpha(s').
\end{equation}
Finally, the best $\alpha$-vector for each $b \in B$ is selected:
\begin{equation} \label{eq:pbvi2ndstep}
\alpha_{b} =  \argmax_{\alpha_{b}^{a}} \sum_{s'} b(s')\alpha_{b}^{a}(s'), \\
\end{equation}
\begin{equation} \label{eq:pbvifinstep}
\Gamma_{t} =  \cup_{b \in B} \alpha_{b}.
\end{equation}

The above algorithm at each timestep $t$, generates $|A_n||\Omega||\Gamma_{t-1}|$ alpha vectors in $\mathcal{O}(|S|^2 |A||\Omega||\Gamma_{t-1}|)$ time and then reduces them to $|B|$ vectors in $\mathcal{O}(|S||B||A||\Omega||\Gamma_{t-1}|)$ \citep{anytimePBVI}.

\section{Active Perception POMDP}
The goal in an active perception POMDP is to reduce uncertainty about a \emph{feature of interest} that is not directly observable. In general, the feature of interest may be only part of the state, e.g., if a surveillance system cares only about people's positions, not their velocities, or higher-level features derived from the state. However, for simplicity, we focus on the case where the feature of interest is just the state $s$\footnote{We make this assumption without loss of generality. The following sections will make it clear that none of our results require this assumption.} of the POMDP. For simplicity, we also focus on \emph{pure} active perception tasks in which the agent's only goal is to reduce uncertainty about the state, as opposed to hybrid tasks where the agent may also have other goals. For such cases, \emph{hybrid} rewards \citep{eck}, which combine the advantage of belief-based and state-based rewards, are appropriate. Although not covered in this article, it is straightforward to extend our results to hybrid tasks \citep{SpaanJaamas}.

We model the active perception task as a POMDP in which an agent must choose a subset of available sensors at each time step. We assume that all selected sensors must be chosen simultaneously, i.e. it is not possible within a timestep to condition the choice of one sensor on the observations generated by another sensor. This corresponds to the common setting where generating each sensor's observation is time consuming, e.g., in the surveillance task, because it requires applying expensive computer vision algorithms, and thus all the observations from the selected cameras must be generated in parallel. Formally, an active perception POMDP has the following components: 

\begin{itemize}
\item Actions $\textbf{a} = \langle a_{1} \dots a_{N} \rangle$ are vectors of $N$ binary \emph{action features}, each of which specifies whether a given sensor is selected or not. For each $\textbf{a}$, we also define its set equivalent $\mathfrak{a} = \{i : a_i = 1\}$, i.e., the set of indices of the selected sensors. Due to the resource constraints, the set of all actions $A = \{\mathfrak{a} : |\mathfrak{a}| \leq K \}$ contains only sensor subsets of size $K$ or less. $A^+=\{1,\ldots,N\}$ indicates the set of all sensors. 

\item Observations $\mathbf{z} = \langle z_{1} \dots z_{N} \rangle$ are vectors of $N$ \emph{observation features}, each of which specifies the sensor reading obtained by the given sensor.  If sensor $i$ is not selected, then $z_i = \emptyset$. The set equivalent of $\mathbf{z}$ is $\mathfrak{z} = \{z_i : z_i \neq \emptyset\}$. To prevent ambiguity about which sensor generated which observation in $\mathfrak{z}$, we assume that, for all $i$ and $j$, the domains of $z_i$ and $z_j$ share only $\emptyset$. This assumption is only made for notational convenience and does not restrict the applicability of our methods in any way.  
\end{itemize}

For example, in the surveillance task, $\mathfrak{a}$ indicates the set of cameras that are active and $\mathfrak{z}$ are the observations received from the cameras in $\mathfrak{a}$. The model for the sensor selection problem for surveillance task is shown in Figure \ref{fig:sensorselection}. Here, we assume that the actions involve only selecting $K$ out of $N$ sensors. The transition function is thus independent of the actions, as selecting sensors cannot change the state. However, as we outline in the later subsection (6.4), it is possible to extend our results to general active perception POMDPs with arbitrary transition functions, that can model, e.g., mobile sensors that, by moving, change the state.

\begin{figure}
\centering
    \includegraphics[scale=0.42]{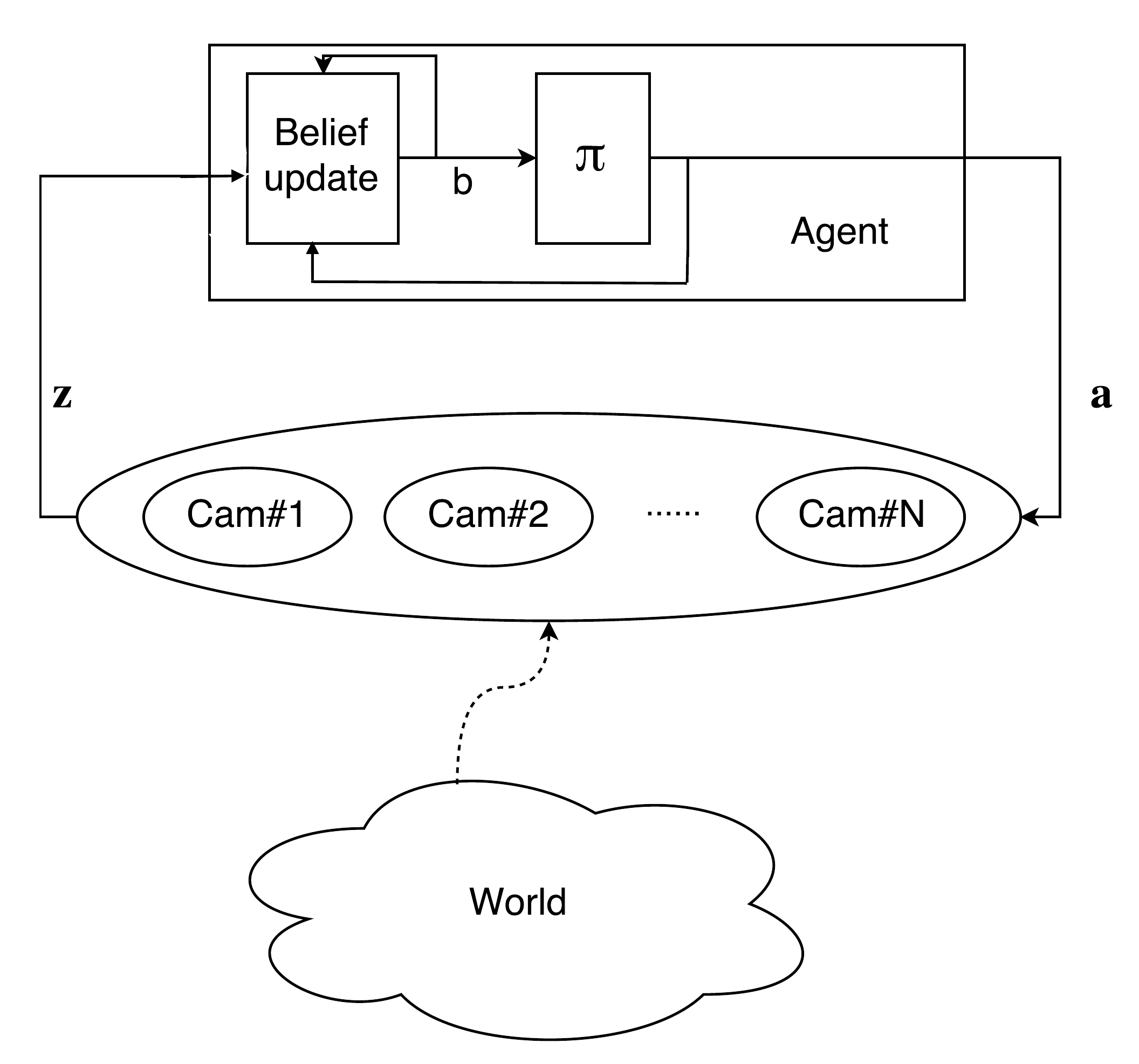}
    \caption{Model for sensor selection problem}
    \label{fig:sensorselection}
\end{figure}

A challenge in these settings is properly formalizing the reward function.  Because the goal is to reduce the uncertainty, reward is a direct function of the belief, not the state, i.e., the agent has no preference for one state over another, so long as it knows what that state is.  Hence, there is no meaningful way to define a state-based reward function $R(s,\mathbf{a})$.  Directly defining $\rho(b,\mathbf{a})$ using, e.g., negative \emph{belief entropy}: $- H_{b}(s) = \sum_s b(s) \log(b(s))$ results in a value function that is not piecewise-linear.  Since $\rho(b,\mathbf{a})$ is no longer a convex combination of a state-based reward function, it is no longer guaranteed to be PWLC, a property most POMDP solvers rely on.
In the following subsections, we describe two recently proposed frameworks designed to address this problem.

\subsection{$\rho$POMDPs}

A $\rho$POMDP \citep{Mauricio}, defined by a tuple  $\langle S,A,T,\Omega,O,\Gamma_{\rho},b_{0}, h \rangle$, is a normal POMDP except that the state-based reward function $R(s,\mathbf{a})$ has been omitted and $\Gamma_{\rho}$ has been added. $\Gamma_{\rho}$ is a set of vectors, that defines the immediate reward for $\rho$POMDP. Since we consider only pure active perception tasks, $\rho$ depends only on $b$, not on $\mathbf{a}$ and can be written as $\rho(b)$. Given $\Gamma_{\rho}$, $\rho(b)$ can be computed as: $\rho(b)  = \max_{\alpha \in \Gamma_{\rho}} \sum_{s} b(s)\alpha(s)$. If the true reward function is not PWLC, e.g., negative belief entropy, it can be approximated by defining $\Gamma_{\rho}$ as a set of vectors, each of which is tangent to the true reward function. Figure \ref{fig:tangents} illustrates approximating negative belief entropy with different numbers of tangents. 

Solving a $\rho$POMDP\footnote{Arguably, there is a counter-intuitive relation between the general class of
POMDPs and the sub-class of pure active perception problems: on the one hand, the class of POMDPs is a more general set of problems, and it is intuitive to assume that there might be harder problems in the class. On the other hand, many POMDP problems admit a representation of the value function using a finite set of vectors. In contrast, the use of entropy would require an infinite number of vectors to merely represent the reward function. Therefore, even though we consider a specific sub-class of POMDPs, this class has properties that make it difficult to address using existing methods.} requires a minor change to the existing algorithms.  In particular, since $\Gamma_{\rho}$ is a set of vectors, instead of a single vector, an additional cross-sum is required to compute $\Gamma_{t}^{\mathbf{a}}$: $\Gamma_{t}^{\mathbf{a}} = \Gamma_{\rho} \oplus \Gamma_{t}^{\mathbf{a},\mathbf{z}_{1}} \oplus \Gamma_{t}^{\mathbf{a},\mathbf{z}_{2}} \oplus \dots$. 
\begin{figure}
\begin{center}
\includegraphics[width=0.45\textwidth]{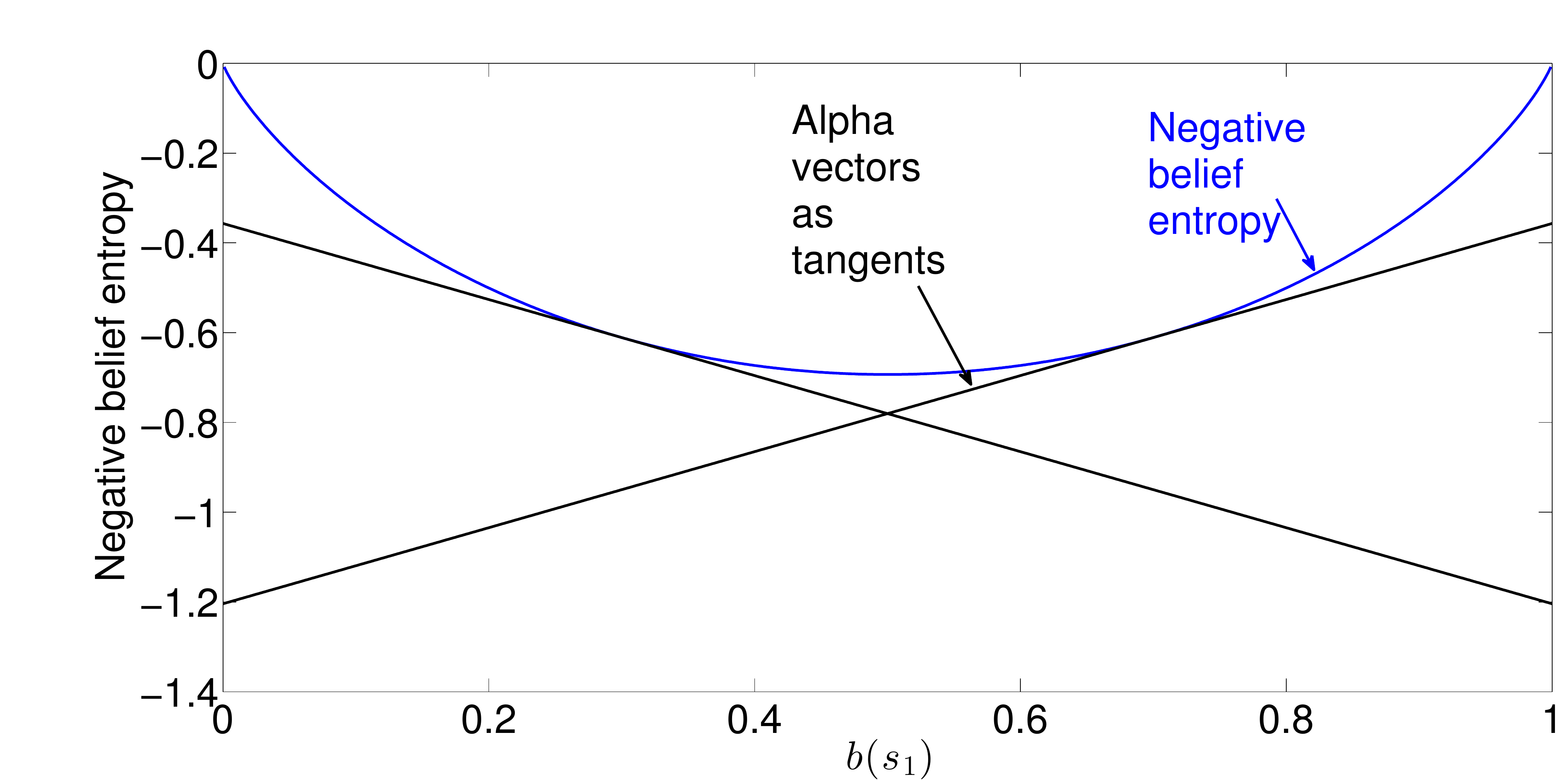} 
\includegraphics[width=0.45\textwidth]{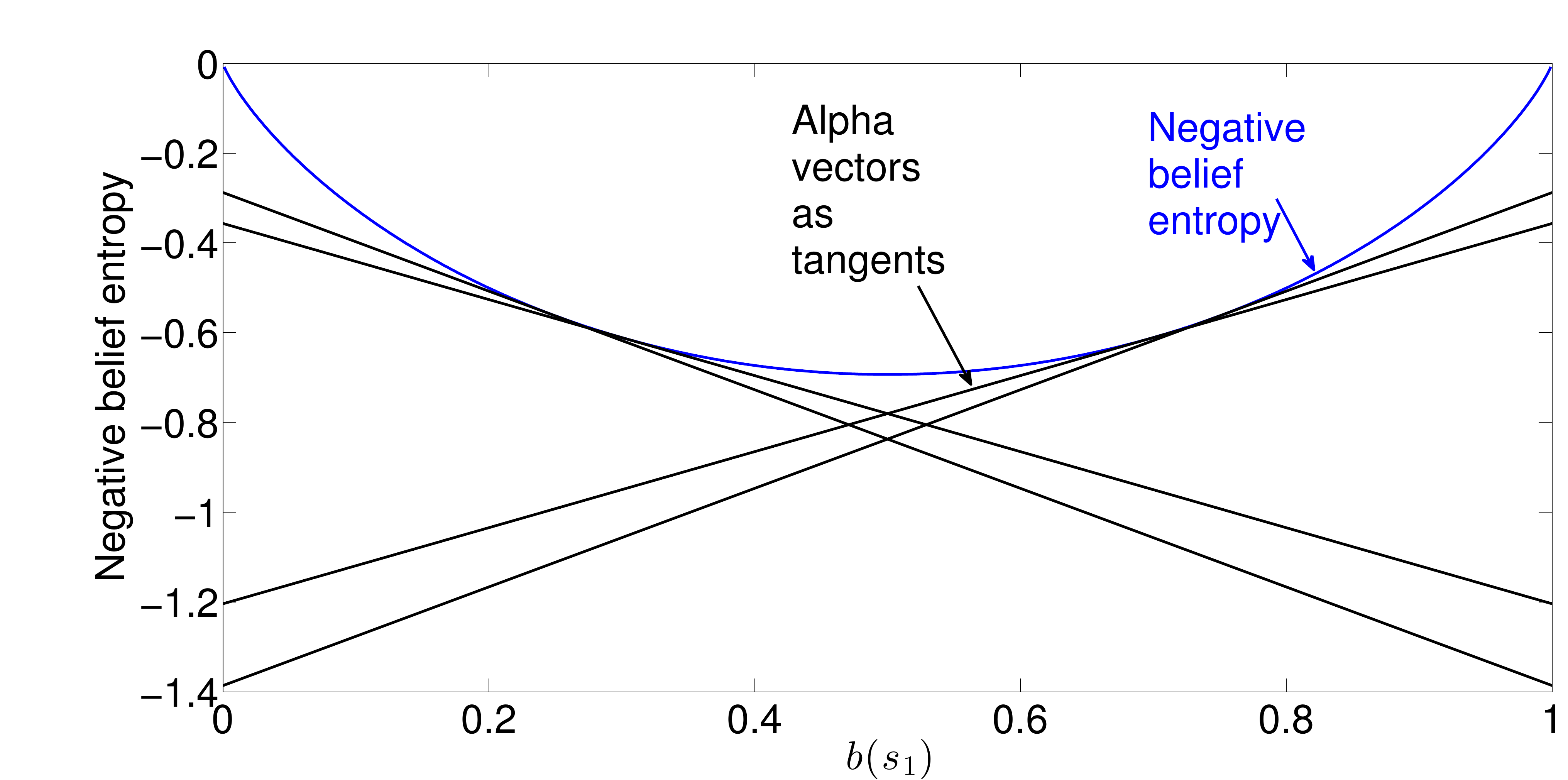}
\includegraphics[width=0.45\textwidth]{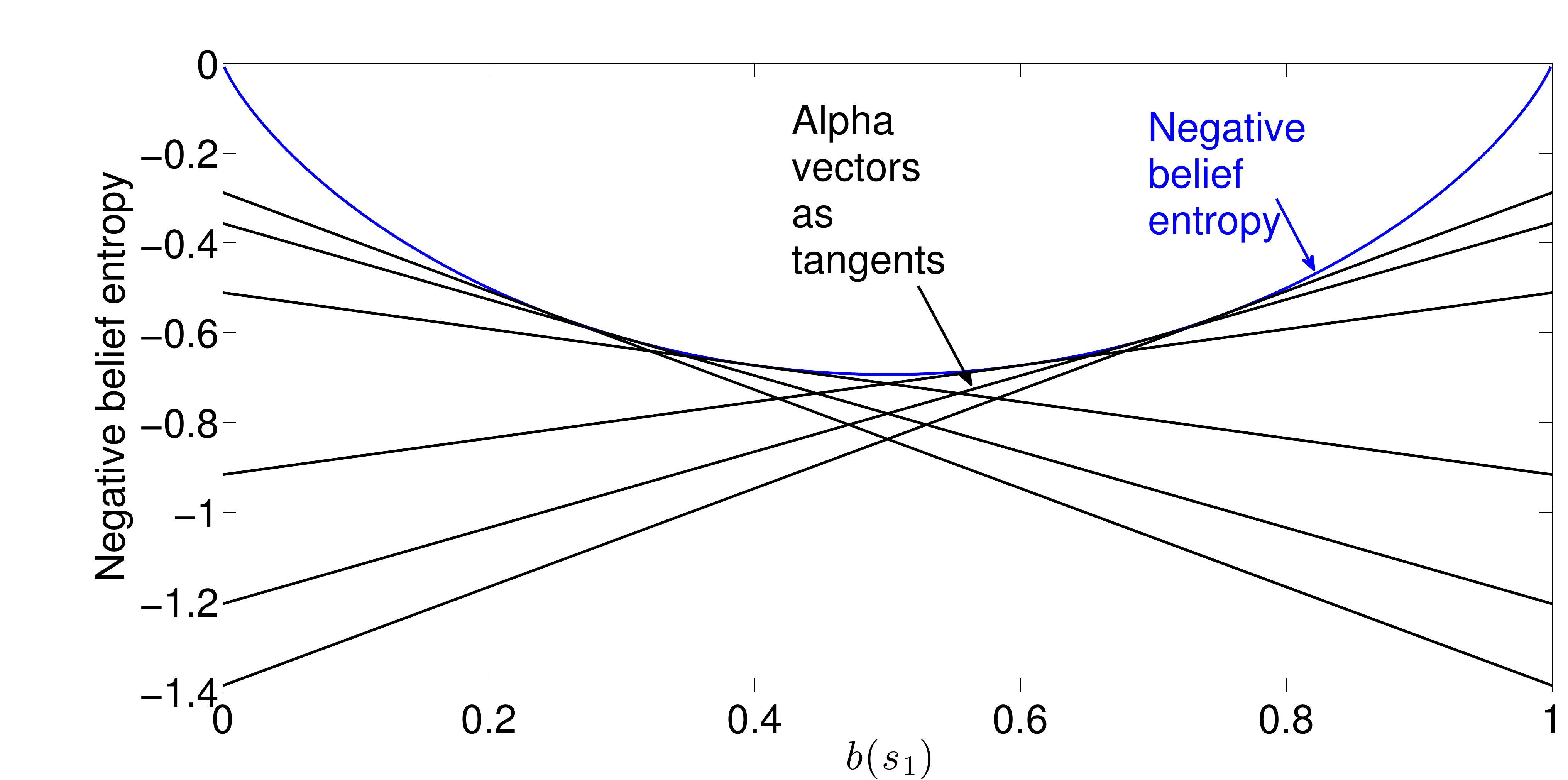}
\end{center}
\caption{Defining $\Gamma_{\rho}^{a}$ with different sets of tangents to the
  negative belief entropy curve in a
  2-state POMDP.
  }
\label{fig:tangents}
\end{figure}
\cite{Mauricio} showed that the error in the value function computed by this approach, relative to the true reward function, whose tangents were used to define $\Gamma_{\rho}$, is bounded. However, the additional cross-sum increases the computational complexity of computing $\Gamma_{t}^{\mathbf{a}}$ to $\mathcal{O}(|S||A||\Gamma_{t-1}||\Omega||B||\Gamma_{\rho}|)$ with point-based methods.

Though $\rho$POMDP do not put any constraints on the definition of $\rho$, we restrict the definition of $\rho$ for an active perception POMDP to be a set of vectors ensuring that $\rho$ is PWLC, which in turn ensures that the value function is PWLC. This is not a severe restriction because solving a $\rho$POMDP using \emph{offline planning} requires a PWLC approximation of $\rho$  anyway. 

\subsection{POMDPs with Information Rewards} \label{sec:pomdp-ir}
Spaan et al.\ proposed \emph{POMDPs with information rewards} (POMDP-IR), an alternative framework for modeling active perception tasks that relies only on the standard POMDP.  Instead of directly rewarding low uncertainty in the belief, the agent is given the chance to make predictions about the hidden state and rewarded, via a standard state-based reward function, for making accurate predictions.  Formally, a POMDP-IR is a POMDP in which each action $\mathrm{a} \in A$ is a tuple $\langle \mathbf{a}_n, a_p \rangle$ where $\mathbf{a}_n \in A_n$ is a \emph{normal action}, e.g., moving a robot or turning on a camera (in our case $\mathbf{a}_{n}$ is $\mathbf{a}$), and $a_p \in A_p$ is a \emph{prediction action}, which expresses predictions about the state.  The joint action space is thus the Cartesian product of $A_n$ and $A_p$, i.e., $A = A_n \times A_p$.

Prediction actions have no effect on states or observations but can trigger rewards via the standard state-based reward function $R(s,\langle \mathbf{a}_n, a_p \rangle)$. While there are many ways to define $A_p$ and $R$, a simple approach is to create one prediction action for each state, i.e., $A_p = S$, and give the agent positive reward if and only if it correctly predicts the true state:
 \begin{equation}\label{eq:example}
     R(s,\langle \mathbf{a}_n, a_p \rangle) =
    \begin{cases}
      1, & \text{if}\ s = a_p \\
      0, & \text{otherwise.}
    \end{cases}
  \end{equation}
  
Thus, POMDP-IR indirectly rewards beliefs with low uncertainty, since these enable more accurate predictions and thus more expected reward.  Furthermore, since a state-based reward function is explicitly defined, $\rho$ can be defined as a convex combination of $R$, as in \eqref{immRew}, guaranteeing a PWLC value function, as in a regular POMDP. Thus, a POMDP-IR can be solved with standard POMDP planners. However, the introduction of prediction actions leads to a blowup in the size of the joint action space $|A| = |A_n||A_p|$ of POMDP-IR. Replacing $|A|$ with $|A_n||A_p|$ in the analysis yields a complexity of computing $\Gamma_{t}^{\mathbf{a}}$ for POMDP-IR of $\mathcal{O}(|S||A_n||\Gamma_{t-1}||\Omega||B||A_p|)$ for point-based methods.

\begin{figure}
\begin{center}
\includegraphics[width=0.25\textwidth]{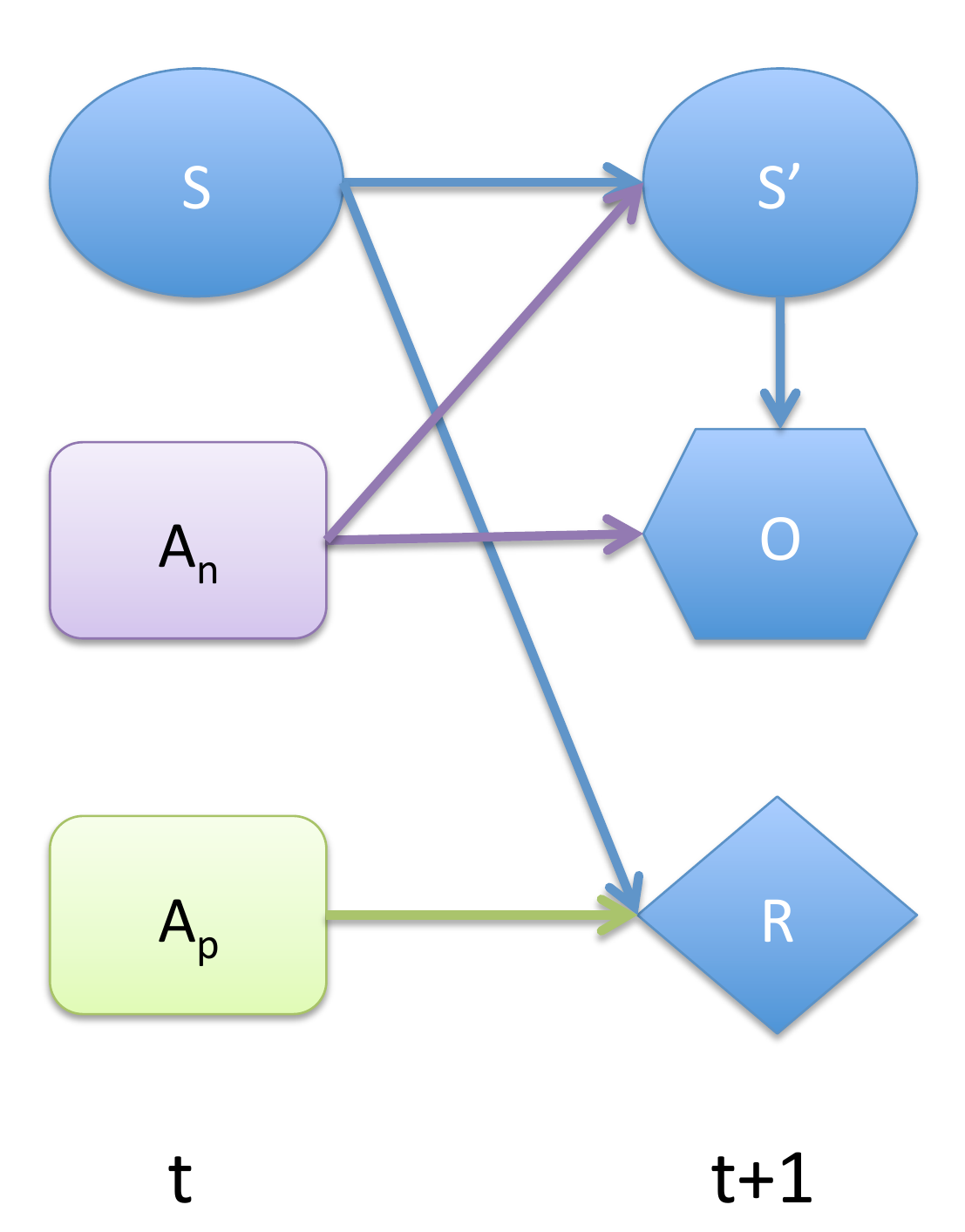}
\caption{Influence diagram for POMDP-IR.}
\label{fig:ID}
\end{center}
\end{figure}

Note that, though not made explicit in \citet{SpaanJaamas}, several independence properties are inherent to the POMDP-IR framework, as shown in Figure \ref{fig:ID}. Specifically, the two important properties are (a) in our setting the reward function is independent of the normal actions; (b) the transition and the observation function are independent of the normal actions. Although POMDP-IR can model \emph{hybrid rewards}, where in addition to prediction actions, normal actions can reward agent as well \citep{SpaanJaamas}, in this article, because we focus on pure active perception, the reward function $R$ is independent of the normal actions. Furthermore, state transitions and observations are independent of the prediction actions. In Section \ref{sec:decompMax}, we introduce a new technique to show that these independence properties can be exploited to solve a POMDP-IR much more efficiently and thus avoid the blowup in the size of the action space caused by the introduction of the prediction actions. Although, the reward function in our setting is independent of the normal actions, the main results we present in this article are not dependent on this property and can be easily extended or applied to cases where the reward is dependent on the normal actions.

\section{$\rho$POMDP and POMDP-IR Equivalence} \label{sec:equiv}
$\rho$POMDP and POMDP-IR offer two perspectives on modeling active perception tasks. $\rho$POMDP starts from a ``true'' belief-based reward function such as the negative entropy and then seeks to find a PWLC approximation via a set of tangents to the curve. By contrast, POMDP-IR starts from the queries that the user of the system will pose, e.g., ``What is the position of everyone in the room?'' or ``How many people are in the room'' and creates prediction actions that reward the agent correctly for answering such queries. In this section we establish the relationship between these two frameworks by proving the \emph{equivalence} of $\rho$POMDP and POMDP-IR. By equivalence of $\rho$POMDP and POMDP-IR, we mean that given a $\rho$POMDP and a policy, we can construct a corresponding POMDP-IR and a policy such that the value function for both the policies is exactly the same. We show this equivalence by starting with a $\rho$POMDP and a policy and introducing a \emph{reduction} procedure for both $\rho$POMDP and the policy (and vice-versa). Using the reduction procedure, we reduce the $\rho$POMDP to a POMDP-IR and the policy for $\rho$POMDP to an equivalent policy for POMDP-IR. We then show that the value function, $V^{\pi}_{t}$ for the $\rho$POMDP we started with and the reduced POMDP-IR is the same for the given and the reduced policy. To complete our proof, we repeat the same process by starting with a POMDP-IR and then reducing it to a $\rho$POMDP. We show that the value function $V^{\pi}_{t}$ for the POMDP-IR and the corresponding $\rho$POMDP is the same.
\begin{definition}
Given a $\rho$POMDP $\mathbf{M}_{\rho} = \langle S,A_{\rho}, \Omega, T_{\rho}, O_{\rho}, \Gamma_{\rho}, b_{0}, h \rangle$ the \textsc{reduce-pomdp-$\rho$-IR}($\mathbf{M}_{\rho})$ produces a POMDP-IR $\mathbf{M}_{\mathit{IR}}$ = $\langle S,A_{\mathit{IR}}, \Omega, T_{\mathit{IR}}, O_{\mathit{IR}}, R_{\mathit{IR}}, b_{0}, h \rangle$ via the following procedure.

\begin{itemize}

\item The set of states, set of observations, initial belief and horizon remain unchanged. Since the set of states remain unchanged, the set of all possible beliefs is also the same for $\mathbf{M}_{\mathit{IR}}$ and $\mathbf{M}_{\rho}$. 

\item The set of normal actions in $\mathbf{M}_{\mathit{IR}}$ is equal to the set of actions in $\mathbf{M}_{\rho}$, i.e., $A_{n,\mathit{IR}} = A_{\rho}$;

\item The set of prediction actions $A_{p,\mathit{IR}}$ in $\mathbf{M}_{\mathit{IR}}$ contains one
  prediction action for each $\alpha_{\rho}^{a_p} \in \Gamma_{\rho}$.

\item The transition and observation functions in $\mathbf{M}_{\mathit{IR}}$ behave the same as in $\mathbf{M}_{\rho}$ for each $\mathbf{a}_n$ and ignore the $a_p$, i.e., for all $\mathbf{a}_n \in A_{n,{IR}}$: $T_{\mathit{IR}}(s, \mathbf{a}_n, s') = T_{\rho}(s,\mathbf{a},s')$ and $O_{\mathit{IR}}(s', \mathbf{a}_n, \mathbf{z}) = O_{\rho}(s',\mathbf{a,z})$, where $\mathbf{a} \in A_{\rho}$ corresponds to $\mathbf{a}_n$.

\item The reward function in $\mathbf{M}_{\mathit{IR}}$ is defined such that
$\forall a_p \in A_p$, $R_{\mathit{IR}}(s,a_p) = \alpha_{\rho}^{a_p}(s)$, where $\alpha_{\rho}^{a_p}$
  is the $\alpha$-vector corresponding to $a_p$. 
\end{itemize}
\end{definition}

For example, consider a $\rho$POMDP with 2 states, if $\rho$ is defined using tangents to belief entropy at $b(s_1) = 0.3$ and $b(s_1) = 0.7$.  When reduced to a POMDP-IR, the resulting reward function gives a small negative reward for correct predictions and a larger one for incorrect predictions, with the magnitudes determined by the value of the tangents when $b(s_1) = 0$ and $b(s_1)=1$:
\begin{equation} \label{eq:ex}
     R_{\mathit{IR}}(s,a_p) =
    \begin{cases}
      -0.35, & \text{if}\ s = a_p \\
      -1.21, & \text{otherwise.}
    \end{cases}
\end{equation}
This is illustrated in Figure \ref{fig:tangents} (top). 

\begin{definition}
Given a policy $\pi_{\rho}$ for a $\rho$POMDP, $\mathbf{M}_{\rho}$, the \textsc{reduce-policy-$\rho$-IR}($\pi_{\rho}$) procedure produces a policy $\pi_{\mathit{IR}}$ for a POMDP-IR as follows.  For all $b$,
\begin{equation}
\pi_{\mathit{IR}}(b) = \langle \pi_{\rho}(b), \argmax_{a_p} \sum_{s} b(s)R(s,a_p) \rangle.
\end{equation}
That is, $\pi_{\mathit{IR}}$ selects the same normal action as $\pi_{\rho}$ and the prediction action that maximizes expected immediate reward.
\end{definition}

Using these definitions, we prove that solving $\mathbf{M}_{\rho}$ is the same as solving $\mathbf{M}_{\mathit{IR}}$.
\begin{theorem} \label{th:equiv1}
Let $\mathbf{M}_{\rho}$ be a $\rho$POMDP and $\pi_{\rho}$ an arbitrary policy for $\mathbf{M}_{\rho}$.  Furthermore let $\mathbf{M}_{\mathit{IR}}$ = \textsc{reduce-pomdp-$\rho$-IR}$(\mathbf{M}_{\rho})$ and $\pi_{\mathit{IR}}$ = \textsc{reduce-policy-$\rho$-IR}$(\pi_{\rho})$. Then, for all $b$,
\begin{equation}
V_{t}^{\mathit{IR}}(b) = V_{t}^{\rho}(b),
\end{equation}
where $V_{t}^{\mathit{IR}}$ is the $t$-step value function for $\pi_{\mathit{IR}}$ and $V_{t}^{\rho}$ is the $t$-step value function for $\pi_{\rho}$.
\end{theorem}

\begin{proof}
See Appendix.    
\qed
\end{proof}

\begin{definition}
Given a POMDP-IR $\mathbf{M}_{\mathit{IR}}$ = $\langle S,A_{\mathit{IR}}, \Omega, T_{\mathit{IR}}, O_{\mathit{IR}}, R_{\mathit{IR}}, b_{0}, h \rangle$
 the \textsc{reduce-pomdp-IR-$\rho$}($\mathbf{M}_{\mathit{IR}})$ produces a $\rho$POMDP $\mathbf{M}_{\rho} = \langle S,A_{\rho}, \Omega, T_{\rho}, O_{\rho}, \Gamma_{\rho}, b_{0}, h \rangle$  via the following procedure.

\begin{itemize}

\item The set of states, set of observations, initial belief and horizon remain unchanged. Since the set of states remain unchanged, the set of all possible belief is also the same for $\mathbf{M}_{\mathit{IR}}$ and $\mathbf{M}_{\rho}$. 

\item The set of actions in $\mathbf{M}_{\rho}$ is equal to the set of normal actions in $\mathbf{M}_{\mathit{IR}}$, i.e., $A_{\rho} = A_{n,\mathit{IR}}$.

\item The transition and observation functions in $\mathbf{M}_{\rho}$ behave the same as in $\mathbf{M}_{\mathit{IR}}$ for each $\mathbf{a}_{n}$ and ignore the $a_p$, i.e., for all $\mathbf{a} \in A_{\rho}$: $T_{\rho}(s,\mathbf{a},s') = T_{\mathit{IR}}(s,\mathbf{a}_n,s')$ and $O_{\rho}(s',\mathbf{a},\mathbf{z}) = O_{\mathit{IR}}(s', \mathbf{a}_n, \mathbf{z})$ where $\mathbf{a}_n \in A_{n,\mathit{IR}}$ is the action corresponding to $\mathbf{a} \in A_{\rho}$.

\item The $\Gamma_{\rho}$ in $\mathbf{M}_{\rho}$ is defined such that, for each prediction action in $A_{p,\mathit{IR}}$, there is a corresponding $\alpha$ vector in $\Gamma_{\rho}$, i.e.,
$\Gamma_{\rho} = \{\alpha_{\rho}^{a_p}(s) : \alpha_{\rho}^{a_p}(s) = R(s,a_{p}) \mbox{ for each } a_{p} \in A_{p,\mathit{IR}} \}$. Consequently, by definition, $\rho$ is defined as: $\rho(b) = \max_{\alpha_{\rho}^{a_p}}[\sum_{s}b(s)\alpha_{\rho}^{a_p}(s)]$.

\end{itemize}
\end{definition}
\begin{definition}
Given a policy $\pi_{\mathit{IR}} = \langle \mathbf{a}_{n}, a_{p} \rangle$ for a POMDP-IR, $\mathbf{M}_{\mathit{IR}}$, the \textsc{reduce-policy-IR-$\rho$}($\pi_{\mathit{IR}}$) procedure produces a policy $\pi_{\rho}$ for a POMDP-IR as follows.  For all $b$,
\begin{equation}
\pi_{\rho}(b) = \pi_{\mathit{IR}}^{n}(b),
\end{equation}
\end{definition}

\begin{theorem} \label{th:equiv2}
Let $\mathbf{M}_{\mathit{IR}}$ be a POMDP-IR and $\pi_{\mathit{IR}} = \langle \mathbf{a}_{n}, a_p \rangle $ a policy for $\mathbf{M}_{\mathit{IR}}$, such that $a_{p} = \argmax_{a_{p}'}b(s)R(s,a_{p}')$.  Furthermore let $\mathbf{M}_{\rho}$ = \textsc{reduce-pomdp-IR-$\rho$}($\mathbf{M}_{\mathit{IR}})$ and $\pi_{\rho}$ = \textsc{reduce-policy-IR-$\rho$}($\pi_{\mathit{IR}})$. Then, for all $b$,
\begin{equation}
V_{t}^{\rho}(b) = V_{t}^{IR}(b),
\end{equation}
where $V_{t}^{\mathit{IR}}$ is the value of following $\pi_{\mathit{IR}}$ in $\mathbf{M}_{\mathit{IR}}$ and 
 $V_{t}^{\rho}$ is the value of following $\pi_{\rho}$ in $\mathbf{M}_{\rho}$.
\end{theorem}

\begin{proof}
See Appendix.
\qed
\end{proof}

The main implication of these theorems is that any result that holds for either $\rho$POMDP or POMDP-IR also holds for the other framework. For example, the results presented in Theorem 4.3 in \cite{Mauricio} that bound the error in the value function of $\rho$POMDP also hold for POMDP-IR. Furthermore, with this equivalence, the computational complexity of solving $\rho$POMDP and POMDP-IR comes out to be the same, since POMDP-IR can be converted into $\rho$POMDP (and vice-versa) trivially, without any significant blow-up in representation. Although we have proved the equivalence of $\rho$POMDP and POMDP-IR only for pure active perception task where the reward is solely conditioned on the belief, it is straightforward to extend it to hybrid active perception tasks, where the reward is conditioned both on belief and the state. Although, the resulting active perception POMDP for dynamic sensor selection is such that the action does not affect the state, the results from this section do not use that property at all and thus are valid for active perception POMDPs where an agent might take an action which can affect the state in the next time step.

\section{Decomposed Maximization for POMDP-IR} \label{sec:decompMax}
The POMDP-IR framework enables us to formulate uncertainty as an objective, but it does so at the cost of additional computations, as adding prediction actions enlarges the action space. The computational complexity of performing a point-based backup for solving POMDP-IR is $\mathcal{O}(|S|^2|A_n||A_p||\Omega||\Gamma_{t-1}|) + \mathcal{O}(|S||B||A_n||\Gamma_{t-1}||\Omega||A_p|)$. In this section, we present a new technique that exploits the independence properties of POMDP-IR, mainly that the transition function and the observation function are independent of the prediction actions, to reduce the computational costs. We also show that the same principle is applicable to $\rho$POMDPs.

The increased computational cost of solving POMDP-IR arises from the size of the action space, $|A_n||A_p|$. However, as shown in Figure \ref{fig:ID}, prediction actions only affect the reward function and normal actions only affect the observation and transition function. We exploit this independence to decompose the maximization in the Bellman optimality equation:
\begin{equation}
\begin{split}
V_{t}^{*}(b) &= \max_{\langle \mathbf{a}_n, a_p \rangle \in A} \Big[\sum_{s}b(s)R(s,a_p) \\ & \hspace{30mm }+ \sum_{\mathbf{z} \in \Omega}\Pr(\mathbf {z}|\mathbf{a}_n,b)V^{*}_{t-1}(b^{\mathbf{a}_n,\mathbf{z}})\Big]  \\
 &= \max_{a_p \in A_p} \sum_{s}b(s)R(s,a_p) \\ & \hspace{20mm}+ \max_{\mathbf{a}_n \in A_n} \sum_{\mathbf{z} \in \Omega}\Pr(\mathbf {z}|\mathbf{a}_n,b)V^{*}_{t-1}(b^{\mathbf{a}_n,\mathbf{z}})  \nonumber
\end{split}
\end{equation}

These decomposition can be exploited by point-based methods by computing $\Gamma_{t}^{a,z}$ only for normal actions, $\mathbf{a}_n$ and $\alpha^{a_p}$ only for prediction actions. That is, \eqref{eq:gamma-az} can be changed to: 
\begin{equation}
\Gamma_{t}^{\mathbf{a}_n,\mathbf{z}} = \{ \alpha_{i}^{\mathbf{a}_n,\mathbf{z}} : \alpha_{i} \in \Gamma_{t-1}\}.
\end{equation}
For each prediction action, we compute the vector specifying the immediate reward for performing the prediction action in each state: $\Gamma^{A_p} = \{\alpha^{a_p}\}$,
where $\alpha^{a_p}(s) = R(s,a_p)$ $\forall \ a_p \in A_p$.
The next step is to modify \eqref{eq:pbvistep} to separately compute the vectors maximizing expected reward induced by prediction actions and the expected return induced by the normal action: 

\begin{equation}
\begin{split}
\alpha^{\mathbf{a}_n}_{b} &= \argmax_{\alpha^{a_p} \in \Gamma^{A_p}} \sum_{s} b(s) \alpha^{a_p}(s) \\ & \hspace{20mm}+ \sum_{\mathbf{z}} \argmax_{\alpha^{\mathbf{a}_n,\mathbf{z}} \in \Gamma_{t}^{\mathbf{a}_n,\mathbf{z}}} \sum_{s} \alpha^{\mathbf{a}_n, \mathbf{z}}(s) b(s). \nonumber
\end{split}
\end{equation}

By decomposing the maximization, this approach avoids iterating over all $|A_n||A_p|$ joint actions. At each timestep $t$, this approach generates $|A_n||\Omega||\Gamma_{t-1}| + |A_p|$ backprojections in $\mathcal{O}(|S|^{2}|A_n||\Omega||\Gamma_{t-1}| + |S||A_p|)$ time and then prunes them to $|B|$ vectors, with a computational complexity of $\mathcal{O}(|S||B|(|A_p| + |A_n||\Gamma_{t-1}||\Omega|))$.

The same principle can be applied to $\rho$POMDP by changing \eqref{eq:pbvistep} such that it maximizes over immediate reward independently from the future return:
\begin{equation}
\begin{split}
\alpha^{\mathbf{a}}_{b} &= \argmax_{\alpha^{\rho} \in \Gamma_{\rho}}\sum_{s}b(s)\alpha_{\rho}^{a_p}(s)  \\ & \hspace{20mm}+ \sum_{\textbf{z}} \argmax_{\alpha^{\mathbf{a},\textbf{z}} \in \Gamma_{t}^{\mathbf{a},\textbf{z}}} \sum_{s} \alpha^{\mathbf{a},\textbf{z}}(s) b(s). \nonumber
\end{split}
\end{equation}
The computational complexity of solving $\rho$POMDP with this approach is $\mathcal{O}(|S|^{2}|A||\Omega||\Gamma_{t-1}| + |S||\Gamma_{\rho}|) + \mathcal{O}(|S||B|(|\Gamma_{\rho}| + |A||\Gamma_{t-1}||\Omega|)$. Thus, even though both POMDP-IR and $\rho$POMDP use extra actions or vectors to formulate belief-based rewards, they can both be solved at only minimal additional computational cost.

\section{Greedy PBVI} \label{sec:gpbvi}
The previous sections allow us to model the active perception task efficiently, such that the PWLC property of the value function is maintained. Thus, we can now directly employ traditional POMDP solvers that exploit this property to compute the optimal value function $V^{*}_{t}$.While point-based methods scale better in the size of the state space, they are still not practical for our needs as they do not scale in the size of the normal action space of active perception POMDPs. 

While the computational complexity of one iteration of PBVI is linear in the size of the action space $|A|$ of a POMDP, for an active perception POMDP, the action space is modeled as selecting $K$ out of the $N$ available sensors, yielding $|A|$ = $\binom{N}{K}$. For fixed $K$, as the number of sensors $N$ grows, the size of the action space and the computational cost of PBVI grows exponentially with it, making use of traditional POMDP solvers infeasible for solving active perception POMDPs. 

In this section, we propose \emph{greedy PBVI}, a new point-based planner for solving active perception POMDPs which scales much better in the size of the action space. To facilitate the explication of greedy PBVI, we now present the final step of PBVI, described earlier in \eqref{eq:pbvi2ndstep} and \eqref{eq:pbvifinstep}, in a different way. For each $b \in B$, and $\mathfrak{a} \in A$, we must find the best $\alpha^{\mathfrak{a}}_{b} \in \Gamma^{\mathfrak{a}}_{t}$.
\begin{equation} 
\alpha^{\mathfrak{a},*}_{b} = \argmax_{\alpha^{\mathfrak{a}}_{b} \in \Gamma^{\mathfrak{a}}_{t}} \sum_{s} \alpha^{\mathfrak{a}}_{b}(s)b(s),
\end{equation}
and simultaneously record its value $Q(b,\mathfrak{a}) = \sum_{s}\alpha^{\mathfrak{a},*}_{b}b(s)$. Then, for each $b$ we find the best vector across all actions: $\alpha_{b} = \alpha^{\mathfrak{a}^{*}}_{b}$, where
\begin{equation} \label{eq:pbvistep2}
\mathfrak{a}^{*} = \argmax_{\mathfrak{a} \in A} Q(b,\mathfrak{a}).
\end{equation}

The main idea of greedy PBVI is to exploit \emph{greedy maximization} \citep{Nemhauser}, an algorithm that operates on a set function $Q : 2^{X} \rightarrow \mathbb{R}$. Greedy maximization is much faster than full maximization as it avoids going over the $\binom{N}{K}$ choices and instead constructs a subset of $K$ elements iteratively. Thus, we replace the maximization operator in the Bellman optimality equation with greedy maximization. Algorithm~\ref{GreedyAlg} shows the argmax variant, which constructs a subset $Y \subseteq X$ of size $K$ by iteratively adding elements of $X$ to $Y$.  At each iteration, it adds the element that maximally increases \emph{marginal gain} $\Delta_{Q}(e|\mathfrak{a})$ of adding a sensor $e$ to a subset of sensors $\mathfrak{a}$:
\begin{equation}
\Delta_{Q}(e|\mathfrak{a}) = Q(b,e \cup \mathfrak{a}) - Q(b,\mathfrak{a}).
\end{equation}

\begin{algorithm}
\caption{$\mathtt{greedy}\hbox{-}\mathtt{argmax}(Q,X,K)$}\label{GreedyAlg}
\begin{algorithmic}
\State $Y \gets \emptyset$
\For{$m = 1 \ to \ K$}
\State $Y \gets Y \cup \{\argmax_{e \in X \setminus Y} \Delta_{Q}(e|Y) \}$
\EndFor
\State return $Y$
\end{algorithmic}
\end{algorithm}

To exploit greedy maximization in PBVI, we need to replace an argmax over $A$ with $\mathtt{greedy}\hbox{-}\mathtt{argmax}$. Our alternative description of PBVI above makes this straightforward: (\ref{eq:pbvistep2}) contains such an argmax and $Q(b,.)$ has been intentionally formulated to be a set function over $A^{+}$. Thus, implementing greedy PBVI requires only replacing (\ref{eq:pbvistep2}) with: 
\begin{equation}
\mathfrak{a}^{G} = \mathtt{greedy}\hbox{-}\mathtt{argmax}(Q(b,\cdot),A^{+},K).
\end{equation}

Since the complexity of $\mathtt{greedy}\hbox{-}\mathtt{argmax}$ is only $\mathcal{O}(|N||K|)$, the complexity of greedy PBVI is only $\mathcal{O}(|S||B||N||K||\Gamma_{t-1}|)$ (as compared to $\mathcal{O}(|S||B|\binom{n}{k})$ for traditional PBVI for computing $\Gamma_{t}^{\mathbf{a}}$).

Using point-based methods as a starting point is essential to our approach. Algorithms like Monahan's enumeration algorithm \citep{monahan1982} that rely on pruning operations to compute $V^{*}$ instead of performing an explicit argmax, cannot directly use $\mathtt{greedy}\hbox{-}\mathtt{argmax}$. Thus, it is precisely because PBVI operates on a finite set of beliefs that an explicit argmax is performed, opening the door to using $\mathtt{greedy}\hbox{-}\mathtt{argmax}$ instead. 

\subsection{Bounds given submodular value function} \label{sec:BoundSubmodVal}

In the following subsections, we present the highlights of the theoretical guarantees associated with greedy PBVI. The detailed analysis can be found in the appendix. Specifically, we show that a value function computed by greedy PBVI is guaranteed to have bounded error with respect to the optimal value function under \emph{submodularity}, a property of set functions that formalizes the notion of diminishing returns. Then, we establish the conditions under which the value function of a POMDP is guaranteed to be submodular. We define $\rho(b)$ as negative belief entropy, $\rho(b) = -H_{b}(s)$ to establish the submodularity of value function. Both $\rho$POMDP and POMDP-IR approximate $\rho(b)$ with tangents. Thus, in the last subsection, we show that even if belief entropy is approximated using tangents, the value function computed by greedy PBVI is guaranteed to have bounded error with respect to the optimal value function.

Submodularity is a property of set functions that corresponds to diminishing returns, i.e., adding an element to a set increases the value of the set function by a smaller or equal amount than adding that same element to a subset. In our notation, this is formalized as follows. Given a policy $\pi$, the set function $Q^{\pi}_{t}(b,\mathfrak{a})$ is submodular in $\mathfrak{a}$, if for every $\mathfrak{a}_{M} \subseteq \mathfrak{a}_{N} \subseteq A^+$  and $a_{e} \in A^+ \setminus \mathfrak{a}_{N}$,
\begin{equation}
\Delta_{Q_{b}}(a_{e}|\mathfrak{a}_{M}) \geq \Delta_{Q_{b}}(a_{e}|\mathfrak{a}_{N}),
\end{equation}

Equivalently, $Q^{\pi}_{t}(b,\mathfrak{a})$ is submodular if for every $\mathfrak{a}_{M}, \mathfrak{a}_{N} \subseteq A^+$,
\begin{equation}
Q^{\pi}_{t}(b,\mathfrak{a}_{M} \cap \mathfrak{a}_{N}) + Q^{\pi}_{t}(b,\mathfrak{a}_{M} \cup \mathfrak{a}_{N}) \leq Q^{\pi}_{t}(b,\mathfrak{a}_{M}) + Q^{\pi}_{t}(b,\mathfrak{a}_{N}).  \nonumber
\end{equation}

Submodularity is an important property because of the following result:
\begin{theorem} \label{theorem1}
\citep{Nemhauser} If $Q^{\pi}_{t}(b,\mathfrak{a})$ is non-negative, monotone and  submodular in $\mathfrak{a}$, then for all $b$,
\begin{equation}\label{greedyGuarantee}
Q^{\pi}_{t}(b, \mathfrak{a}^{G}) \geq (1-e^{-1})Q^{\pi}_{t}(b, \mathfrak{a}^{*}),
\end{equation}
where $\mathfrak{a}^{G} = \mathtt{greedy}\hbox{-}\mathtt{argmax}(Q^\pi_t(b,\cdot),A^+,K)$ and $\mathfrak{a}^{*} = \argmax_{\mathfrak{a} \in A} Q^\pi_t(b,\mathfrak{a})$.
\end{theorem}

Theorem \ref{theorem1} gives a bound only for a single application of $\mathtt{greedy}\hbox{-}\mathtt{argmax}$, not for applying it within each backup, as greedy PBVI does.  
In this subsection, we establish such a bound. Let the \emph{greedy Bellman operator} $\mathfrak{B}^{G}$ be:
\begin{equation*}
(\mathfrak{B}^{G}V_{t-1}^{\pi})(b)  =  \max^{G}_{\mathfrak{a}}[ \rho(b,\mathfrak{a}) +  \gamma \sum_{\mathbf{z}\in\Omega} \Pr(\mathbf{z}|\mathfrak{a},b)V_{t-1}^{\pi}(b^{\mathfrak{a},\mathbf{z}})],
\end{equation*}
where $\max^{G}_{\mathfrak{a}}$ refers to greedy maximization.  This immediately implies the following corollary to Theorem~\ref{theorem1}:

\begin{corollary}
\label{Lemma1}
Given any policy $\pi$, if $Q^{\pi}_{t}(b,\mathfrak{a})$ is non-negative, monotone, and submodular in $\mathfrak{a}$, then for all~$b$,
\begin{equation}
(\mathfrak{B}^{G}V^{\pi}_{t-1})(b) \geq (1 - e^{-1}) (\mathfrak{B}^{*}V^{\pi}_{t-1})(b).
\end{equation}
\end{corollary}
\begin{proof} 
From Theorem \ref{theorem1} since $(\mathfrak{B}^{G}V^{\pi}_{t-1})(b) = Q^\pi_t(b,\mathfrak{a}^G)$ and $(\mathfrak{B}^{*}V^{\pi}_{t-1})(b) = Q^\pi_t(b,\mathfrak{a}^*)$. 
\qed
\end{proof}

Next, we define the \emph{greedy Bellman equation}: $V^{G}_{t}(b) = (\mathfrak{B}^{G}V^{G}_{t-1})(b)$, where $V^{G}_{0} = \rho(b)$. Note that $V^{G}_{t}$ is the true value function obtained by greedy maximization, without any point-based approximations. 
Using Corollary \ref{Lemma1}, we can bound the error of $V^G$ with respect to $V^*$.
\begin{theorem} \label{th:bound}
If for all policies $\pi$, $Q^{\pi}_{t}(b,\mathfrak{a})$ is non-negative, monotone and submodular in $\mathfrak{a}$, then for all~$b$,
\begin{equation} \label{eq:bound}
V^{G}_{t}(b) \geq (1 - e^{-1})^{2t}V^{*}_{t}(b).
\end{equation}
\end{theorem}

\begin{proof}
See Appendix. 
\end{proof}

Theorem \ref{th:bound} extends Nemhauser's result to a full sequential decision making setting where multiple application of greedy maximization are employed over multiple time steps. This theorem gives a theoretical guarantee on the performance of greedy PBVI. Given a POMDP with a submodular value function, greedy PBVI is guaranteed to have bounded error with respect to the optimal value function. Moreover, this performance comes at a computational cost that is much less than that of solving the same POMDP with traditional solvers. Thus, greedy PBVI scales much better in the size of the action space of active perception POMDPs, while still retaining bounded error.

The results presented in this subsection are applicable only if the value function for a POMDP is submodular. In the following subsections, we establish the submodularity of value function for active perception POMDP under certain conditions.

\subsection{Submodularity of value functions}

The previous subsection showed that the value function computed by greedy PBVI is guaranteed to have bounded error as long as it is non-negative, monotone and submodular. In this subsection, we establish sufficient conditions for these properties to hold. Specifically, we show that, if the belief-based reward is negative entropy, i.e., $\rho(b) = -H_{b}(s) + \log(\frac{1}{|S|})$ then under certain conditions $Q^{\pi}_{t}(b,\mathfrak{a})$ is submodular, non-negative and monotone as required by Theorem \ref{th:bound}. We point out that the second part, $\log(\frac{1}{|S|})$ is only required (and sufficient) to guarantee non-negativity, but is independent of the actual beliefs or actions. For the sake of conciseness, in the remainder of this paper we will omit this term.

We start by observing that
$Q^{\pi}_{t}(b,\mathfrak{a}) = \rho(b) + \sum_{k = 1}^{t-1} G^{\pi}_{k}(b^{t},\mathfrak{a}^{t})$, 
where $G^{\pi}_{k}(b^{t},\mathfrak{a}^{t})$ is the expected immediate reward with $k$ steps to go, conditioned on the belief and action with $t$ steps to go and assuming policy $\pi$ is followed after timestep $t$:
\begin{equation*}
G^{\pi}_{k}(b^{t},\mathfrak{a}^{t}) = \gamma^{t - k}\sum_{\mathbf{z}^{t:k}}Pr(\mathbf{z}^{t:k}|b^{t},\mathfrak{a}^{t},\pi)(-H_{b^{k}}(s^{k})),
\end{equation*}
where $\mathbf{z}^{t:k}$ is a vector of observations received in the interval from $t$ steps to go to $k$ steps to go, $b^{t}$ is the belief at $t$ steps to go, $\mathfrak{a}^{t}$ is the action taken at $t$ steps to go, and $\rho(b^{k}) = -H_{b^k}(s^{k})$, where $s^{k}$ is the state at $k$ steps to go.
To show that $Q^{\pi}_{t}(b,\mathfrak{a})$ is submodular the main condition is \emph{conditional independence} as defined below:
\begin{definition}
The observation set $\mathfrak{z}$ is conditionally independent given $s$ if any pair of observation features are conditionally independent given the state, i.e.,
\begin{equation} \label{CondEntSub}
\Pr(z_i,z_j|s) = \Pr(z_i|s)\Pr(z_j|s), \quad \forall z_i,z_j \in \mathfrak{z}.
\end{equation}
\end{definition}

Using above definition, the submodularity of $Q(b,\mathfrak{a})$ can be established as: 

\begin{theorem} \label{th:submod}
If $\mathfrak{z}^{t:k}$ is conditionally independent given ${s}^{k}$ and $\rho(b) = - H_b(s)$, then $Q^{\pi}_{t}(b,\mathfrak{a})$ is submodular in $\mathfrak{a}$, for all $\pi$.
\end{theorem}
\begin{proof}
See Appendix.
\end{proof}

\begin{theorem} \label{th:core}
If $\mathfrak{z}^{t:k}$ is conditionally independent given $s^{k}$, $V^{\pi}_{t}$ is convex over the belief space for all $t, \pi$, and $\rho(b) = - H_b(s) + log(\frac{1}{|S|})$, then for all $b$,
\begin{equation}
V^{G}_{t}(b) \geq (1 - e^{-1})^{2t}V^{*}_{t}(b).
\end{equation}
\end{theorem}
\begin{proof}
See Appendix.
\end{proof}

In this subsection we showed that if the immediate belief-based reward $\rho(b)$ is defined as negative belief entropy, then the value function of an active perception POMDP is guaranteed to be submodular under certain conditions. However, as mentioned earlier, to solve active perception POMDP, we approximate the belief entropy with vector tangents. This might interfere with the submodularity of the value function. In the next subsection, we show that, even though the PWLC approximation of belief entropy might interfere with the submodularity of the value function, the value function computed by greedy PBVI is still guaranteed to have bounded error.

\subsection{Bounds given approximated belief entropy}

While Theorem \ref{th:core} bounds the error in $V^{G}_{t}(b)$, it does so only on the condition that $\rho(b) = - H_b(s)$.  However, as discussed earlier, our definition of active perception POMDPs instead defines $\rho$ using a set of vectors $\Gamma^{\rho} = \{\alpha^{\rho}_1,\ldots,\alpha^\rho_m\}$, each of which is a tangent to $- H_{b}({s})$, as suggested by \cite{Mauricio}, in order to preserve the PWLC property. While this can interfere with the submodularity of  $Q^{\pi}_{t}(b,\mathfrak{a})$, here we show that the error generated by this approximation is still bounded in this case.

Let $\tilde{\rho}(b)$ denote the PWLC approximated entropy and $\tilde{V}_{t}^{*}$ denote the optimal value function when using a PWLC approximation to negative entropy for the belief-based reward, as in an active perception POMDP, i.e.,
\begin{equation}
\tilde{V}_{t}^{*}(b) = \max_{\mathfrak{a}}[\tilde{\rho}(b) + \sum_{\mathbf{z} \in \Omega} \Pr(\mathbf{z}|b,\mathfrak{a})\tilde{V}_{t-1}^{*}(b^{\mathfrak{a},\mathbf{z}})].
\end{equation} 
\citet{Mauricio} showed that, if $\rho(b)$ verifies the $\alpha$-H\"{o}lder condition \citep{alphaHolder}, a generalization of the Lipschitz condition, then the following relation holds between $V^{*}_{t}$ and $\tilde{V}^{*}_{t}$:  
\begin{equation}\label{nipseq}
||V_{t}^{*} - \tilde{V}_{t}^{*}||_{\infty} \leq \frac{C \delta^{\alpha}}{1 - \gamma},
\end{equation}
where $V_{t}^{*}$ is the optimal value function with $\rho(b) = - H_b({s})$, $\delta$ is the \emph{density} of the set of belief points at which tangent are drawn to the belief entropy, and $C$ is a constant.

Let $\tilde{V}^{G}_{t}(b)$ be the value function computed by greedy PBVI when immediate belief-based reward is $\tilde{\rho}(b)$:
\begin{equation}
\tilde{V}^{G}_{t}(b) = \max_{\mathfrak{a}}^{G}[\tilde{\rho}(b) + \sum_{\mathbf{z} \in \Omega} \Pr(\mathbf{z}|b,\mathfrak{a})\tilde{V}_{t-1}^{G}(b^{\mathfrak{a},\mathbf{z}})], 
\end{equation}
then the error between $\tilde{V}^{G}_{t}(b)$ and $V^{*}_{t}(b)$ is bounded as stated in the following theorem.
\begin{theorem} \label{th:approx-bound}
For all beliefs, the error between $\tilde{V}^{G}_{t}(b)$ and $\tilde{V}^{*}_{t}(b)$ is bounded, if $\rho(b) = -H_{b}(s)$, $V^{\pi}_{t}$ is convex in the belief space for all $\pi, t$, and if $\mathfrak{z}^{t:k}$ is conditionally independent given $s^{k}$.
\end{theorem}
\begin{proof}
See Appendix.
\end{proof}

In this subsection we showed that if the negative entropy is approximated using tangent vectors, greedy PBVI still computes a value function that has bounded error. In the next subsection we outline how greedy PBVI can be extended to general active perception tasks.

\subsection{General Active Perception POMDPs}
The results presented in this section apply to the active perception POMDP in which the evolution of the state over time is independent of the actions of the agent. Here, we outline how these results can be extended to general active perception POMDPs without many changes. The main application for such an extension is in tasks involving a mobile robot coordinating with sensors to intelligently take actions to perceive its environment. In such cases, the robot's actions, by causing it to move, can change the state of the world.

The algorithms we proposed can be extended to such settings by making small modifications to the greedy maximization operator. The greedy algorithm can be run for $K+1$ iterations and in each iteration the algorithm would choose to add either a sensor (only if fewer than $K$ sensors have been selected), or a movement action (if none has been selected so far). Formally, using the work of \cite{fisher}, which extends that of \cite{Nemhauser} on submodularity to combinatorial structures such as \emph{matroids}, the action space of a POMDP involving a mobile robot can be modeled as a \emph{partition matroid} and  greedy maximization subject to matroid constraints \citep{fisher} can be used to maximize the value function approximately. 

The guarantees associated with greedy maximization subject to matroid constraints \citep{fisher} can then be used to bound the error of greedy PBVI. However, deriving exact theoretical guarantees for greedy PBVI for such tasks is beyond the scope of this article. Assuming that the reward function is still defined as the negative belief entropy, the submodularity of such POMDPs still holds under the conditions mentioned in Section 6.2.

In this subsection, we presented greedy PBVI, which uses greedy maximization to improve the scalability in the action space of an active perception POMDP. We also showed that, if the value function of an active perception POMDP is submodular, then greedy PBVI computes a value function that is guaranteed to have bounded error. We established that if the belief-based reward is defined as the negative belief entropy, then the value function of an active perception POMDP is guaranteed to be submodular. We showed that if the negative belief entropy is approximated by tangent vectors, as is required to solve active perception POMDPs efficiently, greedy PBVI still computes a value function that has bounded error. Finally, we outlined how greedy PBVI and the associated theoretical bounds can be extended to general active perception POMDPs.

\section{Experiments}

\begin{figure}
\centering
\hspace{5mm} \includegraphics[scale=0.5]{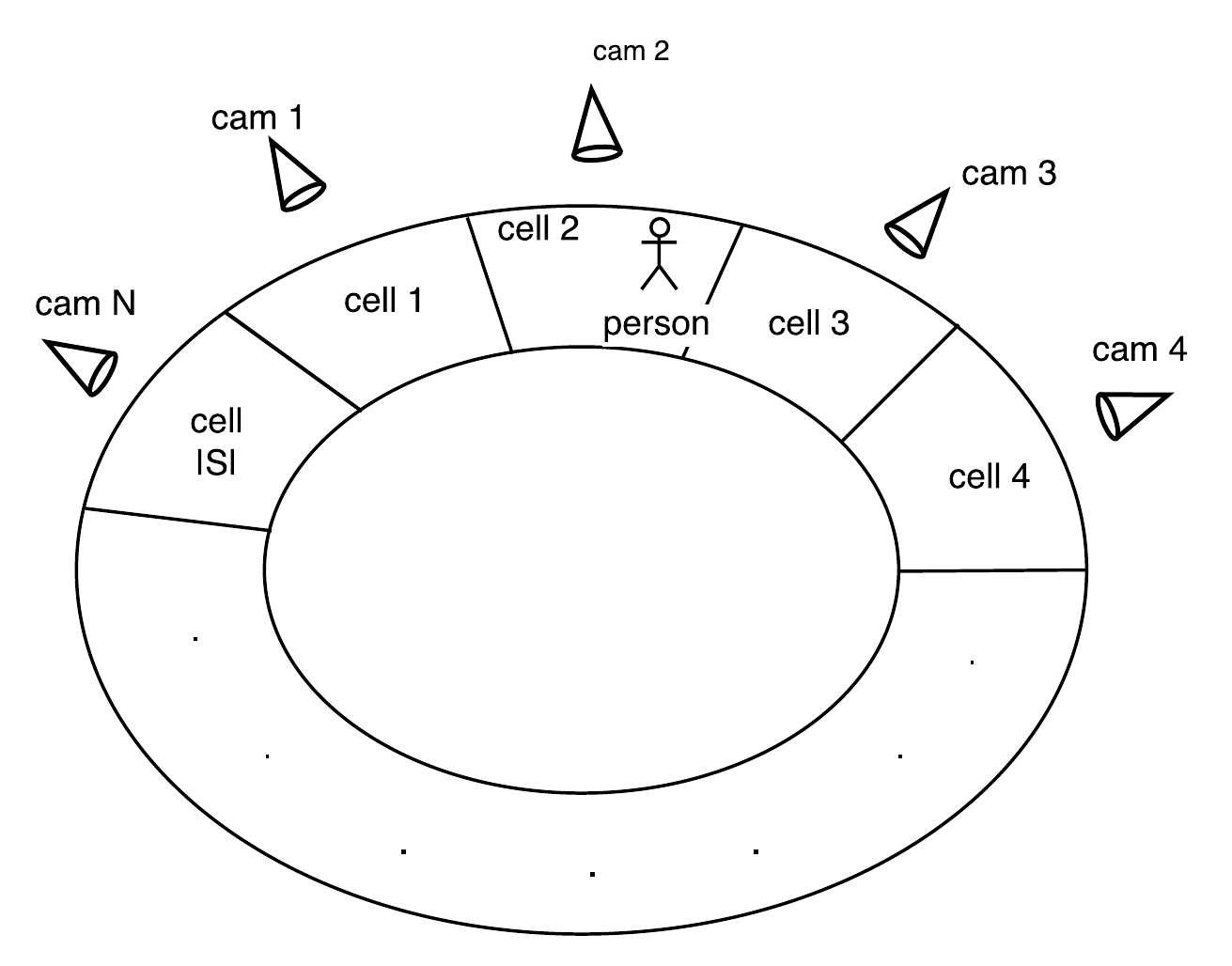}
\caption{ Problem setup for the task of tracking one person. We model this task as a POMDP with one state for each cell. Thus the person can move among $|S|$ cells. Each cell is adjacent to two other cells and each cell is monitored by a single camera. Thus, in this case there are $N = |S|$ cameras. At each time step, the person can stay in the same cell as she was in the previous time step with probability $p$ or she can move to one of the neighboring cells with equal probability. The agent must select $K$ out of $N$ cameras and the task is to predict the state of the person correctly using noisy observations from the $K$ cameras. There is one prediction action for each state and the agent gets a reward of $+1$ if it correctly predicts the state and $0$ otherwise. An observation is a vector of $N$ \emph{observation features}, each of which specifies the person's position as estimated by the given camera. If a camera is not selected, then the corresponding observation feature has a value of null.} \label{fig:setting}
\end{figure}

In this section, we present an analysis of the behavior and performance of belief-based rewards for active perception tasks, which is the main motivation of our work. We present the results of experiments designed to study the effect on the performance of the choice of prediction actions/tangents, and compare the costs and benefits of myopic versus non-myopic planning. We consider the task of tracking people in a surveillance area with a multi-camera tracking system. The goal of the system is to select a subset of cameras, to correctly predict the position of people in the surveillance area, based on the observations received from the selected cameras. In the following subsections, we present results on real-data collected from a multi-camera system in a shopping mall and we present the experiments comparing performance of greedy PBVI to PBVI.

\begin{figure*}
\begin{center}
  \subfigure[]{\includegraphics[width=0.45 \textwidth ]{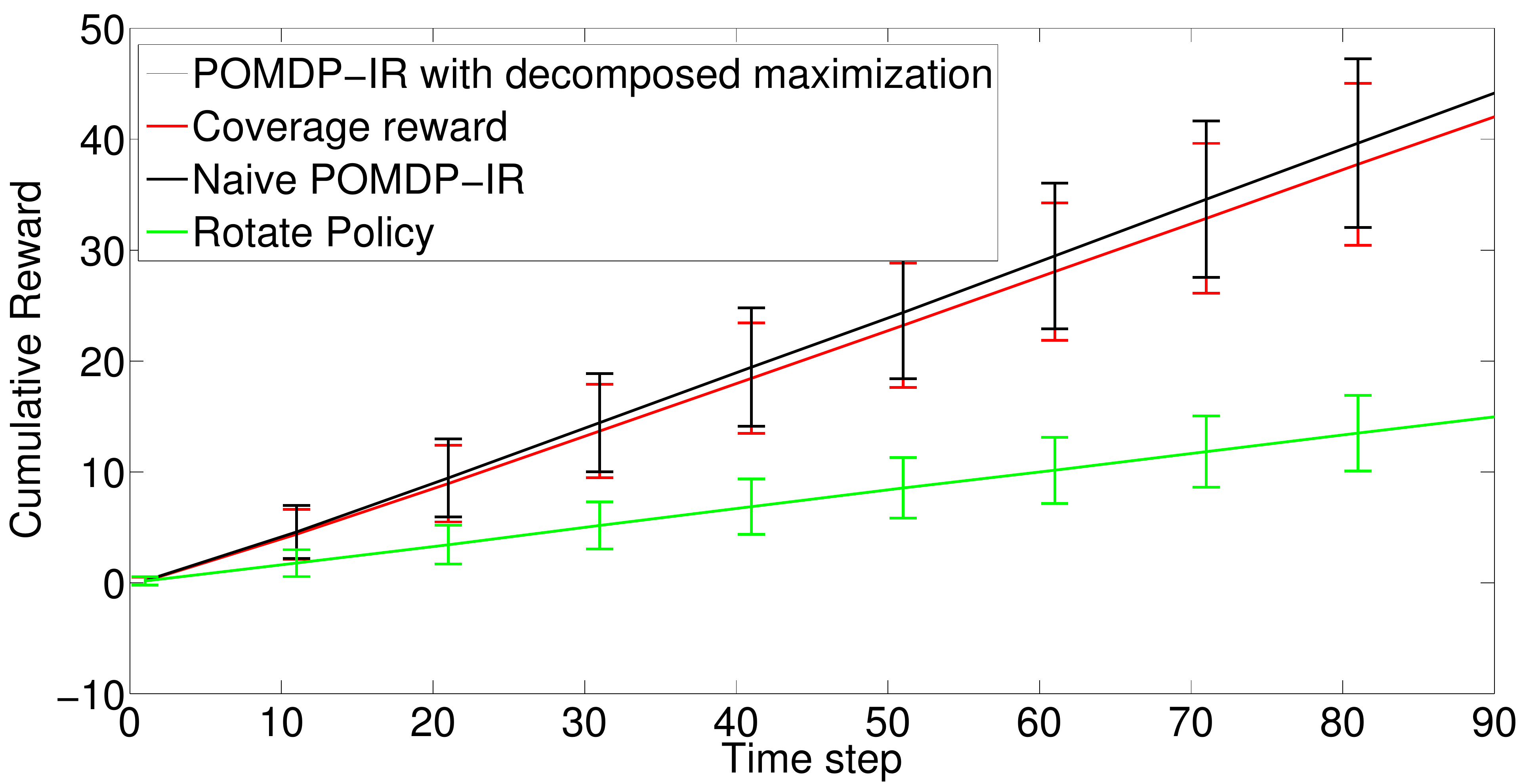}} \hfill
  \subfigure[]  {\includegraphics[width=0.45 \textwidth]{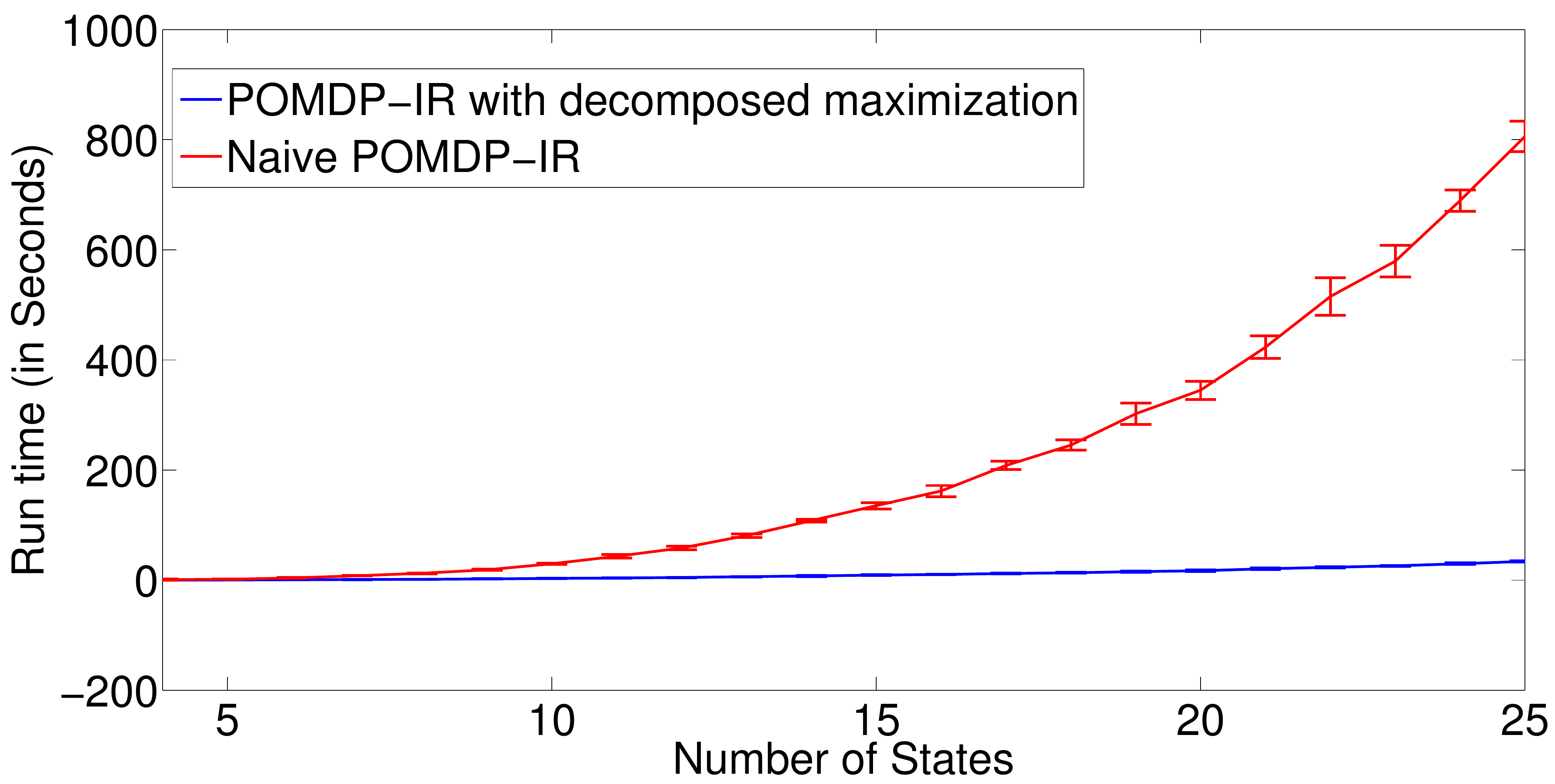} \hfill}  \\
  \subfigure[]  {\includegraphics[width=0.45 \textwidth]{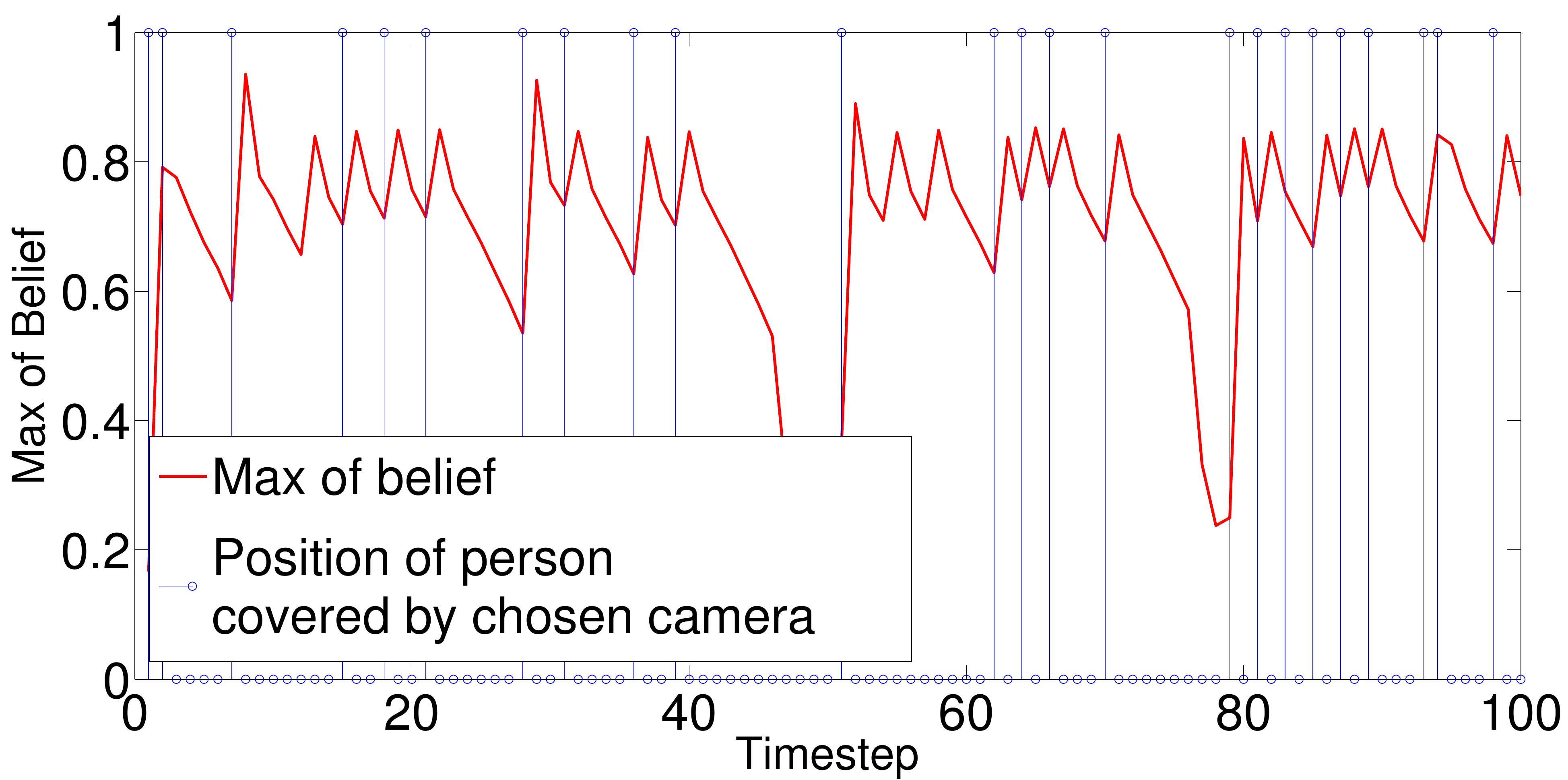}} \hfill
  \subfigure[]  {\includegraphics[width=0.45 \textwidth]{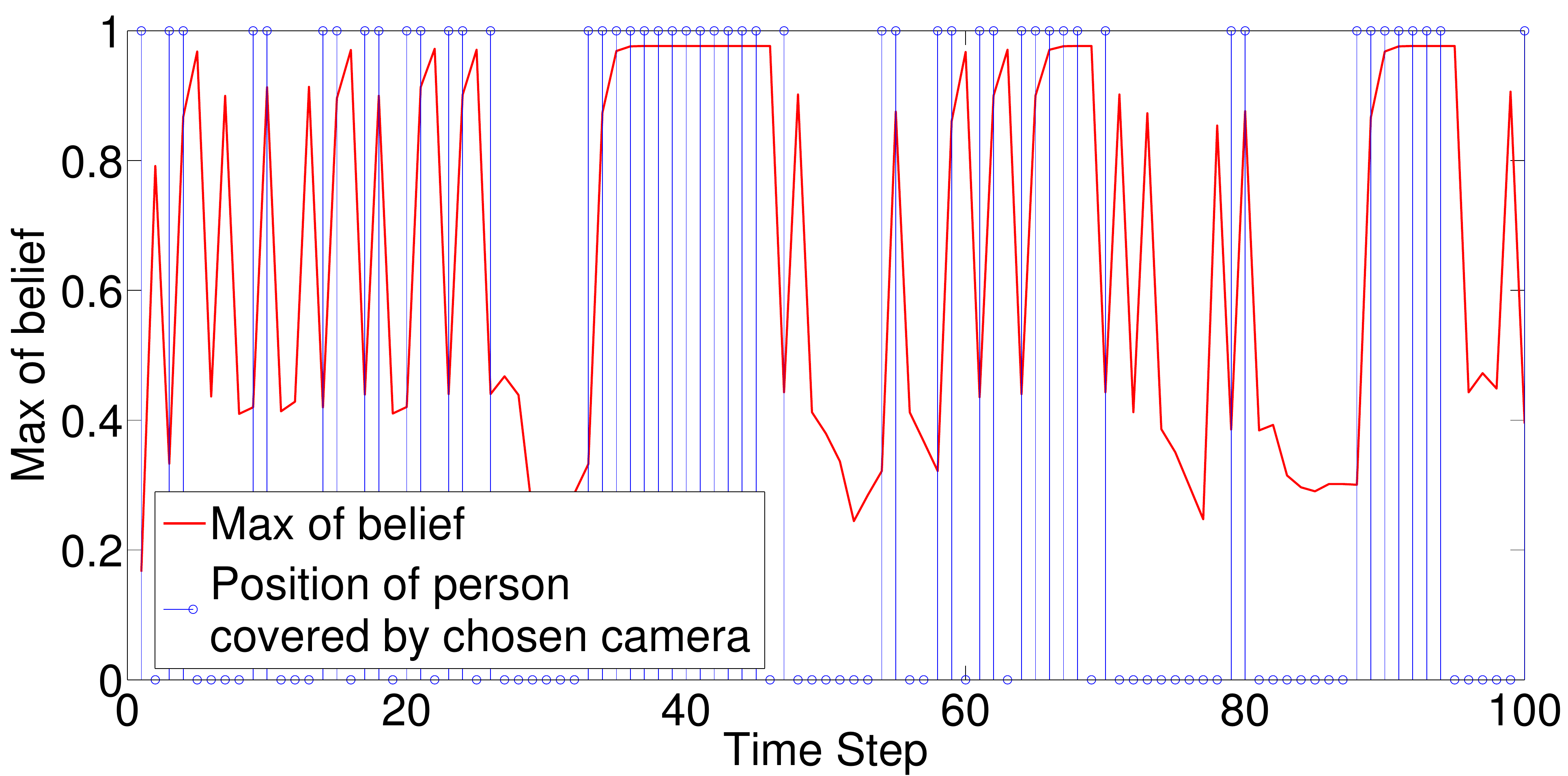} \hfill}
\end{center}
  \caption{(a) Performance comparison between POMDP-IR with decomposed maximization, naive POMDP-IR, coverage policy, and rotate policy; (b) Runtime comparison between POMDP-IR with decomposed maximization and naive POMDP-IR; (c) Behavior of POMDP-IR policy; (d) Behavior of coverage  policy.}
  \label{fig:perfComp}
\end{figure*}

We compare the performance of POMDP-IR with decomposed maximization to a naive POMDP-IR that does not decompose the maximization.  Thanks to Theorems \ref{th:equiv1} and \ref{th:equiv2}, these approaches have performance equivalent to their $\rho$POMDP counterparts.  We also compare against two baselines. The first is a weak baseline we call the \emph{rotate policy} in which the agent simply keeps switching between cameras on a turn-by-turn basis. The second is a stronger baseline we call the \emph{coverage policy}, which was developed in earlier work on active perception \citep{Spaan08pomdp,Spaan09icaps}.  The coverage policy is obtained after solving a POMDP that rewards the agent for observing the person, i.e., the agent is encouraged to select the cameras that are most likely to generate positive observations. Thanks to the decomposed maximization, the computational cost of solving for the coverage policy and belief-based rewards is the same.

\subsection{Simulated Setting}

We start with experiments conducted in a simulated setting, first considering the task of tracking a single person with a multi-camera system and then considering the more challenging task of tracking multiple people.

\subsubsection{Single-Person Tracking}
We start by considering the task of tracking one person walking in a grid-world composed of $|S|$ cells and $N$ cameras as shown in Figure \ref{fig:setting}. At each timestep, the agent can select only $K$ cameras, where $K \leq N$.  Each selected camera generates a noisy observation of the person's location.  The agent's goal is to minimize its uncertainty about the person's state. In the experiments in this section, we fixed $K = 1$ and $N=10$. The problem setup and the POMDP model is shown and described in Figure \ref{fig:setting}.

To compare the performance of POMDP-IR to the baselines, 100 trajectories were simulated from the POMDP. The agent was asked to guess the person's position at each time step. Figure \ref{fig:perfComp}(a) shows the cumulative reward collected by all four methods. POMDP-IR with decomposed maximization and naive POMDP-IR perform identically as the lines indicating their respective performance lie on top of each other in figure \ref{fig:perfComp}(a).  However, Figure \ref{fig:perfComp}(b), which compares the runtimes of POMDP-IR with decomposed maximization and naive POMDP-IR, shows that decomposed maximization yields a large computational savings. Figure \ref{fig:perfComp}(a) also shows that POMDP-IR greatly outperforms the rotate policy and modestly outperforms the coverage policy.

Figures \ref{fig:perfComp}(c) and \ref{fig:perfComp}(d) illustrate the qualitative difference between POMDP-IR and the coverage policy. The blue lines mark the points in trajectory when the agent selected the camera that observes the person's location. If the agent selected a camera such that the person's location is not covered then the blue vertical line is not there at that point in the trajectory in the figure. The agent has to select one out of $N$ cameras and does not have an option of not selecting any camera. The red line plots the max of the agent's belief. The main difference between the two policies is that once POMDP-IR gets a good estimate of the state, it proactively observes neighboring cells to which the person might transition. This helps it to more quickly find the person when she moves. By contrast, the coverage policy always looks at the cell where it believes her to be. Hence, it takes longer to find her again when she moves.  This is evidenced by the fluctuations in the max of the belief, which often drops below 0.5 for the coverage policy but rarely does so for POMDP-IR. The presence of false positives and negatives can also be seen in the figure, when max of the belief goes down even though the agent selected the camera which can observe the person's location and in some cases even though the agent did not select the camera which can observe the person's location but still the max of belief shoots up.

\begin{figure}
\begin{center}
	\includegraphics[width=0.46 \textwidth]{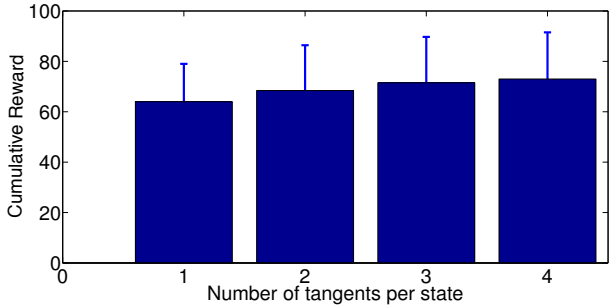} \quad
\end{center}
	\caption{Performance comparison as negative belief entropy is better approximated.}
	\label{fig:moreTans}
\end{figure}

Next, we examine the effect of approximating a true reward function
like belief entropy with more and more tangents.  Figure
\ref{fig:tangents} illustrates how adding more tangents can better
approximate negative belief entropy.  To test the effects of this, we
measured the cumulative reward when using between one
and four tangents per state.  Figure \ref{fig:moreTans} shows the
results and demonstrates that, as more tangents are added, the performance improves. However, performance also quickly saturates, as four tangents perform no better than three.

\begin{figure}
\begin{center}
	 	\includegraphics[width=0.46 \textwidth]{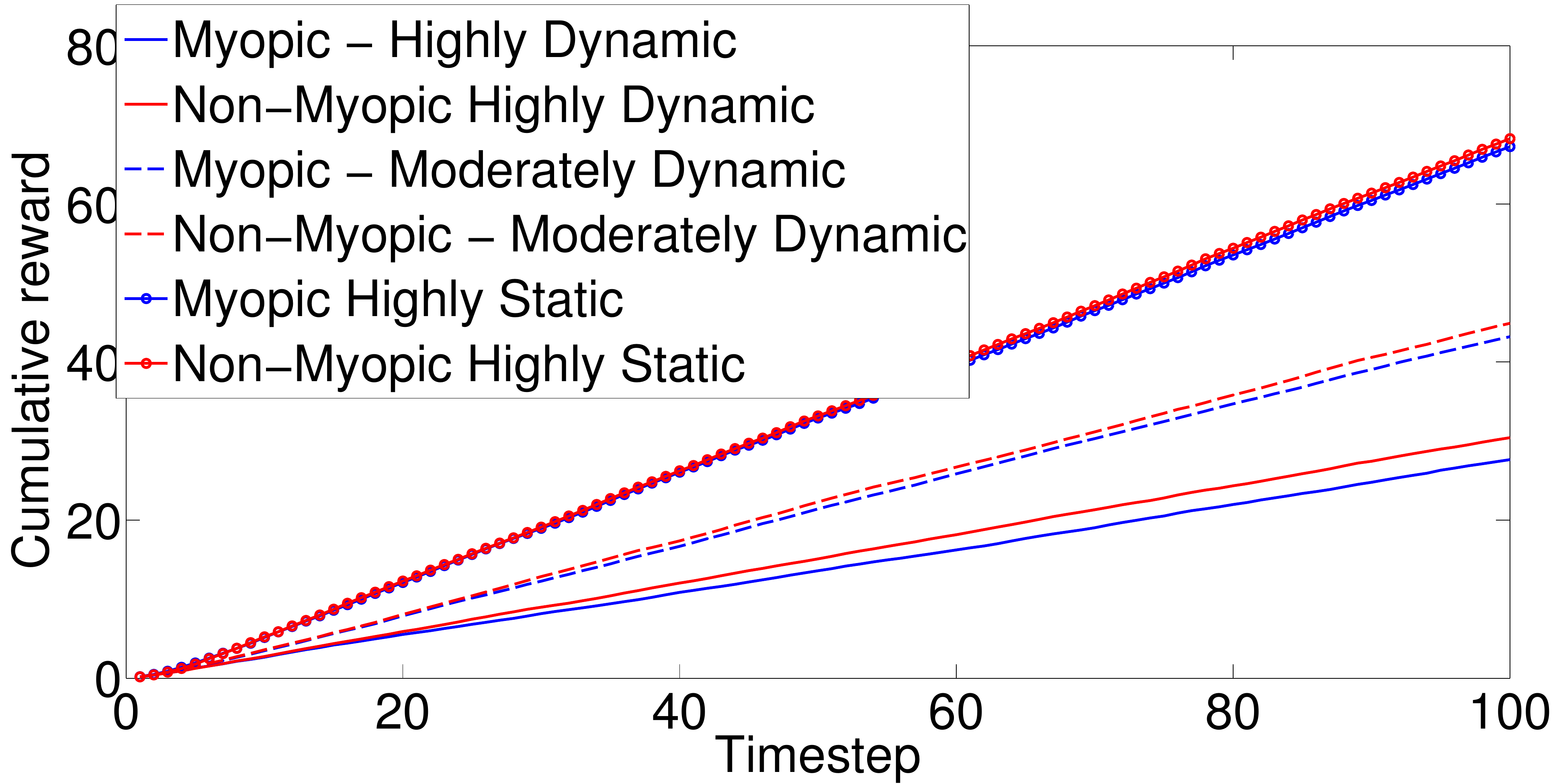}  \\
\hspace{-2mm}    \includegraphics[width=0.45 \textwidth]{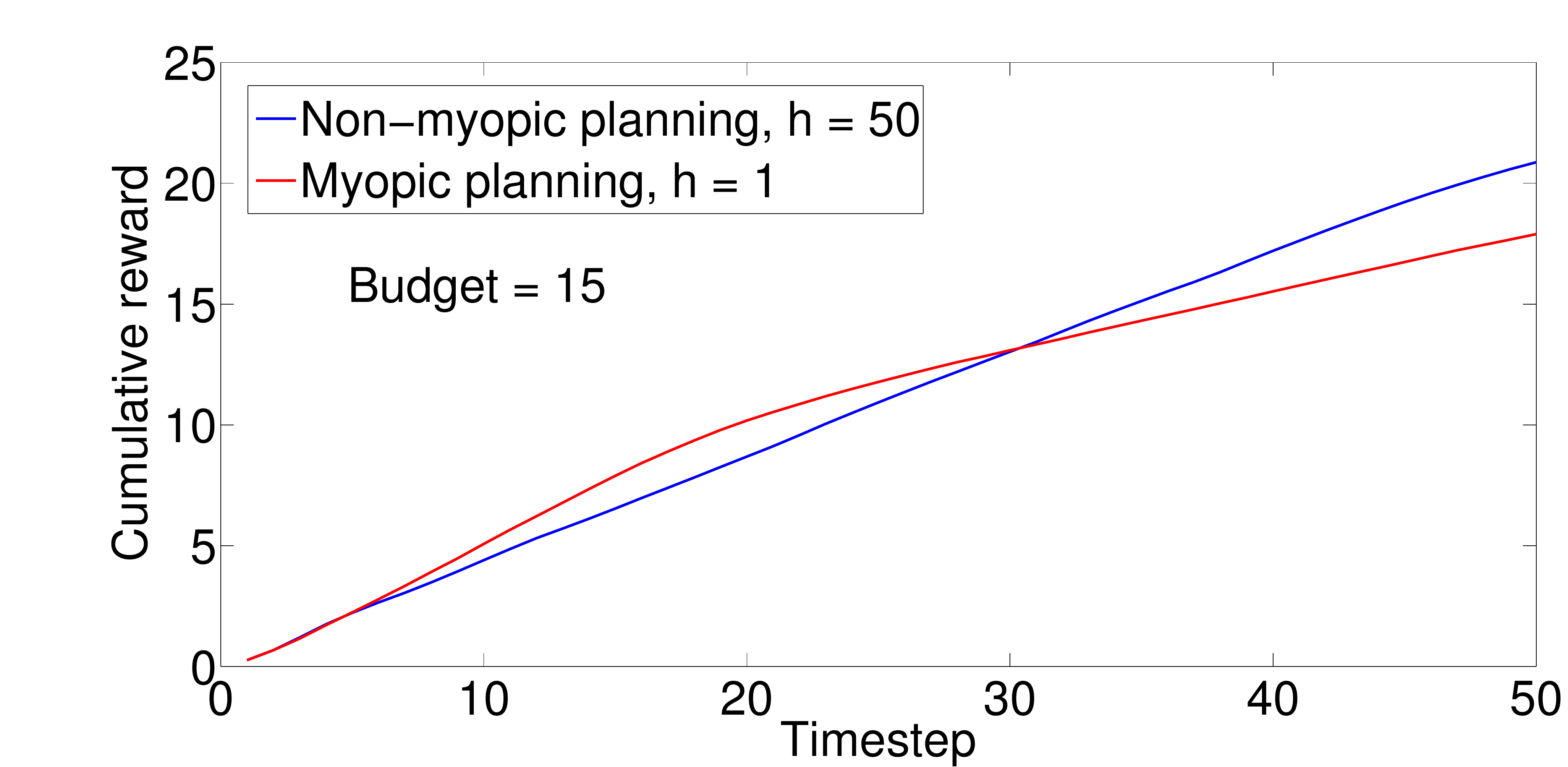}
\end{center}
	\caption{(top) Performance comparison for myopic vs.\ non myopic policies; (bottom) Performance comparison for myopic vs non myopic policies in budget-based setting.}
	\label{fig:myoVsNonMyo}
\end{figure}

Next, we compare the performance of POMDP-IR to a myopic variant that seeks only to maximize immediate reward, i.e., $h=1$.  We perform this comparison in three variants of the task. In the \emph{highly static} variant, the state changes very slowly: the probability of staying is the same state is 0.9. In the \emph{moderately dynamic} variant, the state changes more frequently, with a same-state transition probability of 0.7. In the \emph{highly dynamic} variant, the state changes rapidly (with a same-state transition probability of 0.5). Figure \ref{fig:myoVsNonMyo} (top) shows the results of these comparisons.  In each setting, non-myopic POMDP-IR outperforms myopic POMDP-IR. In the highly static variant, the difference is marginal.  However, as the task becomes more dynamic, the importance of look-ahead planning grows. Because the myopic planner focuses only on immediate reward, it ignores what might happen to its belief when the state changes, which happens more often in dynamic settings. 

We also compare the performance of myopic and non-myopic planning in a \emph{budget-constrained} environment. This specifically corresponds to an energy constrained environment, where cameras can be employed only a few times over the entire trajectory. This is augmented with resource constraints, so that the agent has to plan not only when to use the cameras, but also decide which camera to select. Specifically, the agent can only employ the multi-camera system a total of 15 times across all 50 timesteps and the agent can select which camera (out of the multi-camera system) to employ at each of the 15 instances.  On the other timesteps, it must select an action that generates only a null observation. Figure \ref{fig:myoVsNonMyo} (bottom) shows that non-myopic planning is of critical importance in this setting.  Whereas myopic planning greedily consumes the budget as quickly as possible, thus earning more reward in the beginning, non-myopic planning saves the budget for situations in which it is highly uncertain about the state.

Finally, we compare the performance of myopic and non-myopic planning when the multi-camera system can communicate with a mobile robot that also has sensors. This setting is typical of a networked robot system \citep{SpaanNRS} in which a robot coordinates with a multi-camera system to perform surveillance of a building, detect any emergency situations like fire, or help people navigate to their destination. Here, the task is to minimize uncertainty about the location of one person who is moving in the space monitored by the robot and the cameras. The robot's sensors are assumed to be more accurate than the stationary cameras. Specifically, the sensors attached to the robot can detect if a person is in the current cell with 90\% accuracy compared to the stationary cameras, each of which has an accuracy of 75\% of detecting a person in the cell it observes. The robot's sensor can observe the presence or absence of a person only for the cell that the robot occupies. In addition to using its sensors to generate observations about its current cell, the robot can also move forward or backward to an adjacent cell or choose to stay at the current cell. To model this task, the action vector introduced earlier is augmented with another action feature that indicates the direction of the robot's motion, which can take three values: forward, backward or stay. 

\begin{figure}
\begin{center}
	\includegraphics[width=0.45 \textwidth]{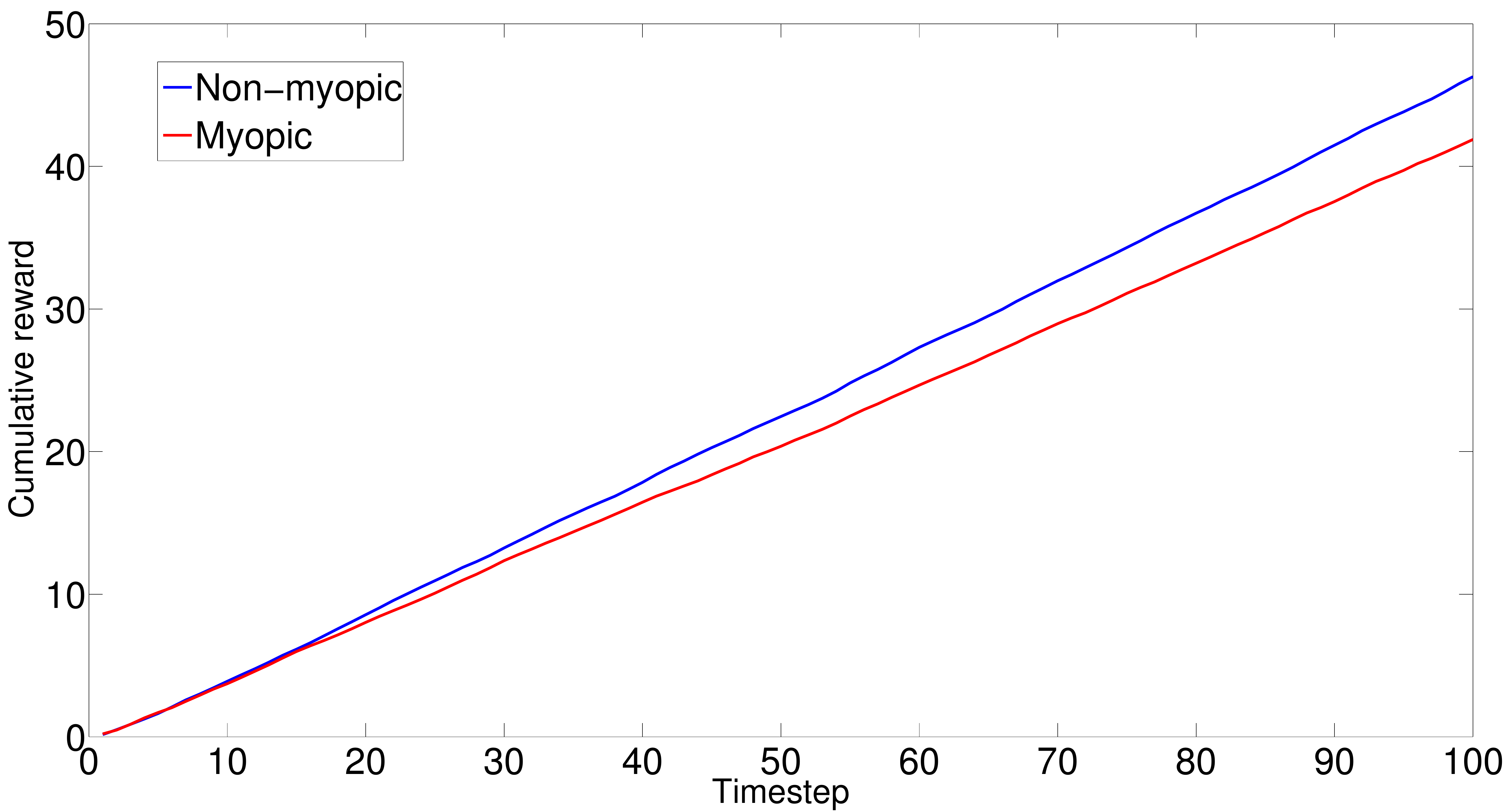}
\end{center}	
	\caption{Performance comparison for myopic vs.\ non myopic policies when camera system is assisting a moving robot.}
\label{fig:nwrresults}
\end{figure}

Performance is quantified as the total number of times the correct location of the person is predicted by the system. Figure \ref{fig:nwrresults}, which shows the performance of myopic and non-myopic policies for this task, demonstrates that when planning non-myopically the agent is able to utilize the accurate sensors more effectively as to compared to when planning myopically.

\subsubsection{Multi-Person Tracking}
To extend our analysis to a more challenging problem, we consider a simulated setting in which multiple people must be tracked simultaneously.  Since $|S|$ grows exponentially in the number of people, the resulting POMDP quickly becomes intractable.  Therefore, we compute instead a factored value function
\begin{equation}
V_{t}(b) = \sum_i V_{t}^{i}(b^{i}),
\end{equation} 
where $V_{t}^{i}(b^{i})$ is the value of the agent's current belief $b^{i}$ about the $i$-th person. Thus, $V_{t}^{i}(b^{i})$ needs to be computed only once, by solving a POMDP of the same size as that in the single-person setting. 
During action selection, $V_{t}(b)$ is computed using the current $b^{i}$ for each person. This kind of factorization corresponds to the assumption that each person's movement and observations is independent of that of other people. Although violated in practice, such an assumption can nonetheless yield good approximations.

\begin{figure}
\begin{center}
  \includegraphics[width=0.45 \textwidth]{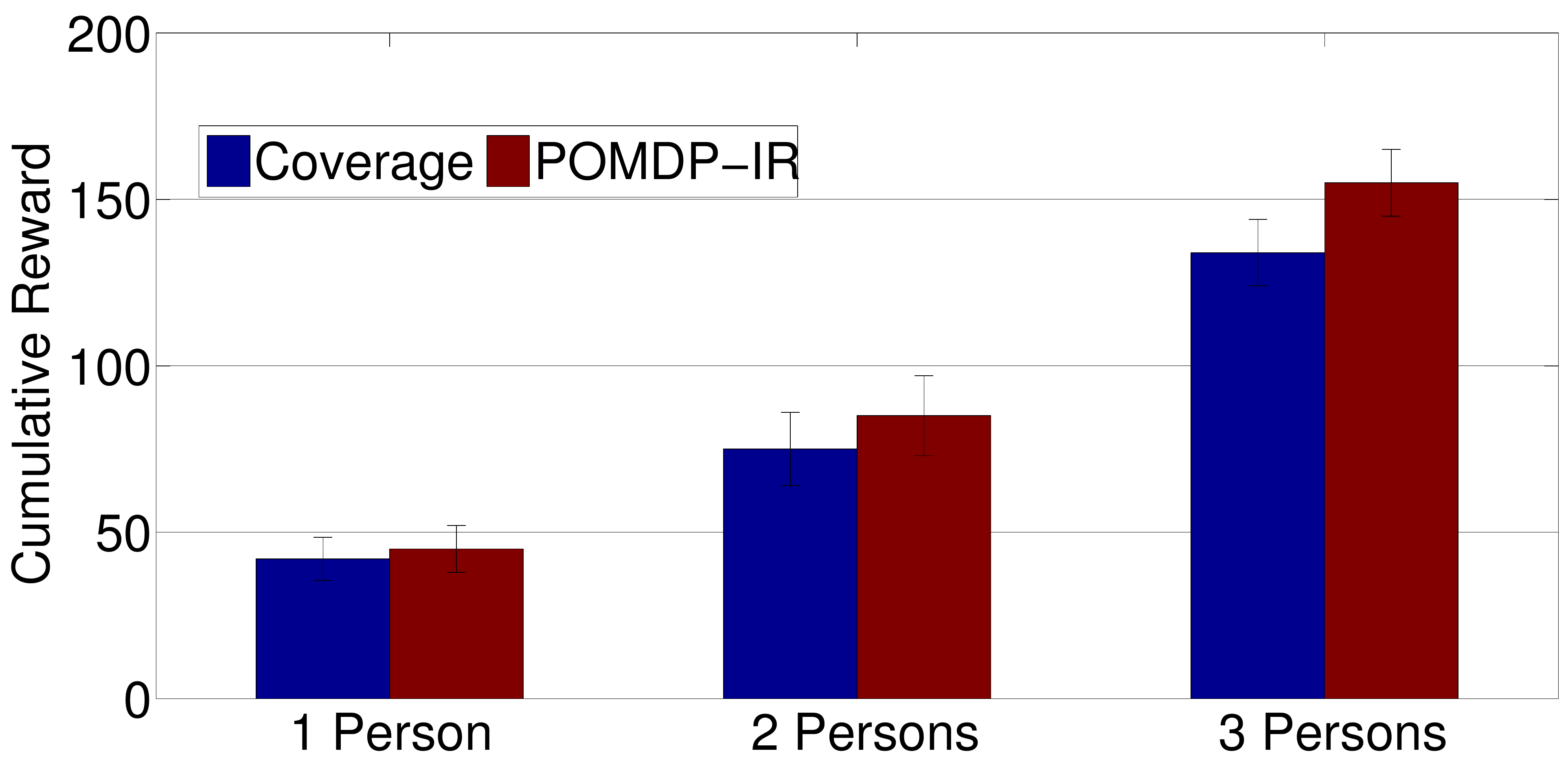}   \hfill
  \includegraphics[width=0.45 \textwidth]{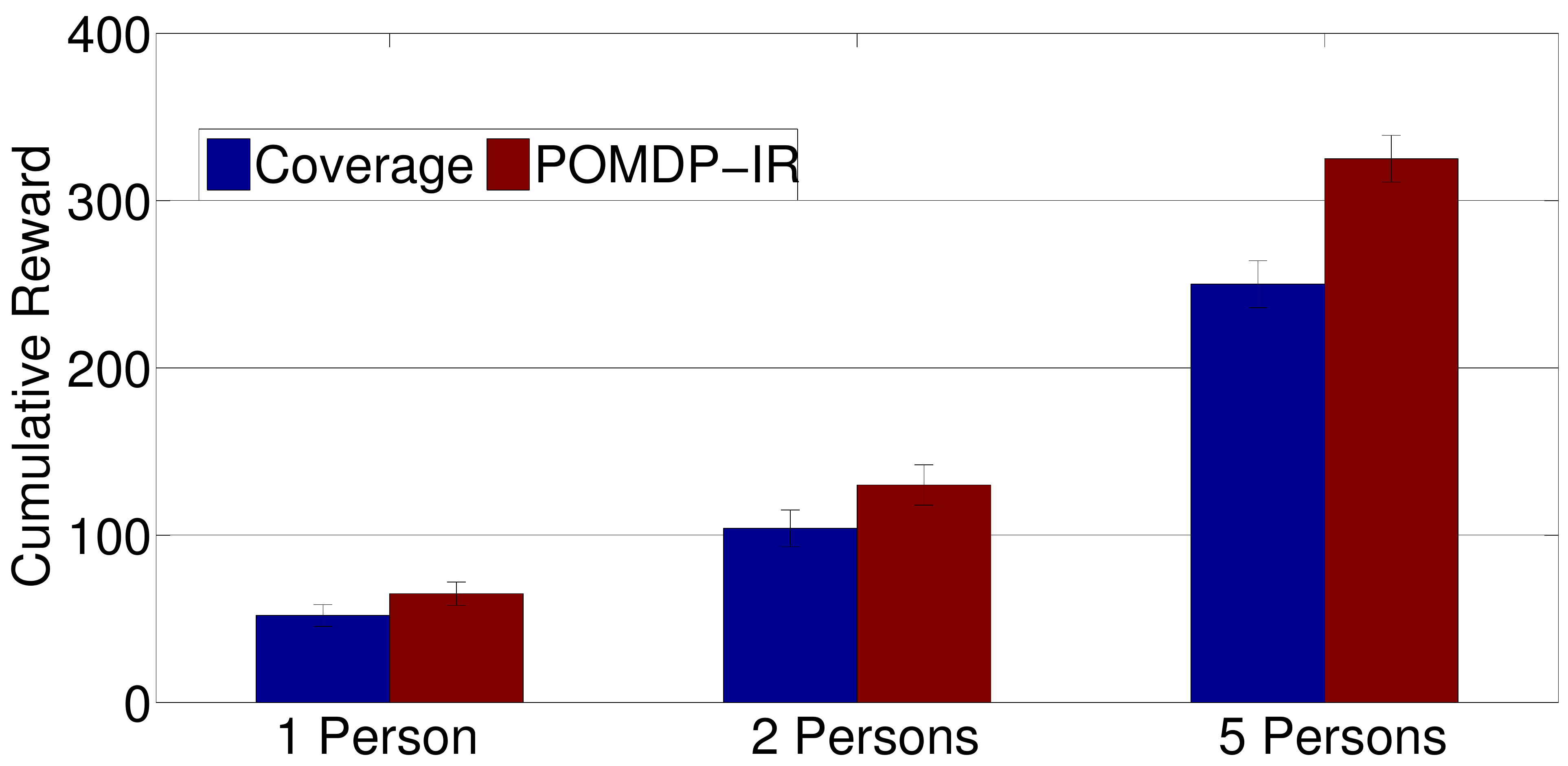} 
\end{center}
  \caption{(top) Multi-person tracking performance for POMDP-IR and coverage policy; (bottom) Performance of POMDP-IR and coverage policy when only important cells must be tracked. }
\label{fig:MPPerf}
\end{figure}

Figure \ref{fig:MPPerf} (top), which compares POMDP-IR to the coverage policy with one, two, and three people, shows that the advantage of POMDP-IR grows substantially as the number of people increases.  Whereas POMDP-IR tries to maintain a good estimate of everyone's position, the coverage policy just tries to look at the cells where the maximum number of people might be present, ignoring other cells completely. 

Finally, we compare POMDP-IR and the coverage policy in a setting in which the goal is only to reduce uncertainty about a set of ``important cells'' that are a subset of the whole state space.  For POMDP-IR, we prune the set of prediction actions to allow predictions only about important cells.  For the coverage policy, we reward the agent only for observing people in important cells.  The results, shown in Figure \ref{fig:MPPerf} (bottom), demonstrate that the advantage of POMDP-IR over the coverage policy is even larger in this variant of the task. POMDP-IR makes use of information coming from cells that neighbor the important cells (which is of critical importance if the important cells do not have good observability), while the coverage policy does not. As before, the difference gets larger as the number of people increases.

\subsection{Real Data}
\begin{figure}
  \center \includegraphics[width=0.38\textwidth]{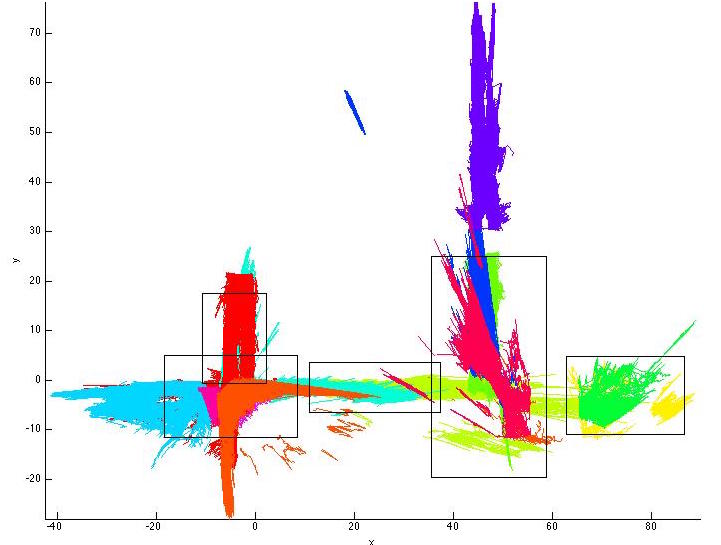}  
  \caption{Sample tracks for all the cameras.  Each color represents all the tracks observed by a given camera. The boxes denote regions of high overlap between cameras.}
  \label{fig:camtracks}
\end{figure}

Finally, we extended our analysis to a real-life dataset collected in a shopping mall. This dataset was gathered over 4 hours using 13 CCTV cameras located in a shopping mall \citep{henri}. Each camera uses a FPDW \citep{FPDW} pedestrian detector to detect people in each camera image and in-camera tracking \citep{henri} to generate tracks of the detected people's movements over time. 

The dataset consists of 9915 tracks each specifying one person's $x$-$y$ position over time. Figure \ref{fig:camtracks} shows the sample tracks from all of the cameras.

To learn a POMDP model from the dataset, we divided the continuous space into 20 cells ($|S| = 21$: 20 cells plus an external state indicating the person has left the shopping mall). Using the data, we learned a maximum-likelihood tabular transition function. However, we did not have access to the ground truth of the observed tracks so we constructed them using the overlapping regions of the camera.
\begin{figure}
\begin{center} 
	\includegraphics[width=0.45\textwidth]{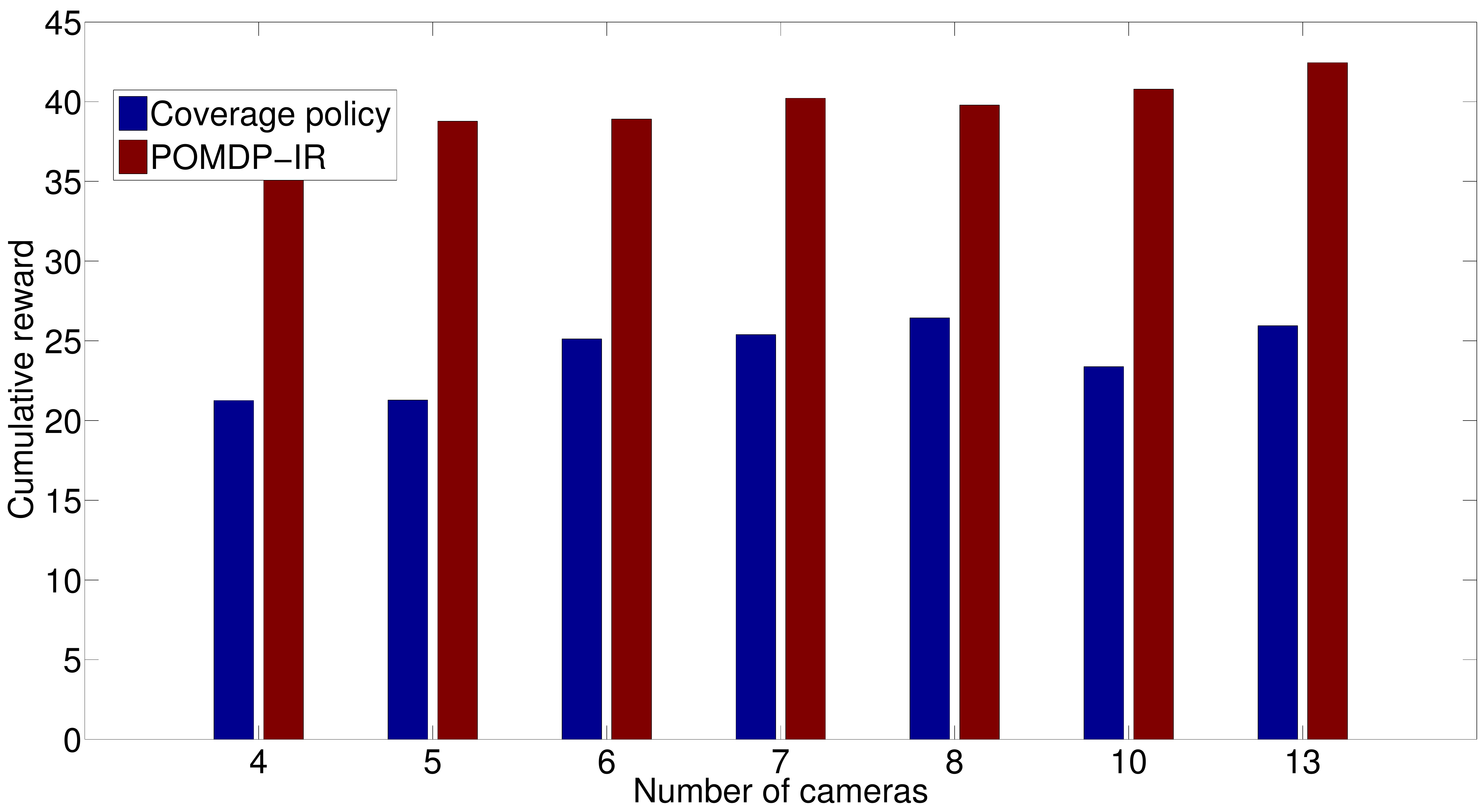} \quad
	\caption{Performance of POMDP-IR and the coverage policy on the shopping mall dataset.}
\label{fig:TNORes}
\end{center}
\end{figure} 

Because the cameras have many overlapping regions (see Figure \ref{fig:camtracks}), we were able to manually match tracks of the same person recorded individually by each camera. The ``ground truth'' was then constructed by taking a weighted mean of the matched tracks. Finally, this ground truth was used to estimate noise parameters for each cell (assuming zero-mean Gaussian noise), which was used as the observation function. Figure \ref{fig:TNORes} shows that, as before,  POMDP-IR substantially outperforms the coverage policy for various numbers of cameras. In addition to the reasons mentioned before, the high overlap between the cameras contributes to POMDP-IR's superior performance.  The coverage policy has difficulty ascertaining people's exact locations because it is rewarded only for observing them somewhere in a camera's large overlapping region, whereas POMDP-IR is rewarded for deducing their exact locations.

\subsection{Greedy PBVI}
To empirically evaluate greedy PBVI, we tested it on the problem of tracking either one or multiple people using a multi-camera system. 
\begin{figure}
\centering
    \includegraphics[width=0.45\textwidth]{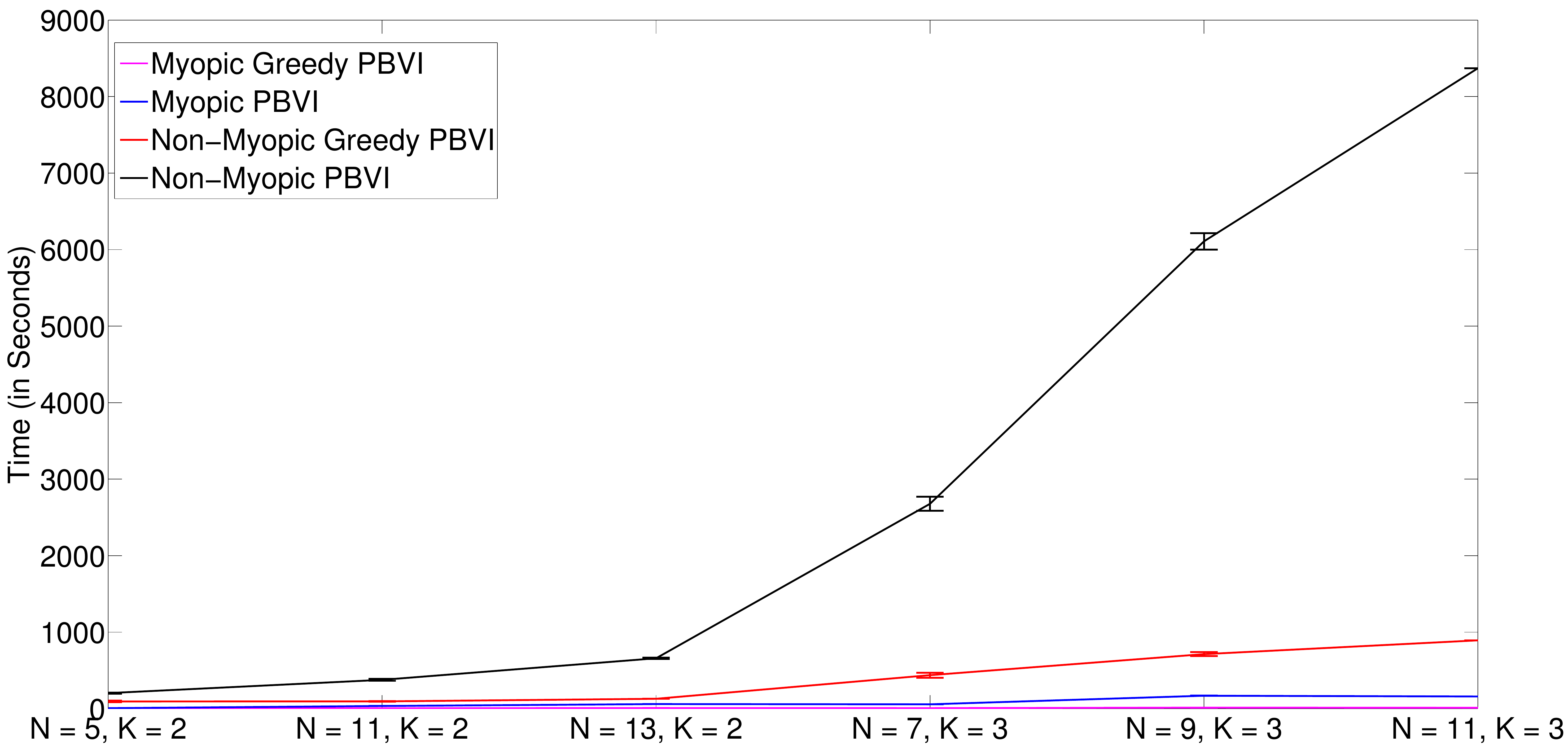}
    \caption{Runtimes for the different methods.}
    \label{fig:rt}
\end{figure}
The reward function is described as a set of $|S|$ vectors, $\Gamma^{\rho} = \{\alpha_{1} \dots \alpha_{|S|} \}$, with $\alpha_{i}(s) = 1$ if $s=i$ and $\alpha_{i}(s) = 0$ otherwise.
The initial belief is uniform across all states. We planned for horizon $h = 10$ with $\gamma = 0.99$.

As baselines, we tested against regular PBVI and \emph{myopic} versions of both greedy and regular PBVI that compute a policy assuming $h=1$ and use it at each timestep. Figure \ref{fig:rt} shows runtimes under different values of $N$ and $K$. Since multi-person tracking uses the value function obtained by solving a single-person POMDP, single and multi-person tracking have the same runtimes. These results demonstrate that greedy PBVI requires only a fraction of the computational cost of regular PBVI. In addition, the difference in the runtime grows quickly as the action space gets larger: for $N=5$ and $K=2$ greedy PBVI is twice as fast, while for $N=11, K=3$ it is approximately nine times as fast.  Thus, greedy PBVI enables much better scalability in the action space.
\begin{figure}
\centering
    \includegraphics[width=0.45\textwidth]{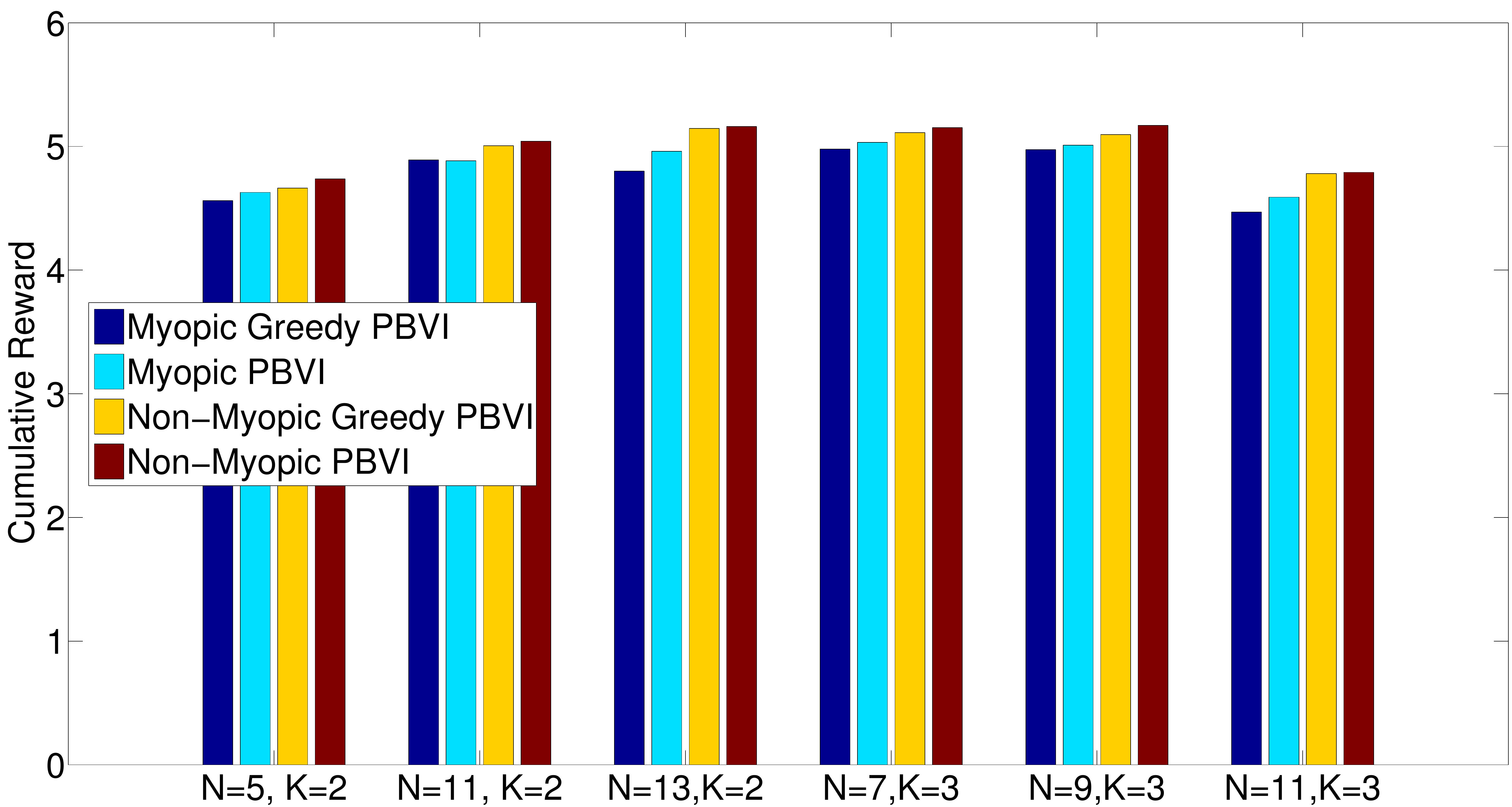} \\
    \hspace{-10mm} \includegraphics[width=0.5\textwidth]{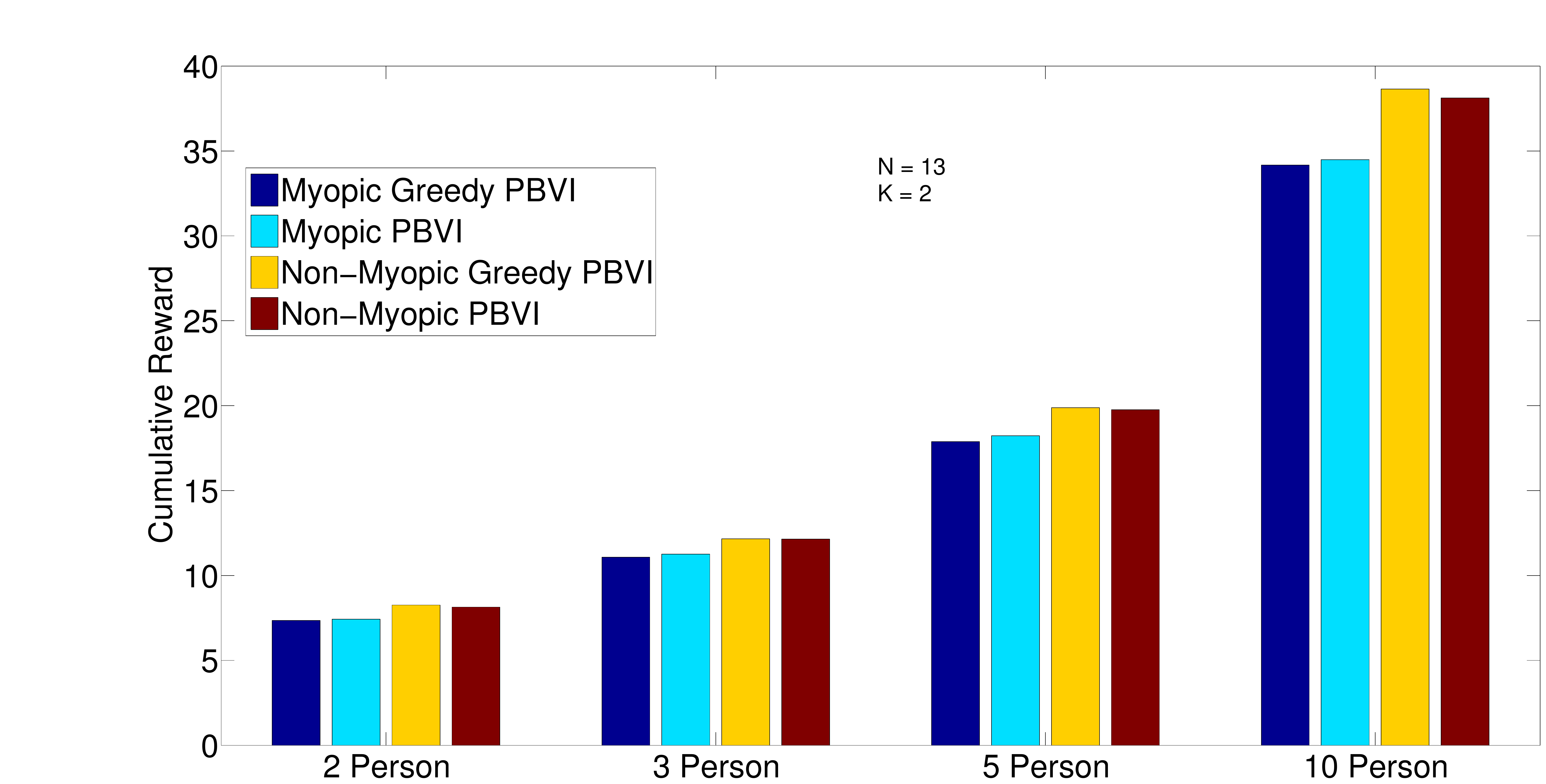} 
    \caption{Cumulative reward for single-person (top) and multi-person (bottom) tracking.}
  \label{fig:cum-rew}
\end{figure}
Figure \ref{fig:cum-rew}, which shows the cumulative reward under different values of $N$ and $K$ for single-person (top) and multi-person (bottom) tracking, verifies that greedy PBVI's speedup does not come at the expense of performance, as greedy PBVI accumulates nearly as much reward as regular PBVI. They also show that both PBVI and greedy PBVI benefit from non-myopic planning.  While the performance advantage of non-myopic planning is relatively modest, it increases with the number of cameras and people, which suggests that non-myopic planning is important to making active perception scalable.

Furthermore, an analysis of the resulting policies showed that myopic and non-myopic
policies differ qualitatively. A myopic policy, in order to minimize uncertainty in the
next step, tends to look where it believes the person to be. By contrast, a non-myopic
policy tends to proactively look where the person might go next, so as to more quickly
detect her new location when she moves. Consequently, non-myopic policies exhibit less
fluctuation in belief and accumulate more reward, as illustrated in Figure
\ref{fig:belFluct}. The blue lines mark when the agent chooses
the camera that can observe the cell occupied by the person. The red line plots the max of the agent's belief. The difference in fluctuation in belief is evident, as the max of the belief often drops below 0.5 for the myopic policy but rarely does so for the non-myopic policy.

\begin{figure}
\begin{center}
    \includegraphics[width=0.47\textwidth]{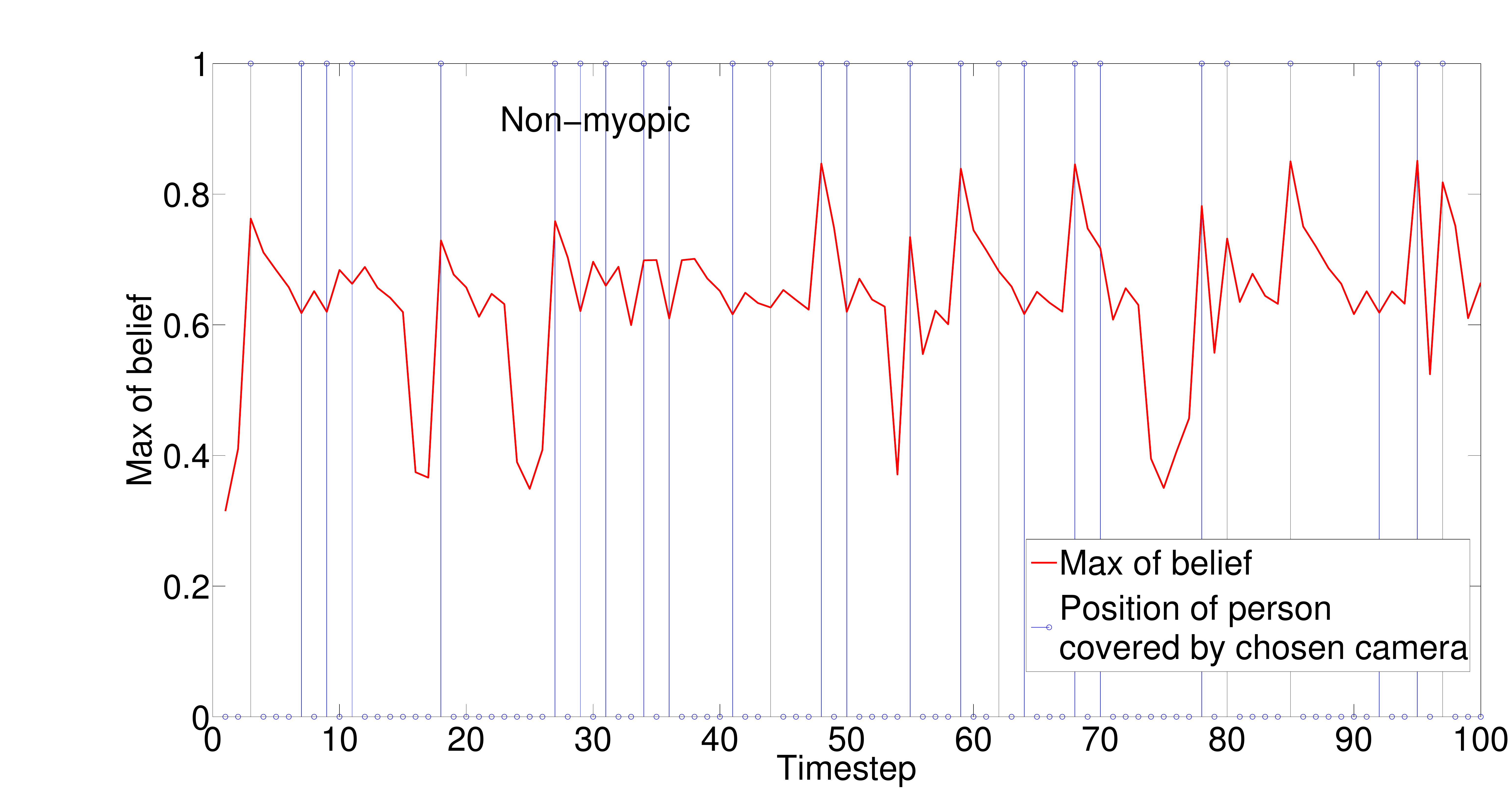} \hfill
    \includegraphics[width=0.47\textwidth]{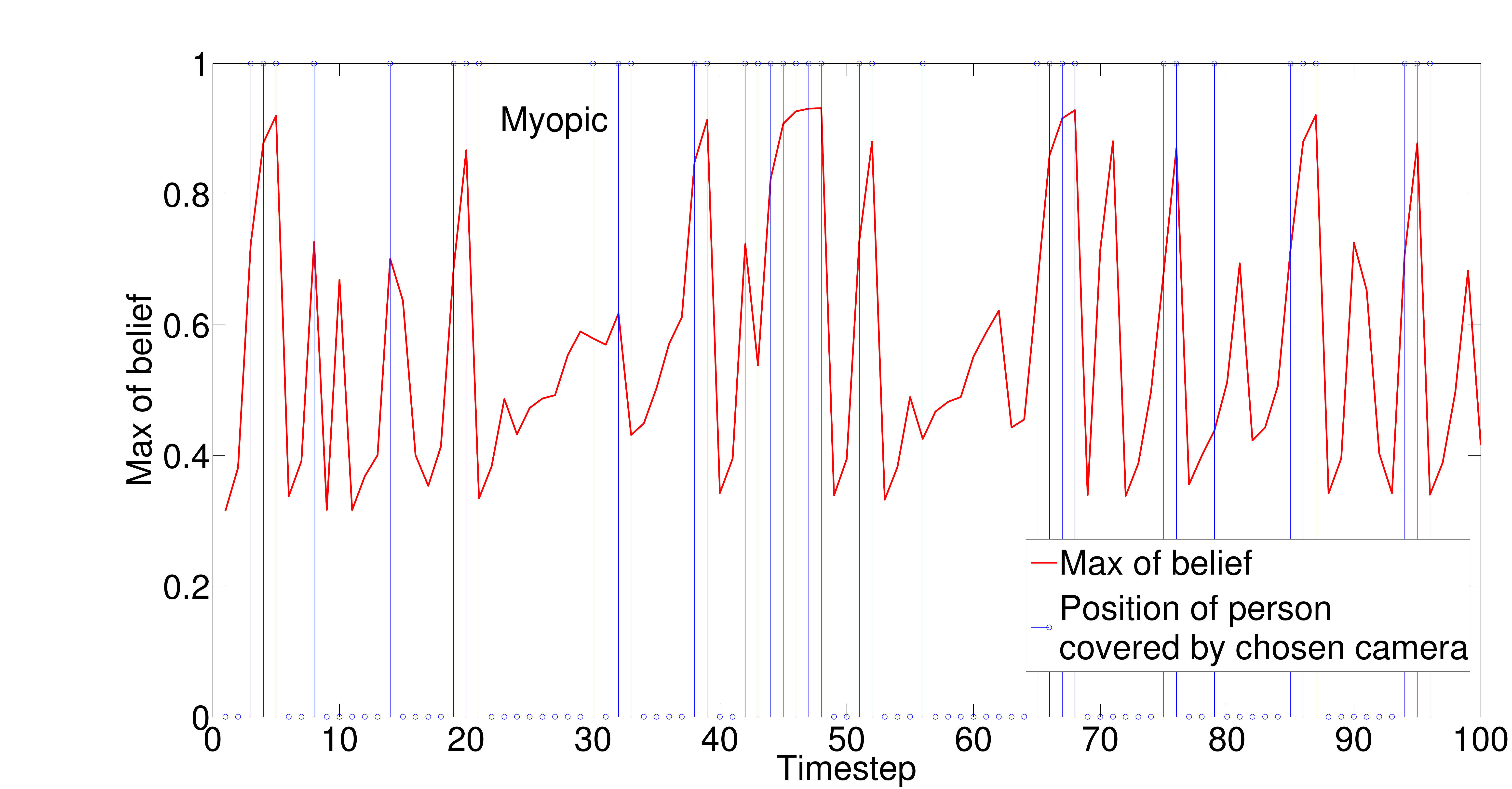} 
  \caption{Behavior of myopic vs.\ non-myopic policy.}
  \label{fig:belFluct}
\end{center}
\end{figure}

\section{Discussion \& Conclusions}

In this article, we addressed the problem of active perception, in which an agent must take actions to reduce uncertainty about a hidden variable while reasoning about various constraints. Specifically, we modeled the task of surveillance with multi-camera tracking systems in large urban spaces as an active perception task. Since the state of the environment is dynamic, we model this task as a POMDP to compute closed-loop non-myopic policies that can reason about the long-term consequences of selecting a subset of sensors. 

Formulating uncertainty reduction as an end in itself is a challenging task, as it breaks the PWLC property of the value function, which is imperative for solving POMDPs efficiently. $\rho$POMDP and POMDP-IR are two frameworks that allow formulating uncertainty reduction as an end in itself and does not break the PWLC property.

We showed that $\rho$POMDP and POMDP-IR are two equivalent frameworks for modeling active perception task. Thus, results that apply to one framework are also applicable to the other. While $\rho$POMDP does not restrict the definition of $\rho$ to a PWLC function, in this work we restrict the definition of $\rho$POMDP to a case where $\rho$ is approximated with a PWLC function, as it is not feasible to efficiently solve a $\rho$POMDP where the $\rho$ is not a PWLC function. 

We model the action space of the active perception POMDP as selecting $K$ out of $N$ sensors, where $K$ is the maximum number of sensors allowed by the resource constraints. Recent POMDP solvers enable scalability in the state space. However, for active perception, as the number of sensors grow, the action space grows exponentially. We proposed greedy PBVI, a POMDP planning method, that improves scalability in the action space of a POMDP. While we do not directly address the scaling in the observation space, we believe recent ideas on factorization of observation space \citep{tiagoFO} can be combined with our approach to improve scalability in state, action and observation space to solve active perception POMDPs.

By leveraging the theory of submodularity, we showed that the value function computed by greedy PBVI is guaranteed to have bounded error. Specifically, we extend Nemhauser's result on greedy maximization of submodular functions to long-term planning.
To apply these results to the active perception task, we showed that under certain conditions the value function of an active perception POMDP is submodular. One  such condition requires that the series future of observations be independent of each other given the state. While this is a strong condition, it is only a sufficient condition and not may not be a  necessary one. Thus, one line of future work is to attempt to relax this condition for proving the submodularity of the value function. Finally, we showed that, even with a PWLC approximation to the true value function, which is submodular, the error in the value function computed by greedy PBVI remains bounded, thus enabling us to compute efficiently value functions for active perception POMDP. 

Greedy PBVI is ideally suited for active perception POMDPs for which the value function is submodular. However, in real-life situations submodularity of value function might not always hold. For example, in our setting when there is occlusion, it is possible for combinations of sensors that when selected together yield higher utility than the sum of their utilities when selected individually. Similar case can arise when a mobile robots is trying to sense the best point of view to observe a scene that is occluded. Thus in cases like this, greedy PBVI might not return the best solution. 

Our empirical analysis established the critical factors involved in the performance active perception tasks. We showed that a belief-based formulation of uncertainty reduction beats a corresponding popular state-based reward baseline as well as other simple policies. While, the non-myopic policy beats the myopic one, the gain in certain cases the gain is marginal. However, in cases involving mobile sensors and budgeted constraints, non-myopic policies become critically important. Finally, experiments on a real-world dataset showed that the performance of greedy PBVI is similar to the existing methods but requires only a fraction of the computational cost, leading to much better scalability for solving active perception tasks.

\section{Appendix} 

\subsection{Results from Section 4}
\noindent{\textbf{Theorem 1}} 
\emph{Let $\mathbf{M}_{\rho}$ be a $\rho$POMDP and $\pi_{\rho}$ an arbitrary policy for $\mathbf{M}_{\rho}$.  Furthermore let $\mathbf{M}_{\mathit{IR}}$ = \textsc{reduce-pomdp-$\rho$-IR}$(\mathbf{M}_{\rho})$ and $\pi_{\mathit{IR}}$ = \textsc{reduce-policy-$\rho$-IR}$(\pi_{\rho})$. Then, for all $b$,
\begin{equation}
V_{t}^{\mathit{IR}}(b) = V_{t}^{\rho}(b),
\end{equation}
where $V_{t}^{\mathit{IR}}$ is the $t$-step value function for $\pi_{\mathit{IR}}$ and $V_{t}^{\rho}$ is the $t$-step value function for $\pi_{\rho}$.}
\begin{proof}
By induction on $t$.  To prove the base case, we observe that, from the definition of $\rho(b)$, 
\begin{equation}
V^{\rho}_{0}(b) = \rho(b) = \max_{\alpha_{\rho}^{a_p} \in \Gamma_{\rho}} \sum_{s} b(s)\alpha_{\rho}^{a_p}(s). \nonumber
\end{equation}
Since $\mathbf{M}_{\mathit{IR}}$ has a prediction action corresponding to each $\alpha_{\rho}^{a_p}$, thus the $a_p$ corresponding to $\alpha = \argmax_{\alpha_{\rho}^{a_p} \in \Gamma_{\rho}} \sum_{s} b(s)\alpha_{\rho}^{a_p}(s)$, must also maximize $\sum_{s}b(s)R(s,a_p)$.
Then, 
\begin{equation} \label{eq:base-case1}
\begin{split}
V^{\rho}_{0}(b) &= \max_{a_p} \sum_{s} b(s)R_{\mathit{IR}}(s,a_p) \\
			&= V^{\mathit{IR}}_{0}(b).   
			\end{split}
\end{equation}

For the inductive step, we assume that $V_{t-1}^{\mathit{IR}}(b) = V_{t-1}^{\rho}(b)$ and must show that $V_{t}^{\mathit{IR}}(b) = V_{t}^{\rho}(b)$.  Starting with $V_{t}^{\mathit{IR}}(b)$,
\begin{equation} \label{eq:ind-step1}
\begin{split}
V_{t}^{\mathit{IR}}(b) = & \max_{a_{p}}\sum_s b(s)R(s,a_p) \\ & \hspace{6mm}+ \sum_{z}  \Pr(\mathbf{z}|b,\pi^{n}_{\mathit{IR}}(b)) V_{t-1}^{\mathit{IR}}(b^{\pi^{n}_{\mathit{IR}}(b),\mathbf{z}}),  
\end{split}
\end{equation}
where $\pi^{n}_{\mathit{IR}}(b)$ denotes the normal action of the tuple specified by $\pi_{\mathit{IR}}(b)$ 
and: 
\begin{equation}
\resizebox{0.48\textwidth}{!}{$\Pr(\mathbf{z}|b,\pi^{n}_{\mathit{IR}}(b) ) = \sum_{s}\sum_{s''} O_{\mathit{IR}}(s'',\pi^{n}_{\mathit{IR}}(b),\mathbf{z})T_{\mathit{IR}}(s,\pi^{n}_{\mathit{IR}}(b),s'')b(s)$}.  \nonumber
\end{equation}

Using the reduction procedure, we can replace $T_{\mathit{IR}}$ and $O_{\mathit{IR}}$ and $\pi^{n}_{\mathit{IR}}(b)$ with their $\rho$POMDP counterparts on right hand side of the above equation:
\begin{equation}
\begin{split}
\resizebox{0.48\textwidth}{!}{$\Pr(\mathbf{z}|b,\pi^{n}_{\mathit{IR}}(b)) = \sum_{s}\sum_{s''} O_{\rho}(s'',\pi_{\rho}(b),z) T_{\rho}(s,{\pi_{\rho}}(b),s'')b(s)$}  \nonumber \\
\hspace{-40mm} = \Pr(\mathbf{z}|b,{\pi_{\rho}}(b)) \hspace{38mm}.
\end{split}
\end{equation}
Similarly, for the belief update equation,
\begin{equation}
\begin{split}
b^{\pi^{n}_{\mathit{IR}}(b),\mathbf{z}} &= \frac{O_{\mathit{IR}}(s',\pi^{n}_{\mathit{IR}}(b),\mathbf{z})}{\Pr(\mathbf{z}|\pi^{n}_{\mathit{IR}}(b),b)} \sum_{s} b(s)T_{\mathit{IR}}(s,\pi^{n}_{\mathit{IR}}(b),s') \\
		 &=  \frac{O_{\rho}(s',{\pi_{\rho}}(b),\mathbf{z})}{\Pr(\mathbf{z}|{\pi_{\rho}}(b),b)} \sum_{s} b(s)T_{\rho}(s,{\pi_{\rho}}(b),s') \\
		 &= b^{{\pi_{\rho}}(b),\mathbf{z}}. 
\end{split}
\end{equation}			
Substituting the above result in \eqref{eq:ind-step1} yields:
\begin{equation}
\resizebox{0.48\textwidth}{!}{$V_{t}^{\mathit{IR}}(b) = \max_{a_p}\sum_s b(s)R(s,a_p) + \sum_{\mathbf{z}}  Pr(\mathbf{z}|b,{\pi}_{\rho}(b)) V_{t-1}^{\mathit{IR}}(b^{{\pi_{\rho}}(b),\mathbf{z}})$}.
\end{equation}
 
Since the inductive assumption tells us that $V_{t-1}^{\mathit{IR}}(b) = V_{t-1}^{\rho}(b)$ and \eqref{eq:base-case1} shows that $\rho(b) = \max_{a_p} \sum_s b(s)R(s,a_p)$:
\begin{equation}
\begin{split}
V_{t}^{\mathit{IR}}(b) = &~ [\rho(b) +  \sum_{\mathbf{z}}  Pr(\mathbf{z}|b,{\pi_{\rho}}(b)) V_{t-1}^{\rho}(b^{{\pi_{\rho}}(b),\mathbf{z}}) ] \\ 
 = &~ V_{t}^{\rho}(b).
 \end{split}
\end{equation}
\qed
\end{proof}

\noindent{\textbf{Theorem 2}} \emph{Let $\mathbf{M}_{\mathit{IR}}$ be a POMDP-IR and $\pi_{\mathit{IR}} = \langle \mathbf{a}_{n}, a_p \rangle $ an policy for $\mathbf{M}_{\mathit{IR}}$, such that $a_{p} = \max_{a_{p}'}b(s)R(s,a_{p}')$.  Furthermore let $\mathbf{M}_{\rho}$ = \textsc{reduce-pomdp-IR-$\rho$}($\mathbf{M}_{\mathit{IR}})$ and $\pi_{\rho}$ = \textsc{reduce-policy-IR-$\rho$}($\pi_{\mathit{IR}})$. Then, for all $b$,
\begin{equation}
V_{t}^{\rho}(b) = V_{t}^{IR}(b),
\end{equation}
where $V_{t}^{\mathit{IR}}$ is the value of following $\pi_{\mathit{IR}}$ in $\mathbf{M}_{\mathit{IR}}$ and 
 $V_{t}^{\rho}$ is the value of following $\pi_{\rho}$ in $\mathbf{M}_{\rho}$.}
 
\begin{proof}
By induction on $t$.  To prove the base case, we observe that, from the definition of $\rho(b)$, 

\begin{eqnarray} \label{eq:base-case2}
\begin{split}
V_{0}^{IR}(b) &= \max_{a_{p}} \sum_{s} b(s)R(s,a_{p})  \\
&= \sum_s b(s) \alpha(s) \ \{\mbox{where $\alpha(s)$ is the $\alpha(s)$} \\  &  \hspace{-10mm}  \mbox{ corresponding to $a_p = \argmax_{a_p'}\sum_s b(s)R(s,a_{p}')$}. \} \\
&= \rho(b)  \\
&= V_{0}^{\rho}(b)
\end{split}
\end{eqnarray}

For the inductive step, we assume that $V_{t-1}^{\rho}(b) = V_{t-1}^{IR}(b)$ and must show that $V_{t}^{\rho}(b) = V_{t}^{IR}(b)$.  Starting with $V_{t}^{\rho}(b)$,
\begin{equation} \label{eq:ind-step2}
V_{t}^{\rho}(b) =  \rho(b) +    \sum_{\mathbf{z}}  Pr(\mathbf{z}|b,\pi_\rho(b)) V_{t-1}^{\rho}(b^{\pi_{\rho}(b),\mathbf{z}}),
\end{equation}
where $\pi^{n}_{\mathit{IR}}(b)$ denotes the normal action of the tuple specified by $\pi_{\mathit{IR}}(b)$ 
and:
\begin{equation}
\resizebox{0.48\textwidth}{!}{$Pr(\mathbf{z}|b,\pi_{\rho}(b) ) =  \sum_{s}\sum_{s''} O_{\rho}(s'',\pi_{\rho}(b),\mathbf{z})T_{\rho}(s,\pi_{\rho}(b),s'')b(s)$}. 
\end{equation}
From the reduction procedure, we can replace $T_{\rho}$ and $O_{\rho}$ and $\pi_{\rho}(b)$ with their POMDP-IR counterparts:
\begin{equation}
\begin{split}
\resizebox{0.48\textwidth}{!}{$Pr(\mathbf{z}|b,\pi_{\rho}(b)) = \sum_{s}\sum_{s''} O_{\mathit{IR}}(s'',\pi^{n}_{\mathit{IR}}(b),\mathbf{z}) T_{\mathit{IR}}(s,\pi^{n}_{\mathit{IR}}(b),s'')b(s)$}  \nonumber \\
\hspace{-40mm} = Pr(\mathbf{z}|b,{\pi_{\mathit{IR}}}(b)). \hspace{38mm}
\end{split}
\end{equation}
Similarly, for the belief update equation,
\begin{equation}
\begin{split}
b^{\pi_{\rho}(b),\mathbf{z}} &= \frac{O_{\rho}(s',{\pi_{\rho}}(b),\mathbf{z})}{Pr(\mathbf{z}|{\pi_{\rho}}(b),b)} \sum_{s} b(s)T_{\rho}(s,{\pi_{\rho}}(b),s') \\
&= \frac{O_{\mathit{IR}}(s',\pi^{n}_{\mathit{IR}}(b),\mathbf{z})}{Pr(\mathbf{z}|\pi^{n}_{\mathit{IR}}(b),b)} \sum_{s} b(s)T_{\mathit{IR}}(s,\pi^{n}_{\mathit{IR}}(b),s') \\
		 &= b^{{\pi_{\mathit{IR}}}(b),\mathbf{z}}. 
\end{split}
\end{equation}			
Substituting the above result in \eqref{eq:ind-step2} yields:
\begin{equation}
\begin{split}
V_{t}^{\rho}(b) = \rho(b) +  \sum_{\mathbf{z}}  Pr(\mathbf{z}|b,{\pi}_{\mathit{IR}}(b)) V_{t-1}^{\mathit{IR}}(b^{{\pi_{\mathit{IR}}}(b),\mathbf{z}}).
\end{split}
\end{equation}
 
Since the inductive assumption tells us that $V_{t-1}^{\rho}(b) = V_{t-1}^{\mathit{IR}}(b) $ and \eqref{eq:base-case2} shows that $\max_{a_p} \sum_{s}b(s)R(s,a_p) = \rho(b)$:
\begin{equation}
\begin{split}
V_{t}^{\rho}(b) &= [\max_{a_p} \sum_{s}b(s)R(s,a_p) \\ & \hspace{20mm} + \sum_{\mathbf{z}}  Pr(\mathbf{z}|b,{\pi_{\mathit{IR}}}(b)) V_{t-1}^{IR}(b^{{\pi_{\mathit{IR}}}(b),\mathbf{z}}) ] \\ 
 = & \ V_{t}^{IR}(b). \nonumber
\end{split}
\end{equation}
\qed
\end{proof}

\subsection{Results from subsection 6.1}

The following Lemma proves that the error in the value function remains bounded after application of $\mathfrak{B}^{G}$.

\begin{lemma}  \label{eoptimal2}
If for all $b$, $\rho(b) \geq 0$, 
\begin{equation}  \label{eq:assum}
V^{\pi}_{t}(b) \geq (1 - \epsilon) V^{*}_{t}(b), 
\end{equation} 
and $Q^{\pi}_{t}(b,\mathfrak{a})$ is non-negative, monotone, and submodular in $\mathfrak{a}$, then, for $\epsilon \in [0,1]$,
\begin{equation}
(\mathfrak{B}^{G}V^{\pi}_{t})(b) \geq (1-e^{-1})(1 - \epsilon) (\mathfrak{B}^{G}  V^{*}_{t})(b).
\end{equation}
\end{lemma}
\begin{proof}
Starting from \eqref{eq:assum} and, for a given $\mathfrak{a}$, on both sides multiplying $\gamma \geq 0$, taking the expectation over $\mathbf{z}$, and adding $\rho(b)$ (since $\rho(b) \geq 0$ and $\epsilon \leq 1$):
\begin{equation*}
 \rho(b) +  \gamma \mathbb{E}_{\mathbf{z}|b,\mathfrak{a}} [V^{\pi}_{t}(b^{a,\mathbf{z}})] \geq (1-\epsilon)(\rho(b) + \gamma \mathbb{E}_{\mathbf{z}|b,\mathfrak{a}} [V^{*}_{t}(b^{a,\mathbf{z}})]). 
\end{equation*}
From the definition of $Q^\pi_t$ \eqref{eq:q-def}, we thus have:
\begin{equation} \label{eq:toshow}
Q^\pi_{t+1}(b,\mathfrak{a}) \geq (1-\epsilon) Q^*_{t+1}(b,\mathfrak{a}) \ \ \forall \mathfrak{a}. 
\end{equation}
From Theorem \ref{theorem1}, we know
\begin{equation} 
Q^\pi_{t+1}(b,\mathfrak{a}^G_\pi) \geq (1-e^{-1})Q^\pi_{t+1}(b,\mathfrak{a}^*_\pi),
\end{equation} 
where $\mathfrak{a}^G_\pi = \mathtt{greedy}\hbox{-}\mathtt{argmax}(Q^\pi_{t+1}(b,\cdot),A^+,K)$ and $\mathfrak{a}^*_\pi = \argmax_{\mathfrak{a}} Q^\pi_{t+1}(b,\mathfrak{a}$). Since $Q^\pi_{t+1}(b,\mathfrak{a}^*_\pi) \geq Q^\pi_{t+1}(b,\mathfrak{a})$ for any $\mathfrak{a}$,
\begin{equation} \label{eq:almost}
Q^\pi_{t+1}(b,\mathfrak{a}^G_\pi) \geq (1-e^{-1})Q^\pi_{t+1}(b,\mathfrak{a}^G_*),
\end{equation}
where $\mathfrak{a}^{G}_{*} =  \mathtt{greedy}\hbox{-}\mathtt{argmax}(Q^{*}_{t}(b,\cdot),A^+,K)$. Finally, \eqref{eq:toshow} implies that $Q^\pi_{t+1}(b,\mathfrak{a}^G_*) \geq (1 - \epsilon) Q^*_{t+1}(b,\mathfrak{a}^G_*)$, so:
\begin{equation}
\begin{split}
& Q^\pi_{t+1}(b,\mathfrak{a}^G_\pi) \geq (1-e^{-1})(1 - \epsilon) Q^*_{t+1}(b,\mathfrak{a}^G_*)\\
& (\mathfrak{B}^{G}V^{\pi}_{t})(b) \geq (1-e^{-1})(1 - \epsilon) (\mathfrak{B}^{G}V^{*}_{t})(b).
\end{split}
\end{equation}
\qed
\end{proof}

Using Corollary \ref{Lemma1} and Lemma \ref{eoptimal2}, we can prove Theorem \ref{th:bound}.

\noindent{\textbf{Theorem 4}}
If for all policies $\pi$, $Q^{\pi}_{t}(b,\mathfrak{a})$ is non-negative, monotone and submodular in $\mathfrak{a}$, then for all $b$,
\begin{equation} \label{eq:bound1}
V^{G}_{t}(b) \geq (1 - e^{-1})^{2t}V^{*}_{t}(b).
\end{equation}
\begin{proof}
By induction on $t$. The base case, $t=0$, holds because $V^{G}_{0}(b) = \rho(b) = V^{*}_{0}(b)$. 

In the inductive step, for all $b$, we assume that 
\begin{equation} \label{eq:ind-ass}
V^{G}_{t-1}(b) \geq (1 - e^{-1})^{2t-2}V^{*}_{t-1}(b),
\end{equation}
and must show that
\begin{equation}
V^{G}_{t}(b) \geq (1 - e^{-1})^{2t}V^{*}_{t}(b).
\end{equation}
Applying Lemma \ref{eoptimal2} with $V^{\pi}_{t} = V^{G}_{t-1}$ 
and $(1-\epsilon) = (1 - e^{-1})^{2t-2} $ 
to \eqref{eq:ind-ass}:
\begin{equation*}
\begin{split}
 (\mathfrak{B}^{G}V^{G}_{t-1})(b) 
&
 \geq (1 - e^{-1})^{2t-2}(1 - e^{-1})(\mathfrak{B}^{G}V^{*}_{t-1})(b) \\
 V^{G}_{t}(b)
&
 \geq (1-e^{-1})^{2t-1}(\mathfrak{B}^{G}V^{*}_{t-1})(b). 
\end{split}
\end{equation*}
Now applying Corollary \ref{Lemma1} with $V^{\pi}_{t-1} = V^{*}_{t-1}$:
\begin{equation}
\begin{split}
& V^{G}_{t}(b) \geq (1-e^{-1})^{2t-1}(1-e^{-1})(\mathfrak{B}^{*}V^{*}_{t-1})(b)\\
& V^{G}_{t}(b) \geq (1 - e^{-1})^{2t}V^{*}_{t}(b). 
\end{split}
\end{equation}
\qed
\end{proof}

\subsection{Results from subsection 6.2}
Proving that $Q^{\pi}_{t}(b,\mathfrak{a})$ is submodular in $\mathfrak{a}$ requires three steps.
First, we show that $G^{\pi}_{k}(b^{t},\mathfrak{a}^{t})$ equals the \emph{conditional entropy} of $b^{k}$ over $s^{k}$ given $\mathbf{z}^{t:k}$ and $\mathfrak{a}^t$. 
Second, we show that, under certain conditions, conditional entropy is a submodular set function.
Third, we combine these two results to show that $Q^{\pi}_{t}(b,\mathfrak{a})$ is submodular.

\begin{lemma}\label{lem:GcondEnt}
If $\rho(b) = -H_{b}(s)$, then the expected reward at each time step equals the negative discounted conditional entropy of $b^k$ over $s^{k}$ given $\mathbf{z}^{t:k}$:
\begin{equation*}
\begin{split}
G^{\pi}_{k}(b^{t},\mathfrak{a}^{t}) &= - \gamma^{t - k} (H_{b^{k}}(s^{k}|\mathbf{z}^{t:k},\mathfrak{a}^{t})) \\ &= - \gamma^{t - k} (H_{b^{k}}^{\mathfrak{a}^t}(s^{k}|\mathbf{z}^{t:k})) \  \forall \ \pi . 
\end{split}
\end{equation*}
\end{lemma}
\begin{proof}
To prove the above lemma, we take help of some additional notations and definitions, first we must elaborate on the definition of $b^k$:
\begin{equation}
\resizebox{0.43\textwidth}{!}{$b^{k}(s^{k}) \triangleq Pr(s^{k}|b^{t},\mathfrak{a}^{t},\pi,\mathbf{z}^{t:k}) = \frac{Pr(\mathbf{z}^{t:k},s^{k}|b^{t},\mathfrak{a}^{t},\pi)}{Pr(\mathbf{z}^{t:k}|b^{t},\mathfrak{a}^{t},\pi)}$}. 
\end{equation}
For notational convenience, we also write this as:
\begin{equation}
b^{k}(s^{k}) \triangleq \frac{Pr_{b^{t},\mathfrak{a}^{t}}^{\pi}(\mathbf{z}^{t:k},s^{k})}{Pr_{b^{t},\mathfrak{a}^{t}}^{\pi}(\mathbf{z}^{t:k})}.
\end{equation}
The entropy of $b^k$ is thus:
\begin{equation}
H_{b^{k}}(s^{k}) = \sum_{s^{k}}\frac{Pr_{b^{t},\mathfrak{a}^{t}}^{\pi}(\mathbf{z}^{t:k},s^{k})}{Pr_{b^{t},\mathfrak{a}^{t}}^{\pi}(\mathbf{z}^{t:k})} \log(\frac{Pr_{b^{t},\mathfrak{a}^{t}}^{\pi}(\mathbf{z}^{t:k},s^{k})}{Pr_{b^{t},\mathfrak{a}^{t}}^{\pi}(\mathbf{z}^{t:k})}), \nonumber
\end{equation}
and the conditional entropy of $b^k$ over $s^k$ given $\mathbf{z}^{t:k}$ is:
\begin{equation}
\resizebox{0.48\textwidth}{!} {$H_{b^{k}}^{\mathfrak{a}^{t}}(s^k|\mathbf{z}^{t:k}) = 
\sum_{s^{k}}
\sum_{\mathbf{z}^{t:k}} Pr_{b^{t},\mathfrak{a}^{t}}^{\pi}(\mathbf{z}^{t:k},s^{k}) \log(\frac{Pr_{b^{t},\mathfrak{a}^{t}}^{\pi}(\mathbf{z}^{t:k},s^{k})}{Pr_{b^{t},\mathfrak{a}^{t}}^{\pi}(\mathbf{z}^{t:k})}) $} \nonumber
.
\end{equation}

Then, by definition of $G^{\pi}_{k}(b^{t},\mathfrak{a}^{t})$,
\begin{equation}
\begin{split}
& G^{\pi}_{k}(b^{t},\mathfrak{a}^{t}) = \gamma^{(t-k)}(- \sum_{\mathbf{z}^{t:k}} Pr_{b^{t},\mathfrak{a}^{t}}^{\pi}(\mathbf{z}^{t:k}) H_{b^{k}}(s^{k})) \\
& \mbox{By definition of entropy,} \\
& \resizebox{0.48\textwidth}{!}{$ =\gamma^{t-k} \sum_{\mathbf{z}^{t:k}} Pr_{b^{t},\mathfrak{a}^{t}}^{\pi}(\mathbf{z}^{t:k}) \Bigg[\sum_{s^{k}}\frac{Pr_{b^{t},\mathfrak{a}^{t}}^{\pi}(\mathbf{z}^{t:k},s^{k})}{Pr_{b^{t},\mathfrak{a}^{t}}^{\pi}(\mathbf{z}^{t:k})} \log(\frac{Pr_{b^{t},\mathfrak{a}^{t}}^{\pi}(\mathbf{z}^{t:k},s^{k})}{Pr_{b^{t},\mathfrak{a}^{t}}^{\pi}(\mathbf{z}^{t:k})})\Bigg]$} \\
& = \gamma^{t-k} \sum_{\mathbf{z}^{t:k}} \Bigg[\sum_{s^{k}} Pr_{b^{t},\mathfrak{a}^{t}}^{\pi}(\mathbf{z}^{t:k},s^{k})\log(\frac{Pr_{b^{t},\mathfrak{a}^{t}}^{\pi}(\mathbf{z}^{t:k},s^{k})}{Pr_{b^{t},\mathfrak{a}^{t}}^{\pi}(\mathbf{z}^{t:k})}) \Bigg] \\
& \mbox{By definition of conditional entropy,} \\
& = \gamma^{t-k} (- H_{b^{k}}^{\mathfrak{a}^t}(s^{k}|\mathbf{z}^{t:k})).  \hspace{0.25\textwidth} \qed \\  \nonumber 
\end{split}
\end{equation}

\end{proof}

\begin{lemma}  \label{Lemma4}
If $\mathfrak{z}$ is conditionally independent given $s$ then $ - H(s|\mathfrak{z})$ is submodular in $\mathfrak{z}$, i.e., for any two observations $\mathfrak{z}_{M}$ and $\mathfrak{z}_{N}$,
\begin{equation}\label{toprove}
H(s|\mathfrak{z}_{M} \cup \mathfrak{z}_{N}) + H(s|\mathfrak{z}_{M} \cap \mathfrak{z}_{N}) \geq  H(s|\mathfrak{z}_{M}) + H(s|\mathfrak{z}_{N}). \\
\end{equation}
\end{lemma}
\begin{proof}
By Bayes' rule for conditional entropy \citep{cover1991entropy}: 
\begin{equation} \label{eq:chaineq1}
\resizebox{0.43\textwidth}{!}{ $H(s|\mathfrak{z}_{M} \cup \mathfrak{z}_{N}) =  H(\mathfrak{z}_{M} \cup \mathfrak{z}_{N}|s) + H(s) -  H(\mathfrak{z}_{M} \cup \mathfrak{z}_{N}).$}
\end{equation}
Using conditional independence, we know $H(\mathfrak{z}_{M} \cup \mathfrak{z}_{N}|s) = H(\mathfrak{z}_{M}|s) + H(\mathfrak{z}_{N}|s)$. Substituting this in \eqref{eq:chaineq1}, we get:
\begin{equation}
\resizebox{0.43\textwidth}{!}{$H(s|\mathfrak{z}_{M} \cup \mathfrak{z}_{N})  =  H(\mathfrak{z}_{M}|s) + H(\mathfrak{z}_{N}|s) + H(s) -  H(\mathfrak{z}_{M} \cup \mathfrak{z}_{N}).$}
\end{equation}

By Bayes' rule for conditional entropy: 
\begin{equation} \label{eq:chaineq2}
H(s|\mathfrak{z}_{M} \cap \mathfrak{z}_{N}) = H(\mathfrak{z}_{M} \cap \mathfrak{z}_{N}| s) + H(s) - H(\mathfrak{z}_{M} \cap \mathfrak{z}_{N}).
\end{equation}

Adding \eqref{eq:chaineq1} and \eqref{eq:chaineq2}: 
\begin{equation}  \label{eq:addche1and2}
\begin{split}
H(s|\mathfrak{z}_{M} \cap \mathfrak{z}_{N}) + H(s|\mathfrak{z}_{M} \cup \mathfrak{z}_{N}) &= H(\mathfrak{z}_{M}|s) + H(\mathfrak{z}_{N}|s) \\ & \hspace{-10mm}+ H(\mathfrak{z}_{M} \cap \mathfrak{z}_{N}| s) +  2H(s) \\ & \hspace{-10mm}- H(\mathfrak{z}_{M} \cup \mathfrak{z}_{N}) - H(\mathfrak{z}_{M} \cap \mathfrak{z}_{N}).
\end{split}
\end{equation}

By Bayes' rule for conditional entropy: 
\begin{equation}
\begin{split}
H(\mathfrak{z}_{M}|s) &= H(s|\mathfrak{z}_{M}) + H(\mathfrak{z}_{M}) - H(s), \mbox{and}  \\
H(\mathfrak{z}_{N}|s) &= H(s|\mathfrak{z}_{N}) + H(\mathfrak{z}_{N}) - H(s)
\end{split}
\end{equation}

Substituting $H(\mathfrak{z}_{M}|s)$ and $H(\mathfrak{z}_{N}|s)$  in \eqref{eq:addche1and2}: 
\begin{equation}
\begin{split}
H(s|\mathfrak{z}_{M} \cap \mathfrak{z}_{N}) + H(s|\mathfrak{z}_{M} \cup \mathfrak{z}_{N}) &= H(s|\mathfrak{z}_{M}) + H(s|\mathfrak{z}_{N}) \\ &\hspace{-10mm}+ H(\mathfrak{z}_{M} \cap \mathfrak{z}_{N}| s) + [H(\mathfrak{z}_{M})  \\ & \hspace{-25mm}+ H(\mathfrak{z}_{N}) - H(\mathfrak{z}_{M} \cup \mathfrak{z}_{N}) - H(\mathfrak{z}_{M} \cap \mathfrak{z}_{N})].  \nonumber
\end{split}
\end{equation}

Since entropy is submodular $[H(\mathfrak{z}_{M}) + H(\mathfrak{z}_{N}) - H(\mathfrak{z}_{M} \cup \mathfrak{z}_{N}) - H(\mathfrak{z}_{M} \cap \mathfrak{z}_{N})]$ is positive and since entropy is positive, $H(\mathfrak{z}_{M} \cap \mathfrak{z}_{N}| s)$ is positive. Thus, 
\begin{equation}
\begin{split}
H(s|\mathfrak{z}_{M} \cap \mathfrak{z}_{N}) + H(s|\mathfrak{z}_{M} \cup \mathfrak{z}_{N}) &= H(s|\mathfrak{z}_{M}) + H(s|\mathfrak{z}_{N}) \\ &+ \mbox{a positive term}. \nonumber
\end{split}
\end{equation}
This implies $H(s|\mathfrak{z}_{M} \cup \mathfrak{z}_{N}) + H(s|\mathfrak{z}_{M} \cap \mathfrak{z}_{N}) \geq  H(s|\mathfrak{z}_{M}) + H(s|\mathfrak{z}_{N})$. \qed
\end{proof}

\begin{lemma} \label{Lemma2}
If $\mathfrak{z}^{t:k}$ is conditionally independent given ${s}^{k}$ and $\rho(b) = - H_b(s)$, then $G^{\pi}_{k}(b^{t},\mathfrak{a}^{t})$ is submodular in $\mathfrak{a}^{t}$ $\forall \ \pi$.
\end{lemma}
\begin{proof}
Let $\mathfrak{a}_{M}^{t}$ and $\mathfrak{a}_{N}^{t}$ be two actions and $\mathfrak{z}_{M}^{t:k}$ and $\mathfrak{z}_{N}^{t:k}$ the observations they induce.
Then, from Lemma \ref{lem:GcondEnt}, 
\begin{equation}
\begin{split}
G^\pi_{k}(b^{t},\mathfrak{a}_{M}^{t}) & = \gamma^{(t-k)}(-H_{b^{k}}^{\mathfrak{a}^t}(s^{k}|\mathfrak{z}^{t:k}_{M})). \qquad   \\ 
\end{split}
\end{equation}
\begin{equation}
\begin{split}
& \mbox{From Lemma \ref{Lemma4}, } \\
& H_{b^{k}}^{\mathfrak{a}^t}(s^{k}|\mathfrak{z}_{M}^{t:k} \cup \mathfrak{z}_{N}^{t:k}) + H_{b^{k}}^{\mathfrak{a}^t}(s^{k}|\mathfrak{z}_{M}^{t:k} \cap \mathfrak{z}_{N}^{t:k}) \\ & \hspace{35mm}\geq  H_{b^{k}}^{\mathfrak{a}^t}(s^{k}|\mathfrak{z}_{M}^{t:k}) + H_{b^{k}}^{\mathfrak{a}^t}(s^{k}|\mathfrak{z}_{N}^{t:k}) \\
& \mbox{Multiplying by} - \gamma^{t-k} \mbox{on both sides and} \\ &\mbox{using definition of $G$}   \\
& G^\pi_{k}(b^{t},\mathfrak{a}_{M}^{t} \cup \mathfrak{a}_{N}^{t}) + G^\pi_{k}(b^{t}, \mathfrak{a}_{N}^{t} \cap \mathfrak{a}_{M}^{t}) \\ & \hspace{35mm}\leq  G^\pi_{k}(b^{t},\mathfrak{a}_{M}^{t}) + G^\pi_{k}(b^{t},\mathfrak{a}_{N}^{t}). \nonumber \\ 
\end{split}
\end{equation}
\qed
\end{proof}

\noindent{\textbf{Theorem 5}}
\emph{If $\mathfrak{z}^{t:k}$ is conditionally independent given ${s}^{k}$ and $\rho(b) = - H_b(s)$, then $Q^{\pi}_{t}(b,\mathfrak{a})$ is submodular in $\mathfrak{a}$, for all $\pi$.}
\begin{proof}
$\rho(b)$ is trivially submodular in $\mathfrak{a}$ because it is independent of $\mathfrak{a}$.  Furthermore, Lemma \ref{Lemma2} shows that $G_{k}^{\pi}(b^t,\mathfrak{a}^{t})$ is submodular in $\mathfrak{a}^{t}$.  Since a positively weighted sum of submodular functions is also submodular \citep{krause14survey}, this implies that $\sum_{k = 1}^{t-1}G^{\pi}_{k}(b^t,\mathfrak{a}^{t})$ and thus $Q^{\pi}_{t}(b,\mathfrak{a})$ are also submodular in $\mathfrak{a}$.
\qed
\end{proof}

\begin{lemma} \label{lem:mono}
If $V^{\pi}_{t}$ is convex over the belief space for all $t$, then ${Q}^{\pi}_{t}(b,\mathfrak{a})$ is monotone in $\mathfrak{a}$, i.e., for all $b$ and $\mathfrak{a}_{M} \subseteq \mathfrak{a}_{N}$, ${Q}^{\pi}_{t}(b,\mathfrak{a}_{M}) \leq {Q}^{\pi}_{t}(b,\mathfrak{a}_{N})$. 
\end{lemma}
\begin{proof}
 By definition of $Q^{\pi}_{t}(b,\mathfrak{a})$, 
\begin{equation}
{Q}^{\pi}_{t}(b,\mathfrak{a}_{M}) = [\rho(b) + \gamma \mathbb{E}_{\mathfrak{z}_{M}}[{V}^{\pi}_{t-1}(b^{\mathfrak{a}_{M},\mathfrak{z}_{M}})|b,\mathfrak{a}_{M}]]. 
\end{equation}
Since $\rho(b)$ is independent of $\mathfrak{a}_{M}$, we need only show that the second term is monotone in $\mathfrak{a}$. Let $\mathfrak{a}_P = \mathfrak{a}_N \setminus \mathfrak{a}_M$ and
\begin{equation}
F_{b}^{\pi}(\mathfrak{a}_{N}) = \mathbb{E}_{\mathfrak{z}_{N}}[{V}^{\pi}_{t-1}(b^{\mathfrak{a}_{N},\mathfrak{z}_{N}})||b,\mathfrak{a}_{N}]. 
\end{equation}
Since $\mathfrak{a}_{N} = \{\mathfrak{a}_{M} \cup \mathfrak{a}_{P}\}$,
\begin{equation}
F_{b}^{\pi}(\mathfrak{a}_{N}) = \mathbb{E}_{\{\mathfrak{z}_{M},\mathfrak{z}_{P}\}}[V^{\pi}_{t-1}(b^{\{\mathfrak{a}_{M},\mathfrak{a}_{P}\},\{\mathfrak{z}_{M},\mathfrak{z}_{P}\}})|b,\{\mathfrak{a}_{M},\mathfrak{a}_{P}\}]. \nonumber
\end{equation}
Separating expectations, 
\begin{equation}
F_{b}^{\pi}(\mathfrak{a}_{N}) = \mathbb{E}_{\mathfrak{z}_{M}}[\mathbb{E}_{\mathfrak{z}_{P}}[{V}^{\pi}_{t-1}(b^{\{\mathfrak{a}_{M},\mathfrak{a}_{P}\},\{\mathfrak{z}_{M},\mathfrak{z}_{P}\}})|b,\mathfrak{a}_{P}]|b,\mathfrak{a}_{M}] \nonumber
\end{equation}
Applying Jensen's inequality, since $V^{\pi}_{t-1}$ is convex,
\begin{equation}
F_{b}^{\pi}(\mathfrak{a}_{N}) \geq \mathbb{E}_{\mathfrak{z}_{M}}[{V}^{\pi}_{t-1}(\mathbb{E}_{\mathfrak{z}_{P}}[b^{\mathfrak{a}_{M},\mathfrak{a}_{P},\mathfrak{z}_{M},\mathfrak{z}_{P}}|b,\mathfrak{a}_{P}])||b,\mathfrak{a}_{M}] \nonumber
\end{equation} 
Since the expectation of the posterior is the prior,
\begin{equation}
\begin{split}
& F_{b}^{\pi}(\mathfrak{a}_{N}) \geq \mathbb{E}_{\mathfrak{z}_{M}} [{V}^{\pi}_{t-1}(b^{\mathfrak{a}_{M},\mathfrak{z}_{M}})|b,\mathfrak{a}_{M}] \\
& F_{b}^{\pi}(\mathfrak{a}_{N}) \geq F_{b}^{\pi}(\mathfrak{a}_{M}).
\end{split}
\end{equation}
Consequently, we have:
\begin{equation}
\begin{split}
& {\rho}(b) + \gamma^{t-k}F_{b}^{\pi}(\mathfrak{a}_{N}) \geq {\rho}(b) + \gamma^{t-k}F_{b}^{\pi}(\mathfrak{a}_{M}) \\ 
& {Q}_{t}^{\pi}(b,\mathfrak{a}_{N}) \geq  {Q}_{t}^{\pi}(b,\mathfrak{a}_{M}).
\end{split}
\end{equation}
\end{proof}



\noindent{\textbf{Theorem 6}}
\emph{If $\mathfrak{z}^{t:k}$ is conditionally independent given $s^{k}$, $V^{\pi}_{t}$ is convex over the belief space for all $t, \pi$, and $\rho(b) = - H_b(s) + 
log(\frac{1}{|S|})$, then for all $b$, }
\begin{equation}
V^{G}_{t}(b) \geq (1 - e^{-1})^{2t}V^{*}_{t}(b).
\end{equation}
\begin{proof}
Follows from Theorem \ref{th:bound}, given $Q^{G}_{t}(b,\mathfrak{a})$ is non-negative, monotone and submodular. For $\rho(b) = -H_{b}({s}) + \log(\frac{1}{|S|})$, it is easy to see that $Q^{G}_{t}(b,\mathfrak{a})$ is non-negative, as entropy is always positive \citep{cover1991entropy} and is maximum when $b(s) = \frac{1}{|S|}$ for all $s$ \citep{cover1991entropy}. Theorem \ref{th:submod} showed that $Q^{G}_{t}(b,\mathfrak{a})$ is submodular if $\rho(b) = -H_{b}({s})$. The monotonicity of $Q^{G}_{t}$ follows from the condition that $V^{\pi}_{t}$ is convex in belief space; Lemma \ref{lem:mono} then shows that $Q^{G}_{t}(b,\mathfrak{a})$ is monotone in $\mathfrak{a}$. \qed
\end{proof}

\subsection{Results from subsection 6.3}

\begin{lemma}
For all beliefs $b$, the error between $V^{G}_{t}(b)$ and $\tilde{V}^{G}_{t}(b)$ is bounded by $\frac{C \delta^{\alpha}}{1 - \gamma}$. That is, $||V^{G}_{t} - \tilde{V}^{G}_{t}||_{\infty} \leq \frac{C \delta^{\alpha}}{1 - \gamma} $.
\end{lemma}
\begin{proof} 
Follows exactly the strategy by \cite{Mauricio} used to prove \eqref{nipseq}, which places no conditions on $\pi$ and thus holds as long as $\mathfrak{B}^G$ is a contraction mapping.  Since for any policy the Bellman operator $\mathfrak{B}^{\pi}$ defined as:
\begin{equation}
(\mathfrak{B}^{\pi}V_{t-1})(b) = [\rho(b,\mathfrak{a}_{\pi}) + \gamma\sum_{\mathbf{z} \in \Omega} \Pr(\mathbf{z}|\mathfrak{a}_{\pi},b)V_{t-1}(b^{\mathfrak{a}_{\pi},\mathbf{z}})],  \nonumber
\end{equation} 
is a contraction mapping \citep{bertsekas}, the bound holds for $\tilde{V}^G_{t}$.
\qed
\end{proof}

Let $\eta = \frac{C \delta^{\alpha}}{1 - \gamma}$ and $\tilde{Q}^{*}_{t}(b,\mathfrak{a}) = \tilde{\rho}(b) + \sum_{\mathbf{z}}\Pr(\mathbf{z}|b,\mathfrak{a})\tilde{V}^{*}_{t-1}(b^{\mathfrak{a},\mathbf{z}})$ denote the value of taking action $\mathfrak{a}$ in belief $b$ under an optimal policy. Let $\tilde{Q}^{G}_{t}(b,\mathfrak{a}) = \tilde{\rho}(b) + \sum_{\mathbf{z}}\Pr(\mathbf{z}|b,\mathfrak{a})\tilde{V}^{G}_{t-1}(b^{\mathfrak{a},\mathbf{z}})$ be the action-value function computed by greedy PBVI with immediate reward being $\tilde{\rho}(b)$. Also, let 
\begin{equation}
\begin{split}
\tilde{Q}_{t}^{\pi}(b,\mathfrak{a}) =  \tilde{\rho}(b) + \sum_{\mathbf{z}}\Pr(\mathbf{z}|b,\mathfrak{a})\tilde{V}^{\pi}_{t-1}(b^{\mathfrak{a},\mathbf{z}}), \\
\tilde{V}^{\pi}_{t}(b) =  \tilde{\rho}(b) + \sum_{\mathbf{z}}\Pr(\mathbf{z}|b,\mathfrak{a}_{\pi})\tilde{V}^{\pi}_{t-1}(b^{\mathfrak{a}_{\pi},\mathbf{z}}),
\end{split}
\end{equation} 
denote the value function for a given policy $\pi$, when the belief based reward is $\tilde{\rho}(b)$. 
As mentioned before, it is not guaranteed that $\tilde{Q}^{G}_{t}(b,\mathfrak{a})$ is submodular. Instead, we show that it is \emph{$\epsilon$-submodular}: 

\begin{definition}
The set function $f(\mathfrak{a})$ is $\epsilon$-submodular in $\mathfrak{a}$, if for every $\mathfrak{a}_{M} \subseteq \mathfrak{a}_{N} \subseteq A^+$, $a_{e} \in A^+ \setminus \mathfrak{a}_{N}$ and $\epsilon \geq 0$,
\begin{equation*}
f({a}_{e} \cup \mathfrak{a}_{M}) - f(\mathfrak{a}_{M}) \geq f({a}_{e} \cup \mathfrak{a}_{N}) - f(\mathfrak{a}_{N}) \\ - \epsilon.
\end{equation*}
\end{definition}

\begin{lemma} \label{etasubmod}
If $||V^{\pi}_{t-1} - \tilde{V}^{\pi}_{t-1}||_{\infty} \leq \eta$, and $Q^{\pi}_{t}(b,\mathfrak{a})$ is submodular in $\mathfrak{a}$, then $\tilde{Q}^{\pi}_{t}(b,\mathfrak{a})$ is $\epsilon'$-submodular in $\mathfrak{a}$ for all $b$, where $\epsilon'=4(\gamma + 1)\eta$.
\end{lemma}
\begin{proof}
Since, $||V^{\pi}_{t-1} - \tilde{V}^{\pi}_{t-1}||_{\infty} \leq \eta$, then for all beliefs $b$,
\begin{equation}
V^{\pi}_{t-1}(b) - \tilde{V}^{\pi}_{t-1}(b) \leq \eta, 
\end{equation}
For a given $\mathfrak{a}$, on both sides multiply $\gamma \geq 0$, take the expectation over $\mathbf{z}$ and since $\rho(b) - \tilde{\rho}(b) \leq \eta$,  ,
\begin{equation}
\rho(b) - \tilde{\rho}(b) + \gamma\mathbb{E}_{\mathbf{z}|b,\mathfrak{a}}V^{\pi}_{t-1}(b)  - \gamma \mathbb{E}_{\mathbf{z}|b,\mathfrak{a}}\tilde{V}^{\pi}_{t-1}(b) \leq \gamma\eta + \eta  \nonumber \\ 
\end{equation}
Therefore for all $b$, $\mathfrak{a}$, 
\begin{equation}
Q^{\pi}_{t}(b,\mathfrak{a}) - \tilde{Q}^{\pi}_{t}(b,\mathfrak{a}) \leq (\gamma+1)\eta 
\end{equation}
Now since $Q^{\pi}_{t}(b,\mathfrak{a})$ is submodular, it satisfies the following equation, 
\begin{equation}\label{qsubmod}
\resizebox{0.42\textwidth}{!}{${Q}^{\pi}_{t}(b,{a}_{e} \cup \mathfrak{a}_{M}) - {Q}^{\pi}_{t}(b,\mathfrak{a}_{M}) \geq {Q}^{\pi}_{t}(b,{a}_{e} \cup \mathfrak{a}_{N}) - {Q}^{\pi}_{t}(b,\mathfrak{a}_{N})$}, 
\end{equation}
for every $\mathfrak{a}_{M} \subseteq \mathfrak{a}_{N} \subseteq A^+$, $a_{e} \in A^+ \setminus \mathfrak{a}_{N}$
For each action that appear in \eqref{qsubmod}, that is, $\{a_{e} \cup \mathfrak{a}_{M}\}, \mathfrak{a}_{M}, \{a_{e} \cup \mathfrak{a}_{N}\} $ and $ \mathfrak{a}_{N}$, the value computed by $\tilde{Q}^{\pi}_{t}$ for belief $b$ will be an approximation to $Q^{\pi}_{t}$. Thus the inequality in \eqref{qsubmod} that holds for $Q^{\pi}_{t}$, may not hold for $\tilde{Q}^{\pi}_{t}$. The worst case possible is, for some combination of $b,  \{a_{e} \cup \mathfrak{a}_{M}\}, \mathfrak{a}_{M}, \{a_{e} \cup \mathfrak{a}_{N}\} $, $\tilde{Q}^{\pi}_{t}(b,{a}_{e} \cup \mathfrak{a}_{M})$  and ${Q}^{\pi}_{t}(b,\mathfrak{a}_{N})$ underestimates the true value of ${Q}^{\pi}_{t}(b,{a}_{e} \cup \mathfrak{a}_{M})$ and $\tilde{Q}^{\pi}_{t}(b,\mathfrak{a}_{N})$ by $(\gamma+1)\eta$ each and $\tilde{Q}^{\pi}_{t}(b,\mathfrak{a}_{M})$ and $\tilde{Q}^{\pi}_{t}(b,{a}_{e} \cup \mathfrak{a}_{N})$ overestimates the value of ${Q}^{\pi}_{t}(b,\mathfrak{a}_{M})$ and ${Q}^{\pi}_{t}(b,{a}_{e} \cup \mathfrak{a}_{N})$ by $(\gamma+1)\eta$ each. This can be written formally as:
$\tilde{Q}^{\pi}_{t}(b,{a}_{e} \cup \mathfrak{a}_{M}) - \tilde{Q}^{\pi}_{t}(b,\mathfrak{a}_{M})  \geq \tilde{Q}^{\pi}_{t}(b,{a}_{e} \cup \mathfrak{a}_{N}) - \tilde{Q}^{\pi}_{t}(b,\mathfrak{a}_{N}) - 4(\gamma + 1)\eta.   $
\qed
\end{proof}

\begin{lemma}\label{oneappG}
If $\tilde{Q}^{\pi}_{t}(b,\mathfrak{a})$ is non-negative, monotone and $\epsilon$-submodular in $\mathfrak{a}$, then
\begin{equation}
\tilde{Q}^{\pi}_{t}(b,\mathfrak{a}^{G}) \geq (1-e^{-1})\tilde{Q}^{\pi}_{t}(b,\mathfrak{a}^{*}) - 4\chi_{K}\epsilon,
\end{equation}
where $\chi_{K} = \sum_{p=0}^{K-1} (1 - K^{-1})^{p}$.
\end{lemma}
\begin{proof}
Let $\mathfrak{a}^{*}$ be the optimal set of action features of size $K$, $\mathfrak{a}^{*} = \argmax_{\mathfrak{a}}\tilde{Q}^{\pi}_{t}(b,\mathfrak{a})$ and let $\mathfrak{a}^{l}$ be the greedily selected set of size $l$, that is, $\mathfrak{a}^{l} = \mathtt{greedy}\hbox{-}\mathtt{argmax}(\tilde{Q}^{\pi}_{t}(b,\cdot),A^+,l)$
Also, let $\mathfrak{a}^{*} = \{a^{*}_{1} \dots a^{*}_{K}\}$ be the elements of set $\mathfrak{a}^{*}$. Then,
\begin{equation}
\begin{split}
& \mbox{By monotonicity of $\tilde{Q}^{\pi}_{t}(b,\mathfrak{a})$} \\
& \tilde{Q}^{\pi}_{t}(b,\mathfrak{a}^{*}) \leq \tilde{Q}^{\pi}_{t}(b,\mathfrak{a}^{*} \cup \mathfrak{a}^{l}) \\  
& \mbox{Re-writing as a telescoping sum} \\
& = \tilde{Q}^{\pi}_{t}(b,\mathfrak{a}^{l}) + \sum_{j = 1}^{K}\Delta_{\tilde{Q}_{b}}(a_{j}^{*}|\mathfrak{a}^{l} \cup \{a_{1}^{*} \dots a_{j-1}^{*}\})\\
& \mbox{Using Lemma \ref{etasubmod}, since $Q$ is $\epsilon'$-submodular} \\
& \leq \tilde{Q}^{\pi}_{t}(b,\mathfrak{a}^{l}) + \sum_{j = 1}^{K} \Delta_{\tilde{Q}_{b}}(a_{j}^{*}|\mathfrak{a}^{l}) + 4K\epsilon \\
& \mbox{As $\mathfrak{a}^{l+1}$ is built  greedily from $\mathfrak{a}^{l}$ in order to maximize $\Delta_{\tilde{Q}_{b}}$  }  \\
& \leq \tilde{Q}^{\pi}_{t}(b,\mathfrak{a}^{l}) + \sum_{j = 1}^{K} (\tilde{Q}^{\pi}_{t}(b,\mathfrak{a}^{l+1}) - \tilde{Q}^{\pi}_{t}(b,\mathfrak{a}^{l})) + 4K\epsilon \\
& \mbox{As $|\mathfrak{a}^{*}| = K$} \\
& =  \tilde{Q}^{\pi}_{t}(b,\mathfrak{a}^{l}) + K(\tilde{Q}^{\pi}_{t}(b,\mathfrak{a}^{l+1}) - \tilde{Q}^{\pi}_{t}(b,\mathfrak{a}^{l})) + 4K\epsilon  \nonumber
\end{split}
\end{equation}
Let $\delta_{l} := \tilde{Q}^{\pi}_{t}(b,\mathfrak{a}^{*}) - \tilde{Q}^{\pi}_{t}(b,\mathfrak{a}^{l})$, which allows us to rewrite above equation as: 
$\delta_{l} \leq K(\delta_{l} - \delta_{l+1}) + 4K\epsilon$. Hence, $\delta_{l+1} \leq (1 - \frac{1}{K}) \delta_{l} + 4\epsilon$. Using this relation recursively, we can write,
$\delta_{K} \leq (1 - \frac{1}{K})^{K} \delta_{0} + 4\sum_{p=0}^{K-1}(1 - \frac{1}{K})^{p}\epsilon$.
Also, $\delta_{0} = \tilde{Q}^{\pi}_{t}(b,\mathfrak{a}^{*}) - \tilde{Q}^{\pi}_{t}(b,\mathfrak{a}^{0})$ and using the inequality $1 - x \leq e^{-x}$, we can write 
$\delta_{K} \leq e^{-\frac{K}{K}}\tilde{Q}^{\pi}_{t}(b,\mathfrak{a}^{*}) + 4\sum_{p=0}^{K-1} (1 - K^{-1})\epsilon$. Substituting $\delta_{K}$ and rearranging terms (Also $\chi_{K} = \sum_{p=0}^{K-1}(1 - \frac{1}{K})^{p}$): 
$\tilde{Q}^{\pi}_{t}(b,{\mathfrak{a}^{G}}) \geq (1 - e^{-1})\tilde{Q}^{\pi}_{t}(b,\mathfrak{a}^{*}) - 4\chi_{K}\epsilon.$
\qed
\end{proof}

\noindent{\textbf{Theorem 7}}
\emph{For all beliefs, the error between $\tilde{V}^{G}_{t}(b)$ and $\tilde{V}^{*}_{t}(b)$ is bounded, if $\rho(b) = -H_{b}(s)$, $V^{\pi}_{t}$ is convex in the belief space for all $\pi, t$, and if $\mathfrak{z}^{t:k}$ is conditionally independent given $s^{k}$.
}
\begin{proof}
Theorem \ref{th:core} shows that, if $\rho(b) = -H_{b}(s)$, and $\mathfrak{z}^{t:k}$ is conditionally independent given $s^{k}$, then $Q^{G}_{t}(b,\mathfrak{a})$ is submodular.
Using Lemma \ref{etasubmod}, for $V^{\pi}_{t}=V^{G}_{t}$, $\tilde{V}^{\pi}_{t}=\tilde{V}^{G}_{t}$, $Q^{\pi}_{t}(b,\mathfrak{a})=Q^{G}_{t}(b,\mathfrak{a})$ and $\tilde{Q}^{\pi}_{t}(b,\mathfrak{a})=\tilde{Q}^{G}_{t}(b,\mathfrak{a})$, it is easy to see that  
$\tilde{Q}^{G}_{t}(b,\mathfrak{a})$ is $\epsilon$-submodular. This satisfies one condition of Lemma \ref{oneappG}. Given that $\tilde{V}^{G}_{t}(b)$ is convex, the monotonicity of $\tilde{Q}^{G}_{t}(b,\mathfrak{a})$ follows from Lemma \ref{lem:mono}. Since $\tilde{\rho}(b)$ is non-negative, $\tilde{Q}^{G}_{t}(b,\mathfrak{a})$ is non-negative too. Now we can apply Lemma 9 to prove that the error generated by a one-time application of the greedy Bellman operator to $\tilde{V}^{G}_{t}(b)$, instead of the Bellman optimality operator, is bounded. It is thus easy to see that the error between $\tilde{V}^{G}_{t}(b)$, produced by multiple applications of the greedy Bellman operator, and $\tilde{V}^{*}_{t}(b)$ is bounded for all beliefs.
\qed
\end{proof}

\begin{acknowledgements}

We thank Henri Bouma and TNO for providing us with the dataset used in our experiments.  We also thank the STW User Committee for its advice regarding active perception for multi-camera tracking systems. This research is supported by the Dutch Technology Foundation STW (project \#12622), which is part of the Netherlands Organisation for Scientific Research (NWO), and which is partly funded by the Ministry of Economic Affairs.
Frans Oliehoek is funded by NWO Innovational Research Incentives Scheme Veni \#639.021.336.
\end{acknowledgements}

\bibliographystyle{spbasic}
\bibliography{journalBib}

\end{document}